\documentclass[12pt]{article}
\usepackage{amsthm,amsmath,amssymb}
\usepackage{natbib}
\usepackage{multirow}
\usepackage{graphicx}
\usepackage{subfigure}
\usepackage{makecell}
\usepackage{booktabs}
\usepackage{array}
\usepackage{fullpage}
\usepackage{url}
\usepackage{hyperref}
\usepackage{algorithm}
\usepackage{algorithmic}
\usepackage{bm}
\usepackage{smile}
\usepackage{mathtools}
\usepackage{wrapfig}
\usepackage{lipsum}
\usepackage{mathrsfs}
\usepackage{dsfont}
\usepackage{rotating}
\usepackage{float}
\usepackage{natbib}
\usepackage{multirow}
\usepackage{apalike}
\usepackage{xcolor}

\newtheorem{assumption}{Assumption}

\newtheorem{lemma}{Lemma}
\newtheorem{theorem}{Theorem}
\newtheorem{corollary}{Corollary}
\newtheorem{proposition}{Proposition}

\newcommand{\blind}{1}

\def\reals{{\mathbb R}}

\def\G{{\mathcal G}}

\def\E{{\mathbb E}}
\def\argmin{\mathop{\text{\rm arg\,min}}}
\def\s{\text{s}}
\def\p{\text{p}}

\def\argmax{\mathop{\text{\rm arg\,max}}}

\def\span{{\mathop{\text{\rm span}}}}
\def\tr{{\mathop{\text{\rm tr}}}}

\begin{document}

\def\spacingset#1{\renewcommand{\baselinestretch}%
{#1}\small\normalsize} \spacingset{1}


\if1\blind
{
  \title{\bf Knowledge Transfer across Multiple Principal Component Analysis Studies}
  \author{Zeyu Li\footnotemark[2]\\
    School of Management, Fudan University\\
    Kangxiang Qin\footnotemark[2] \\
    School of Mathematics, Shandong University
    \\
    Yong He\footnotemark[1]\\
    Institute for Financial Studies, Shandong University\\
    Wang Zhou\\
    Department of Statistics and Data Science, National University of Singapore\\
    Xinsheng Zhang\\
    School of Management, Fudan University}
  \maketitle
} \fi

\renewcommand{\thefootnote}{\fnsymbol{footnote}}

\footnotetext[2]{The authors contributed equally to this work.}
\footnotetext[1]{Corresponding author, email: heyong@sdu.edu.cn.}

\if0\blind
{
  \bigskip
  \bigskip
  \bigskip
  \begin{center}
    {\LARGE\bf Knowledge Transfer across Multiple Principal Component Analysis Studies}
\end{center}
  \medskip
} \fi

\bigskip
\begin{abstract}
Transfer learning has aroused great interest in the statistical community. In this article, we focus on knowledge transfer for unsupervised learning tasks in contrast to the supervised learning tasks in the literature.  Given the transferable source populations, we propose a two-step transfer learning algorithm to extract useful information from multiple source principal component analysis (PCA) studies, thereby enhancing estimation accuracy for the target PCA task. In the first step, we integrate the shared subspace information across multiple studies by a proposed method named as Grassmannian barycenter, instead of directly performing PCA on the pooled dataset. The proposed Grassmannian barycenter method enjoys robustness and computational advantages in more general cases. Then the resulting  estimator for the shared subspace from the first step is further utilized to estimate the target private subspace  in the second step. Our theoretical analysis credits the gain of knowledge transfer between PCA studies to the enlarged eigenvalue gap, which is different from the existing  supervised transfer learning tasks where sparsity plays the central role. In addition, we prove that the bilinear forms of the empirical spectral projectors have asymptotic normality under weaker eigenvalue gap conditions after knowledge transfer. When the set of informativesources is unknown, we endow our algorithm with the capability of useful dataset selection by solving a rectified optimization problem on the Grassmann manifold, which in turn leads to a computationally friendly rectified Grassmannian K-means procedure.  In the end, extensive numerical simulation results and a real data case concerning activity recognition are reported to  support our theoretical claims and to illustrate the empirical usefulness of the proposed transfer learning methods.
\end{abstract}
\noindent%
\bigskip
\spacingset{1.45} 

\section{Introduction}
\label{sec:intro}

In this era of big data, the availability of various public datasets enables us to improve the performance of a new study of interest by taking advantage of relevant information from the existing ones. The idea of transferring knowledge from related source studies to the target study, originated from computer science \citep{torrey2010transfer,zhuang2020comprehensive,niu2020decade}, has aroused great interest in the statistical community. For transfer learning of supervised tasks, \cite{bastani2021predicting, li2020transfer} lead the trend of two-step transfer learning procedures: in the first step the useful datasets are pooled to obtain a primal estimator, which is then debiased in the second step using target data only \citep{tian2022transfer,transquantile}. One may also refer to \cite{cai2021transfer,reeve2021adaptive} for non-parametric classification problems. In \cite{li2023estimation}, the authors  propose to jointly estimate the target parameter and contrast vectors between the target and the sources for high-dimension generalized linear model, and their method does not require any regularity condition on the Hessian matrices, which is typically  assumed for the two-step procedures \citep{bastani2021predicting, li2020transfer,tian2022transfer,transquantile}.
On the contrary, knowledge transfer for unsupervised learning tasks is rarely discussed in the literature. In this work, we focus on transfer learning problems for one of the most important unsupervised learning task, that's Principal Component Analysis (PCA). As far as we know, this is the first work for unsupervised learning tasks in the literature.

Principal component analysis is one of the  most popular methods for dimension reduction and feature extraction \citep{pearson1901liii}. The main goal of PCA is to reduce the dimension of a dataset while retaining as much of the original variability as possible. It involves finding the eigenvectors and eigenvalues of the sample covariance matrix, where eigenvectors represent the directions of maximum variance in the data, and eigenvalues indicate the magnitude of variance in those directions. PCA has wide applications in various fields, including  but not limited to linear regression, functional data analysis and factor analysis \citep{jolliffe2003modified,fan2013large,Chen2020StatisticalIF,he2022large}.
However, in the high-dimensional regime when the data dimension $p$ is comparable or even much larger than the sample size $n$, it is well-known that classical PCA would deliver inconsistent estimators \citep{baik2005phase,johnstone2009consistency,bai2010spectral}. One stream of research focuses on the so-called ``diverging eigenvalue" regime \citep{wang2017asymptotics,cai2020limiting,xia2021normal}, where consistency of the PCA estimator is guaranteed by the strong signal strength. Indeed, \cite{fan2013large} shows that the leading eigenvalues of the covariance matrix grow linearly with the data dimension $p$ for factor models with pervasive factors. A profound link between the convergence of the empirical eigen-structure and the effective rank of the covariance matrix is established in \cite{koltchinskii2014asymptotics,koltchinskii2017concentration}.
Alternatively, another way to tackle such ``curse of dimension" problem is the sparse PCA technique, also known as SPCA. Researches on sparse PCA has experienced a surge development including \cite{zou2006sparse,d2004direct,cai2013sparse,lei2015sparsistency,zou2018selective}. Other than the ``diverging eigenvalue" regime, literature related to sparse PCA tends to limit the signal strength of the data and require additional sparsity conditions by making the bounded eigenvalues and sparse eigenvectors assumptions, but the sparsity conditions can be hard to verify in practice. Thus we adopt the diverging eigenvalue mindset in this work so that no additional structure of eigenvectors is required.


\subsection{Knowledge transfer framework}

In this work, we focus on knowledge transfer across multiple PCA studies. Let $\Sigma^*_0$ be the population target covariance matrix, and $\hat{\Sigma}_0=\Sigma^*_0+E_0$ be its  sample version (e.g., the sample covariance matrix), with sample size $n_0$. The target goal is to retrieve the principal subspace $\span(U^*_0)$, where the $p\times r_0$ column orthogonal matrix $U^*_0$ is  the leading $r_0$ eigenvectors of $\Sigma^*_0$, i.e., the eigenvectors corresponding to the largest $r_0$ eigenvalues. Classical PCA estimator only using the target dataset is given by the leading $r_0$ eigenvectors of $\hat{\Sigma}_0$. The purpose of this work is to fully extract  useful homogeneity information from multiple source sample-version covariance matrices $\hat{\Sigma}_k=\Sigma^*_k+E_k$, with sample size $n_k$, for $k\in[K]:=\{1,\dots,K\}$, so as to enhance the estimation accuracy of the target study.

To model the transferable knowledge between the target dataset and the informative sources denoted by $k\in\cI \subseteq[K]$, we assume that certain subspace information is shared across these datasets. For any $\Sigma_k^*$ with $k\in\cI$, let  $U_k^*$  be a column orthogonal matrix with columns being $\Sigma_k^*$'s leading $r_k$ eigenvectors. For transfer learning, it is surely desirable if $r_k=r_0$ and $\span(U_k^*)=\span(U_0^*)$. It's even better if $\Sigma_k^*=\Sigma_0^*$, then knowledge transfer from the $k$-th study would be a relatively easy task, as one could simply calculate the leading $r_0$ eigenvectors of the pooled sample covariance matrix $(n_k\hat{\Sigma}_k+n_0\hat{\Sigma}_0)/(n_k+n_0)$. However, in real applications things shall be more complicated in the following senses.
\begin{itemize}
    \item First, $r_k$ is not necessarily equal to $r_0$, even if it is the case, $\span(U_k^*)$ as a whole might not  be close to $\span(U_0^*)$. There is chance that only a subspace of $\span(U_k^*)$ is just close to a subspace of $\span(U_0^*)$, while the remainder subspace of $\span(U_k^*)$ would result in negative transfer . We denote this transferable (shared) subspace of $\span(U_k^*)$ as $\span(U_k^{\s})$, where $U_k^{\s}$ is a column orthogonal matrix.
    \item Second, even there exists such shared subspace $\span(U_k^{\s}) \subseteq \span(U_k^*)$ that is helpful to the target task, there is no guarantee that $\span(U_k^{\s})$ is spanned by the eigenvectors corresponding to the leading eigenvalues of $\Sigma_k^*$. For instance, it could be spanned by those eigenvectors corresponding to some small eigenvalues, or even be an arbitrary subspace of $\span(U_k^*)$. It makes the naive method of summing the sample covariance matrices and then performing PCA questionable. Consider the toy example when $U_k^{\s}$ is a single eigenvector corresponding to the second largest eigenvalue of $\Sigma_k^*$. Then, the eigen-structure of the pooled sample covariance matrix $(n_k\hat{\Sigma}_k+n_0\hat{\Sigma}_0)/(n_k+n_0)$ could be severely influenced by the largest eigenvalue (and its corresponding eigenvector) of $\Sigma_k^*$, if, e.g., it is significantly larger than the second largest eigenvalue of $\Sigma_k^*$ and the largest eigenvalue of $\Sigma_0^*$. Since the eigenvector corresponding to the largest eigenvalue of $\Sigma_k^*$ could be irrelevant to the target study, knowledge transfer by summing the sample covariance matrices would be negative in this case.
    \item In the end, there is also no guarantee that  $\span(U_0^*)$ as a whole could all benefit from the source information. It's possible that there exists a subspace of $\span(U_0^*)$ that is not close to any subspace of $\span(U_k^*)$, and we have  to estimate such subspace, denoted as $\span(U_0^{\p})$, by using the target dataset alone. In this work, we call $\span(U_0^{\p})$ as the private subspace of $\span(U_0^*)$\footnote[4]{For better comprehension, $\s$ in the superscript stands for ``shared", while $\p$ stands for ``private".}.
\end{itemize}

We turn to briefly introduce our knowledge transfer framework, which takes all the arguments above into account. For $k\in\{0\}\cup\cI$ and $1\leq r_\s\leq \min_{k\in \{0\}\cup\cI}(r_k)$, we decompose each $r_k$-dimensional $\span(U_k^*)$ into the direct sum of two subspaces, namely
$$\span(U_k^*)= \span(U_k^{\s})\oplus \span(U_k^{\p}),$$
where $U_k^{\s}$ and $U_k^{\p}$ are $p\times r_\s$ and $p\times (r_k-r_\s)$ column orthogonal matrices, respectively. Equivalently, let $P_k^*=U_k^*(U_k^*)^{\top}$, $P_k^{\text{s}}=U_k^{\s}(U_k^{\s})^{\top}$ and $P_k^{\text{p}}=U_k^{\p}(U_k^{\p})^{\top}$ be the orthogonal projection matrices to the corresponding subspaces respectively, we then write equivalently that
$$P^*_k=P_k^{\text{s}}+P_k^{\text{p}},\quad \text{such that}\quad P_k^{\text{s}}P_k^{\text{p}} =0,$$
as any subspace can be uniquely determined by the orthogonal projector from $\mathbb{R}^p$ onto itself. For $k\in \cI$, $\span(U_k^{\s})$ contains transferable information of $\span(U_0^{\s})$. To evaluate the informative level of $\span(U_k^{\s})$ for $k\in \cI$, we assume that for some $h>0$,
\begin{equation}\label{eq:inflvl}
    \left\|P_k^\s-P_0^\s\right\|_F\leq h,
\end{equation}
where $\|\cdot\|_F$ is the matrix Frobenius norm and (\ref{eq:inflvl}) bounds the distance between $\span(U_k^{\s})$ and $\span(U_0^{\s})$ by the well-known projection metric. In fact, (\ref{eq:inflvl}) is equivalent to $ \|(U_k^{\s})^{\top }(U_0^{\s})^{\perp}\|_F \leq h/\sqrt{2}$. Geometrically, denote the non-trivial singular values of $(U_k^{\s})^{\top }(U_0^{\s})^{\perp}$ as $\{\sigma^{\s}_{k,i}\}_{i=1}^{r_\s}$, the principal angles between $\span(U_k^{\s})$ and $\span((U_0^{\s})^{\perp})$ are defined as $\{\cos^{-1}(\sigma^{\s}_{k,i})\}_{i=1}^{r_\s}$, then small $\|(U_k^{\s})^{\top }(U_0^{\s})^{\perp}\|_F$ indicates almost orthogonality between $\span(U_k^{\s})$ and $\span((U_0^{\s})^{\perp})$, or equivalently $\span(U_k^{\s})$ shall be very close to $\span(U_0^{\s})$. As for those non-informative source datasets, denoted as $k\in \cI^c$ in this work, $\span(U^*_k)$ can be quite different from $\span(U^*_0)$. For convenience, throughout this work we treat $r_k$ as a finite number, namely there exists some $r_{\max}<\infty$ such that $r_k\leq r_{\max}$ for all $k\in[K]$.

We end this subsection with a few remarks. First, a closely related problem is multi-task learning \citep{zhang2021survey,yamane2016multitask}, whose goal is to jointly solve all multiple tasks at the same time. Indeed, it is worth mentioning that the target dataset and any informative source dataset are interchangeable under our framework, in view of the fact that $\|P_k^\s-P_l^\s\|_F\leq \|P_k^\s-P_0^\s\|_F+\|P_l^\s-P_0^\s\|_F\leq 2 h$ for any $k, l\in\{0\}\cup \cI$. So in principle, any study $k\in\{0\}\cup \cI$ can be viewed as the target study, and it is possible to achieve performance enhancement for each study $k\in\{0\}\cup \cI$ using the knowledge transfer methods proposed in this work. Second, in this article we mainly discuss principal component analysis in a classical way for the sake of brevity. That is say, we assume that the $k$-th dataset includes $n_k$ i.i.d. $p$-dimensional random vectors $x_{k,i}$ with mean zero and covariance $\Sigma^*_k$, and the corresponding sample covariance matrix is
    $$\hat{\Sigma}_k=\frac{1}{n_k}\sum_{i=1}^{n_k} x_{k,i} x_{k,i}^{\top}.$$
In fact, the knowledge transfer framework in this work goes beyond classical PCA and extends readily to more general cases including elliptical PCA for robust dimension reduction against heavy-tailed noise \citep{Fan2018LARGE, he2022large} and two-directional PCA for two-way dimension reduction of matrix-valued observations \citep{zhang20052d,chen2021statistical,Yu2021Projected}.

\subsection{Closely related works and our contributions}

Among the few works concerning knowledge transfer for unsupervised learning studies, \cite{duan2023target} suggests performing PCA to the linear combination of the target and only one source covariance matrix with large sample size. In this way, to consistently identify the subspace corresponding to the smaller eigenvalues in the target sample covariance matrix, such subspace should  correspond to the relatively stronger eigenvalues in the source sample covariance matrix. While the idea is insightful, this requirement is clearly restrictive and not fulfilled in all knowledge transfer cases. Another closely related work is \cite{fan2019distributed} in which the authors propose a distributed algorithm for PCA when the local population covariance matrices share the same leading eigenspace. In contrast, we assume the existence of private subspaces for the local matrices, hence an additional debiasing step is naturally desirable. In this work, we propose a two-step knowledge transfer procedure across multiple PCA studies. In the first step, we integrate the shared subspace information across the informative source studies using a  ``Grassmannian barycenter" (GB) method. Then, in the second step, we debias and estimate the private subspace of the target study with the target dataset only. Such knowledge transfer framework in fact adapts  to more general PCA settings \citep{Fan2018LARGE, he2022large,zhang20052d,chen2021statistical,Yu2021Projected}.

The proposed GB method  in our first step fully integrate shared information across multiple PCA studies. Compared to first pooling all datasets into a large dataset and then perform PCA, the advantage of the GB method lies in the following aspects. First, as the goal of the first step is to estimate the shared subspace, GB method only extacts directional information  and is less sensitive to those private subspaces corresponding to extremely large eigenvalues, resulting in more robust performance. Second, in a similar manner as \cite{fan2019distributed,hu2023optimal}, our GB method turns out to be suitable even when the source datasets are distributed in a relatively large number of machines. Instead of transmitting the entire source datasets to the target machine, the GB method only requires transmitting the source subspace estimators. That is to say, our procedure naturally adapts to the ``divide and conquer" scenarios, and is both computationally friendly and privacy protecting than directly pooling the target and source datasets.

In the second step, we estimate the private subspace of the target study, which is in the spirit of the debiasing step in the transfer learning literature on high-dimensional supervised learning \citep{li2020transfer,tian2022transfer,transquantile}. In these supervised transfer learning works, debiasing is required to be a relatively easier statistical task than directly estimating the target parameters. Indeed, the difficulty of these supervised statistical tasks is often related to the sparsity of the high-dimensional parameter vectors, and the differences between the target and the informative source parameters are assumed to be even more sparse than the target parameter, so that the transfer learning estimator would outperform the one using only the target dataset. Likewise, in the context of principal component analysis and with the well-known Davis-Kahan theorem, we measure the difficulty level of different PCA studies using the eigenvalue gap of the covariance matrices. The knowledge transfer estimator is promising if the task of estimating the private subspace is endowed with a much larger eigenvalue gap than the original target PCA study. Furthermore, the bilinear forms of the empirical projectors would also  achieve asymptotic normality much easier due to the same reason, making room for further statistical inferences.

In practice, informative source datasets are usually unknown. Various methods with statistical guarantees are proposed to avoid ``negative transfer", see, e.g., model selection aggregation in \cite{li2020transfer} and data-driven transferable source detection in \cite{tian2022transfer}. However, the existing tools are often designed for knowledge transfer between a relatively small number of datasets, and might face computational challenges in some cases, e.g., when there are a large number of datasets distributed in different machines.  Indeed, as will be revealed later, estimating the shared subspace using Grassmannian barycenter is rather robust against mild inclusion of non-informative datasets. In this work, we struggle to pursue for a computationally more efficient dataset selection method. We propose to solve a non-convex rectified version of the manifold optimization problem which has led to the Grassmannian barycenter method. To search for the local maximum of this non-convex problem, one way is by resorting to the well-developed manifold optimization techniques. Optionally, we also suggest using a rectified Grassmannian K-means procedure which inherits the nature of the Grassmannian barycenter method. The idea of rectification has been adopted in various clustering analysis problems \citep{shen2012likelihood, Pan2013Cluster, wu2016new,liu2023cluster}, and our approach could be viewed as clustering datasets on a subspace manner. The proposed method is not only statistically accurate but also computationally friendly in both numerical experiments and real data cases, thanks to its capability of simultaneously selecting the useful sources while fully harnessing the information from these sources.

In summary, in this article, we discuss knowledge transfer problem across multiple principal component analysis studies under a general framework that is able to take many real world cases into account. In contrast to most existing works focusing on a small number of source datasets, we propose knowledge transfer methods that are both statistically accurate and computational feasible even with a large number of source datasets potentially distributed in different machines. Indeed, the number of source datasets is allowed to diverge at a reasonable rate. In addition, our theoretical arguments credit the gain of knowledge transfer across PCA studies to the enlarged eigenvalue gap, which is different from the  supervised transfer learning tasks where sparsity of the contrast vector is the essence. In the end, extensive numerical simulation results and a real data case concerning activity recognition are reported, so as to provide evidence on our claims and also the empirical usefulness of the proposed methods.

\subsection{Organization and notations}
The remainder of this article is organized as follows. In Section \ref{sec:meth}, we first present the knowledge transfer algorithms across principal component analysis studies given the informative datasets.  Then, we further discuss the case that the informative datasets are unknown and needs to be selected.  In Section \ref{sec:theo}, theoretical properties of the proposed methods are discussed under mild conditions. Numerical simulation results are reported in Section \ref{sec:num} to support our theoretical arguments. In the end, a real dataset concerning activity recognition is analyzed in Section \ref{sec:real} to illustrate the practical usefulness of our methods.

We introduce some notations used throughout the paper to end this section. For a real symmetric matrix $A$, let $\{\lambda_i(A)\}$ be its non-increasing eigenvalues and $d_i(A) = \lambda_i(A)-\lambda_{i+1}(A)$ be the $i$-th eigenvalue gap. We write $\|A\|_{2}=\lambda_1(A)$ as the operator norm of the matrix $A$, while for the vector $a$, denote its $\ell_2$ norm as $\|a\|$. For a random variable $X\in \reals$, we define $\|X\|_{\psi_2}=\sup_{p\geq 1}p^{-1/2}(\E|X|^p)^{1/p}$ and $\|X\|_{\psi_1}=\sup_{p\geq 1}p^{-1}(\E|X|^p)^{1/p}$, please refer to \cite{vershynin2018high} for details of the sub-Gaussian and sub-exponential norms. Furthermore, the $o_{p}$ is for convergence to zero in probability and the $O_{p}$ is for stochastic boundedness. We write $x\lesssim y$ if $x\leq Cy$ for some $C>0$, $x\gtrsim y$ if $x\geq cy$ for some $c>0$, and $x\asymp y$ if both $x\lesssim y$ and $x\gtrsim y$ hold.  Note that the constants may not be identical in different lines.

\section{Methodology}\label{sec:meth}
In this section we present our knowledge transfer algorithms given informative source datasets either known or unknown. When the informative sources are known in advance, we call it the \emph{oracle knowledge transfer} case, and it is discussed in Section \ref{sec:okt}. Then in Section \ref{sec:nokt}, we work on the case with unknown informative sources.
\subsection{Oracle knowledge transfer}\label{sec:okt}
We first discuss the oracle knowledge transfer scenario where all informative source datasets $\cI \subseteq [K]$ are known in advance. Recall that the goal is to get a better estimation of the $r_0$-dimensional target subspace with the help of useful source datasets. With a slight abuse of notation we identify the target subspace by the projection matrix $P_0^*=P_0^\s+P_0^\p$. As the private subspace $P_0^\p$ could only be estimated using target data alone in principle, we first seek a better estimation of the shared subspace $P_0^\s$, by making fully use of the information from the target and informative sources.

As discussed in Section \ref{sec:intro}, there is no guarantee that $\span(U_k^{\s})$ is spanned by the eigenvectors corresponding to the leading $r_\s$ eigenvalues of $\Sigma_k^*$ for $k\in\{0\}\cup\cI$, so simply summing the sample covariance matrices and then calculating the leading eigenvectors of the resulting pooled matrix might not be promising, if, e.g., eigenvalues corresponding to some private subspaces are extremely large. On the contrary, we suggest taking only the directional information into account, discarding the potentially misleading eigenvalue information of the sample covariance matrices from the source studies. To be more specific, we first acquire column orthogonal $\tilde{U}_k$ by taking the leading $r_k$ eigenvectors of $\hat{\Sigma}_k$. Let $P_{\tilde{U}_k}=\tilde{U}_k\tilde{U}_k^{\top}$ be the projection matrix of $\span(\tilde{U}_k)$. We abbreviate $P_{\tilde{U}_k}$ as $\tilde{P}_k$ for brevity of notations. Then, the estimator of the shared subspace, $\hat{P}_0^\s$, is constructed by  using the leading $r_\s$ eigenvectors of the (weighted) average projection matrix
\begin{equation}\label{eq:aveproj}
    \hat{\Sigma}^\s = \frac{1}{N_{\cI}}\sum_{k\in\{0\}\cup\cI} n_k\tilde{P}_k,
\end{equation}
where $N_{\cI} = \sum_{k\in\{0\}\cup\cI} n_k$ is the pooled sample size. It is  helpful to introduce the notion of Grassmann manifold $\mathcal{G}(p,r)$, which is defined as the set of $r$-dimensional linear subspaces of $\mathbb{R}^p$, so that the target aim $\span(U_0^*)\in \mathcal{G}(p,r_0)$. Meanwhile, any subspace can be uniquely represented by the orthogonal projector from $\mathbb{R}^p$ onto itself. For instance, $\span(U^*_0)$ is uniquely determined by the projection matrix $P^*_0 =U^*_0(U^*_0)^{\top}$. With a slight abuse of notation, we also identify the Grassmann manifold as follows:
\begin{equation*}
	\mathcal{G}(p,r)=\{P\in \mathbb{R}^{p\times p}\mid P^{\top}=P,\, P^2=P,\, \text{rank} (P)=r\},
\end{equation*}
please refer to \cite{bendokat2020grassmann} for detailed discussions on the Grassmann manifold. Indeed, the leading $r_\s$ eigenvectors of (\ref{eq:aveproj}) can be viewed as the solution of the following optimization problem:
\begin{equation}\label{eq:GBoracle}
\begin{aligned}
        \hat{P}_0^{\s}&=\argmax_{P\in \G(p,r_\s)}\frac{1}{N_{\cI}} \sum_{k\in\{0\}\cup\mathcal{I}}n_k\tr(\tilde{P}_kP) \\
        &= \argmin_{P\in \G(p,r_\s)}\frac{1}{N_{\cI}} \sum_{k\in\{0\}\cup\cI}n_k\left\|\tilde{P}_k-P\right\|^2_F.
\end{aligned}
\end{equation}
Hence intuitively, $\hat{P}_0^{\s}$ is the element on $\G(p,r_\s)$ that minimizes the squared projection metric to the informative subspaces $\tilde{P}_k\in \mathcal{G}(p,r_k)$ for $k\in \{0\}\cup \cI$. It coincides with the physical notion of barycenter, so we name $\hat{P}_0^{\s}$ as the Grassmannian barycenter. It is also worth mentioning that the Grassmannian barycenter can also be viewed as an instance of the extrinsic Fr\'{e}chet mean on manifolds discussed in \cite{bhattacharya2003large,bhattacharya2005large,eltzner2019smeary,hundrieser2020finite}.

After acquiring $\hat{P}_0^{\s}$ as an estimator of the shared subspace, if $r_\s<r_0$, we then need to estimate the private subspace $P_0^\p$ using the target dataset only. To do so, for $(\hat{P}_0^{\s})^{\perp}=I_p-\hat{P}_0^{\s}$, we acquire $\hat{P}_0^\p$ by taking the leading $(r_0-r_{\s})$ eigenvectors of the following projected sample covariance matrix
$$\hat{\Sigma}_0^{\p}=(\hat{P}_0^{\s})^{\perp} \hat{\Sigma}_0 (\hat{P}_0^{\s})^{\perp}.$$
Clearly $\hat{P}_0^{\s}\hat{P}_0^\p=0$, and we acquire the knowledge transfer estimator of $P_0^*$ by $\hat{P}_0= \hat{P}_0^{\s}+ \hat{P}_0^{\p}$. Please refer to Algorithm \ref{alg:ora} below for details.

\begin{algorithm}[ht]
\caption{Oracle knowledge transfer for principal component analysis.}\label{alg:ora}
\begin{algorithmic}[1]
\REQUIRE ~~\\
    $(\hat{\Sigma}_k, n_k,r_k)$ for $k\in\{0\}\cup \cI$; $r_{\s}$;\\
\ENSURE ~~\\
    \STATE individual PCA step: acquire $\tilde{P}_k$ by taking the leading $r_k$ eigenvectors of $\hat{\Sigma}_k$ for $k\in\{0\}\cup \cI$;\\
    \STATE GB step: take the leading $r_{\s}$ eigenvectors of $\hat{\Sigma}^\s=\sum_{k\in\{0\}\cup\cI} n_k\tilde{P}_k/N_{\cI}$ to obtain $\hat{P}_0^{\s}$;\\

    \STATE fine-tuning step: let $(\hat{P}_0^{\s})^{\perp}=I_p-\hat{P}_0^{\s}$, acquire $\hat{P}_0^\p$ by taking the leading $(r_0-r_{\s})$ eigenvectors of $\hat{\Sigma}_0^{\p}=(\hat{P}_0^{\s})^{\perp} \hat{\Sigma}_0 (\hat{P}_0^{\s})^{\perp}$;\\
\RETURN $\hat{P}_0= \hat{P}_0^{\s}+ \hat{P}_0^{\p}$.
\end{algorithmic}
\end{algorithm}

If $r_\s=r_0$, namely the target study has no private subspace, the gain by knowledge transfer is trivial as long as the difference between different studies, as measured by $h$, is sufficiently small. To get some intuition on the gain by knowledge transfer when $r_\s< r_0$ and there exists some private subspace to be estimated, we start by considering the ideal case when $\hat{P}_0^{\text{s}}=P_0^{\text{s}}$, i.e., the shared subspace is exactly recovered with the help of the informative sources. First, let $\{\lambda_i\}_{i=1}^{p}$ be the non-increasing eigenvalues of the target population covariance matrix $\Sigma_0^*$. According to Davis-Kahan theorem, the performance of PCA using the target sample covariance matrix $\hat{\Sigma}_0$ relies heavily on the eigenvalue gap $\delta_0:=d_{r_0}(\Sigma^*_0)=\lambda_{r_0}-\lambda_{r_0+1}$. Meanwhile, the fine-tuning step of Algorithm \ref{alg:ora} seeks the leading $(r_0-r_\s)$ eigenvectors of $(P_0^{\text{s}})^{\perp} \hat{\Sigma}_0 (P_0^{\text{s}})^{\perp} =(P_0^{\text{s}})^{\perp} \Sigma^*_0 (P_0^{\text{s}})^{\perp} + (P_0^{\text{s}})^{\perp} E_0 (P_0^{\text{s}})^{\perp}$. We then take a closer look at the ``signal part" $\Sigma_0^{\p} := (P_0^{\text{s}})^{\perp} \Sigma^*_0 (P_0^{\text{s}})^{\perp}$.  Due to the fact that $(P_0^{\text{s}})^{\perp}=P_0^{\p}+ (P_0^*)^{\perp}$, while $P_0^{\p}\Sigma_0^*(P_0^*)^{\perp}=0$, we have $$\Sigma_0^{\p} = (P_0^{\text{s}})^{\perp} \Sigma^*_0 (P_0^{\text{s}})^{\perp}= P_0^{\p} \Sigma^*_0 P_0^{\p} +(P_0^*)^{\perp} \Sigma^*_0 (P_0^*)^{\perp}= \mathcal{U}^{\p}_0 \Lambda_0^{\p} (\mathcal{U}^{\p}_0)^{\top},$$ where
\begin{equation}\label{eq:eigen}
    \mathcal{U}^{\p}_0 := \left(\underbrace{U_0^{\p}}_{p\times (r_0-r_\s)}\,\middle\vert\, \underbrace{(U_0^*)^{\perp}}_{p\times (p-r_0)}\,\middle\vert\, \underbrace{U_0^{\s}}_{p\times r_\s} \right),
\end{equation}
$$\Lambda_0^{\p} := \diag\left(\underbrace{\lambda_1^{\p},\dots,\lambda_{r_0-r_\s}^{\p}}_{r_0-r_\s}\,\middle\vert\, \underbrace{\lambda_{r_0+1},\dots,\lambda_{p}}_{p-r_0}\,\middle\vert\, \underbrace{0,\dots,0}_{r_\s} \right).$$
Here $\{\lambda_i^{\p}\}_{i=1}^{r_0-r_\s}$ are the non-increasing eigenvalues of the $(r_0-r_\s)\times (r_0-r_\s)$ matrix $(U_0^{\p})^{\top}\Sigma_0^*U_0^{\p}$. By the celebrated Courant–Fischer min-max principle, we have
$$\lambda^{\p}_{r_0-r_\s}\geq \lambda_{r_0},$$ please refer to Theorem 4.3.28 in \cite{horn2012matrix}. That is to say, from (\ref{eq:eigen}) we see that now the task of estimating the private subspace $U_0^{\p}$ has a larger eigenvalue gap $\delta_{\p}:=d_{r_0-r_\s}(\Sigma_0^{\p})=\lambda^{\p}_{r_0-r_\s}-\lambda_{r_0+1}$ rather than $\delta_0$. We credit the gain of knowledge transfer for PCA to the enlarged eigenvalue gap, or in plain words, the target performance can be enhanced by knowledge transfer, when the task of estimating the private subspace $U_0^{\p}$ alone becomes easier as the eigenvalue gap is much larger. It is helpful to imagine the case when $U_0^{\s}$ has relatively weak signal strength, and the signal corresponding to $U_0^{\p}$ is much stronger. While it is difficult to estimate $U_0^{\s}$ using the target data only, thankfully we are able to take relevant information from the informative sources to acquire a better estimation of $U_0^{\s}$. After the shared subspace $U_0^{\s}$ is acquired, we only need to estimate the private subspace $U_0^{\p}$, which is in this case a much easier task to do.

In the end, we briefly remark on our choice of the Grassmannian barycenter method for the shared subspace estimation. First, recall that the Grassmannian barycenter seeks the leading $r_s$ eigenvectors of the average projection matrix (\ref{eq:aveproj}), which only integrates subspace information across the datasets. The leading eigenvectors of (\ref{eq:aveproj}) turn out to be the most shared directions across $\tilde{P}_k$, which is exactly what we are looking for under our knowledge transfer framework. The method of seeking the shared subspace is certainly not unique, one might also resort to, e.g., the leading $r_s$ eigenvectors of the pooled sample covariance matrix $\sum_{k\in\{0\}\cup\cI} n_k\hat{\Sigma}_k/N_{\cI}$. However, from the arguments above we know that the gain by knowledge transfer is largely due to the stronger signal in the private subspaces, which might in turn jeopardize the estimation of the shared subspace if we simply sum the sample covariance matrices.

One might consider the following toy example, let
$\Sigma^*_0 = \diag(5,2,1,1,1,1)$ and $\Sigma^*_1 = \diag(1,2,5,1,1,1)$. We denote the natural basis by $\{e_i\}$, so here $e_2$ corresponds to the shared subspace, while $e_1$ is private for $\Sigma^*_0$ and $e_3$ is private for $\Sigma^*_1$. The target aim is to retrieve $e_1$ and $e_2$. Even without randomness, the sum of the population covariance matrices only gives  $\Sigma^*_0+\Sigma^*_1 = \diag(6,4,6,2,2,2)$. Hence the shared subspace $e_2$ could not stand out in this way if its eigenvalues are relatively weak in the individual covariance matrices. On the other hand, the Grassmannian barycenter only takes the subspace information into account. Take the leading $2$-dimensional projectors of $\Sigma^*_0$ and $\Sigma^*_1$ as
 $P_0^* = \diag(1,1,0,0,0,0)$ and $P_1^* = \diag(0,1,1,0,0,0)$, clearly $P_0^*+ P_1^* = \diag(1,2,1,0,0,0)$ and $e_2$ leads in this case. In summary, the Grassmannian barycenter method is more suitable for capturing shared, but potentially weak, subspace information. In the end, we would like to point out that the Grassmannian barycenter method also shows computational advantage when the multiple source datasets are distributed across different machines, as our method only need to transmit the subspace estimators, rather than the whole datasets, please refer to \cite{fan2019distributed} for details about the distributed PCA.

\subsection{Unknown informative sources}\label{sec:nokt}
In the last subsection, we present Algorithm \ref{alg:ora} when all informative source datasets are known in advance. According to (\ref{eq:GBoracle}), we first solve the optimization problem
\begin{equation}\label{eq:oraGB2}
    \hat{P}_0^{\s}=\argmax_{P\in \G(p,r_\s)}\frac{1}{N_{\cI}} \sum_{k\in\{0\}\cup\mathcal{I}}n_k\tr(\tilde{P}_kP)
\end{equation}
to acquire the oracle shared subspace estimator, then we fine-tune the result by estimating the private subspace using target data only. In practice,  however, we do not know $\cI$ in advance. Therefore, there is chance that one may also include some non-informative sources into the study and this is particularly likely to happen when we are dealing with a large number of candidate sources, leading to the potential ``negative transfer" problem. As a matter of fact, as will be shown by numerical simulation, the Grassmannian barycenter step in Algorithm \ref{alg:ora} is actually rather robust against mild inclusion of useless source datasets. To see this, note that any individual subspace containing little shared subspace information could be viewed as one private subspace alone, and it does limited damage to the first step of Algorithm \ref{alg:ora} as long as the identifiability of the shared subspace as depicted by Assumption \ref{assum:2} below is still guaranteed. With such robustness in mind, we endeavor to pursue for an algorithm which is not only statistically accurate but also computationally efficient, even when the number of sources are relatively large.

Speaking concisely, the goal of this subsection is to give a reasonable approximation of the oracle shared subspace estimator $\hat{P}_0^{\s}$ when $\cI\subseteq [K]$ is unknown. To do so, first recall from (\ref{eq:inflvl}) that the informative level of the sources could be measured by the following quantity with simple calculation:
\begin{equation}\label{eq:dk}
    d_k:= r_{\s}-\tr(P_k^*P_0^{\s}),\quad \text{such that}\quad 0\leq d_k\lesssim h \quad \text{for}\quad k\in \cI.
\end{equation}
Equivalently, for those informative sources $k\in \cI$, we tend to have large $\tr(P_k^*P_0^{\s})$, which motivates us to consider the following rectified problem
\begin{equation}\label{major solution-sample}
\hat{P}^{\s}_{0,\tau}=\argmax_{P\in \G(p,r_\s)}\frac{1}{N_{[K]}} \left(n_0\tr(\tilde{P}_0 P)+ \sum_{k\in [K]}n_k\max\{\tr(\tilde{P}_kP),\tau\}\right),
\end{equation}
for some $\tau\in [0,r_{\s}]$ and $N_{[K]}:=\sum_{k\in\{0\}\cup[K]}n_k$. If $\tau=0$, then solving (\ref{major solution-sample}) is equivalent to blindly pooling all $k\in \{0\}\cup [K]$ to perform the Grassmannian barycenter step in Algorithm \ref{alg:ora}. If $\tau=r_{\s}$, then the solution of (\ref{major solution-sample}) is only related to the individual PCA estimator $\tilde{P}_0$ using the target dataset. If $\tau\in (0,r_\s)$, then the problem (\ref{major solution-sample}) is a non-convex optimization problem and the quantity $\tau$ controls the strength of dataset selection. Indeed, if $\tr(P_k^*P_0^{\s})$ is large enough for $k\in \cI$ while small enough for all $k\in \cI^c$, we would expect some properly chosen $\tau$ to precisely separate the informative datasets from the non-informative ones with high probability, such that the oracle subspace estimator from (\ref{eq:oraGB2}) could be a local maximum of (\ref{major solution-sample}).

In the remainder of this subsection, we briefly discuss how to numerically search for the local maximum of optimization problem in (\ref{major solution-sample}). First, one could naturally resort to the well-developed manifold optimization toolbox \citep{huper2004newton,helmke2007newton}. For instance, local quadratic convergence for Newton's method on Grassmann manifolds is assured in \cite{helmke2007newton}. However, Newton's method is notoriously sensitive towards initialization, thus a few Grassmannian gradient descent steps are suggested for a warm start.

Given some proper initialization $(P^{\s}_{0,\tau})^{(0)}$, for $t\geq 1$, let $(P^{\s}_{0,\tau})^{(t-1)}$ be the estimator from the $(t-1)$-th step of iteration. For brevity of notations, we omit $(0, \tau,\s)$ and denote $P_t:=(P^{\s}_{0,\tau})^{(t)}$. Given $P_{t-1}$, by observing (\ref{major solution-sample}), we start by selecting the informative datasets in the $t$-th step according to the following criterion:
\begin{equation}\label{eq:itercri}
    \cI_t=\left\{k\in[K]\mid \tr[P_{t-1}\tilde{P}_k]\geq \tau \right\}.
\end{equation}
With (\ref{eq:itercri}) in hand, one can update $P_{t-1}$ to $P_{t}$ by taking either one Grassmannian gradient descent step or one Grassmannian Newton's step \citep{huper2004newton,helmke2007newton}. The details of the Grassmannian optimization are presented in the supplementary material for saving space here. However, while the aforementioned manifold optimization techniques guarantee stable convergence to some local maximum of (\ref{major solution-sample}), with the intuition of the Grassmannian barycenter in mind, one could not help to accelerate the optimization process at the sacrifice of certain stability, and dive into the following rectified Grassmannian K-means procedure. Given $P_{t-1}$ from the $(t-1)$-th step of iteration, after acquiring $
\cI_t$ from (\ref{eq:itercri}), now instead of performing one single Grassmannian gradient descent step or one single Grassmannian Newton's step, it is also appealing to acquire $P_{t}$ directly via the Grassmannian barycenter method using $k\in\{0\}\cup \cI_t$. After that, we could continue using $P_t$ to obtain $\cI_{t+1}$, and iterate this procedure until convergence. This algorithm is clearly in the spirit of the celebrated K-means algorithm, so it is presented in this work in parallel with the manifold optimization method due to its intuitive and efficient nature. As seen in the numerical experiments, the rectified Grassmannian K-means method is able to achieve similar performance as the manifold optimization methods with much fewer iteration steps. In the end, please refer to Algorithm \ref{alg:nora} for details of the knowledge transfer procedures when the informative sources are unknown in advance.

\begin{algorithm}[ht]
\caption{Non-oracle knowledge transfer for principal component analysis.}\label{alg:nora}
\begin{algorithmic}[1]
\REQUIRE ~~\\
  $(\hat{\Sigma}_k, n_k,r_k)$ for $k\in\{0\}\cup [K]$; $r_{\s}$, $\tau$; initialization $P_{0}$;\\
\ENSURE ~~\\
    \STATE individual PCA step: acquire $\tilde{P}_k$ by taking the leading $r_k$ eigenvectors of $\hat{\Sigma}_k$ for $k\in\{0\}\cup [K]$;\\
    \STATE rectified GB step:  given $P_{t-1}$, first acquire $\cI_t$ according to (\ref{eq:itercri}), then obtain $P_{t}$ in one of the following ways: (a) one Grassmannian gradient descent step; (b) one Grassmannian Newton's step; or (c) directly apply the Grassmannian barycenter method using $k\in\{0\}\cup \cI_t$; iterate until convergence to $\hat{P}^{\s}_{0,\tau}$;\\

    \STATE fine-tuning step: given $\hat{P}^{\s}_{0,\tau}$, let $(\hat{P}^{\s}_{0,\tau})^{\perp}=I_p-\hat{P}^{\s}_{0,\tau}$, acquire $\hat{P}_{0,\tau}^{\p}$ by taking the leading $(r_0-r_{\s})$ eigenvectors of $\hat{\Sigma}_{0,\tau}^{\p}=(\hat{P}^{\s}_{0,\tau})^{\perp} \hat{\Sigma}_0 (\hat{P}^{\s}_{0,\tau})^{\perp}$;\\

\RETURN $\hat{P}_{0,\tau}=\hat{P}_{0,\tau}^{\s}+\hat{P}_{0,\tau}^{\p}$.
\end{algorithmic}
\end{algorithm}

\section{Statistical Theory}\label{sec:theo}
This section is devoted to the theoretical statements of our knowledge transfer estimators. We first introduce the following classical PCA setup in Assumption \ref{assum:1}.

\begin{assumption}[Classical PCA Setup]\label{assum:1}
For $k\in\{0\}\cup[K]$, we assume that the datasets are generated independently such that: (a) the $k$-th dataset consists of $n_k$ i.i.d. $p$-dimensional sub-Gaussian random vectors $x_{k,i}$ with mean zero and covariance $\Sigma^*_k$, and the corresponding sample covariance matrix is $\hat{\Sigma}_k=\sum_{i=1}^{n_k} x_{k,i} x_{k,i}^{\top}/n_k$; (b) let $\kappa_k=\lambda_1\left(\Sigma_k^*\right)
/d_{r_k}\left(\Sigma_k^*\right)$ be the conditional number and $e_k=\text{tr}\left(\Sigma_k^*\right)
/\lambda_1\left(\Sigma_k^*\right)$ be the effective rank of $\Sigma_k^*$, assume that $d_{r_k}(\Sigma^*_k)\geq c$ for some $c>0$, while $n_k\geq Ce_k$ for some $C>0$ as $n_k$, $p\rightarrow \infty$ .
\end{assumption}
The statistical setup in Assumption \ref{assum:1} is standard in the literature, see for example \cite{fan2019distributed} and \cite{he2022distributed}. We present the following Lemma \ref{lemma:indPCA} concerning estimation error of the $k$-th PCA study using its own dataset only, which is merely a re-statement of Lemma 1 from \cite{fan2019distributed} using the notations of this work.


\begin{lemma}[Individual PCA error]\label{lemma:indPCA}
     Under Assumption \ref{assum:1}, let $\Delta_k = \tilde{P}_k-P_k^*$, we have
    $$\left\|\left\|\Delta_k\right\|_F\right\|_{\psi_1}\lesssim \tilde{n}_k^{-1/2}, \quad \text{such that}\quad \tilde{n}_k= \frac{n_k}{\kappa_k^2r_ke_k}.$$
    as $n_k$, $p\rightarrow \infty$, where $\tilde{n}_k$ is called the effective sample size of the $k$-th PCA study.
\end{lemma}

The effective sample size $\tilde{n}_k$ is an important quantity in this section, as it takes both the sample size $n_k$ and the difficulty of estimating the $k$-th study into account. Indeed, as will be clear by the following theoretical arguments, it is better if we calculate the (rectified) GB step in both Algorithms \ref{alg:ora} and \ref{alg:nora} using the effective sample size $\tilde{n}_k$, rather than $n_k$. Intuitively, the Grassmannian barycenter could be viewed as the weighted-averaging of the individual PCA estimators $\tilde{P}_k$, and it is natural to assign more weights to those studies that are easier to estimate. Hence, in the remainder of this section, we would use $\tilde{n}_k$ to substitute $n_k$ in both Algorithms \ref{alg:ora} and \ref{alg:nora}, which also provides a clearer picture on the gain of knowledge transfer, as the resulting convergence rate is eventually related to the pooled effective sample size $\tilde{N}_{\cI}=\sum_{k\in \{0\}\cup\cI}\tilde{n}_k$.

The remainder of this section is organized as follows. In Section \ref{sec:tokt} we present theoretical results concerning the oracle knowledge transfer estimator from Algorithm \ref{alg:ora}, where the asymptotic normality of its bilinear forms is also established. Section \ref{sec:tnokt} justifies Algorithm \ref{alg:nora} in terms of  finding the local maximum of the rectified optimization problem (\ref{major solution-sample}). Finally, we briefly state the extension of our knowledge transfer algorithms to the elliptical PCA setup in Section \ref{sec:tepca}.

\subsection{Oracle knowledge transfer}\label{sec:tokt}

We first discuss the oracle transfer case when the informative source datasets, namely those $k\in\cI\subseteq [K]$, are known in advance. Recall that in Algorithm \ref{alg:ora}, we first integrate the shared subspace information across the target and useful source datasets, then we fine-tune the shared subspace estimator using only the target dataset. To ensure the statistical performance of the shared subspace estimator in the first step, we need the following Assumption \ref{assum:2}.

\begin{assumption}[Identifiability of the shared subspace]\label{assum:2}
Given the decomposition $P^*_k=P_k^{\s}+P_k^{\p}$ for $k\in\{0\}\cup\cI$. For the shared subspace information, we assume that $\|P_k^{\text{s}}-P_0^{\text{s}}\|_F\leq h$ for $k\in \cI$. As for the private information, we assume that there exists some constant $g>0$ such that
$$\frac{1}{\tilde{N}_{\cI}}\left\|\sum_{k\in \{0\}\cup\cI} \tilde{n}_k P_k^{\text{p}} \right\|_{2}\leq 1-g.$$
\end{assumption}

Assumption \ref{assum:2} essentially requires that the private projection matrices do not jointly contribute to some large eigenvalue. Note that $\|\sum_{k\in \{0\}\cup\cI} \tilde{n}_k P_k^{\text{p}}\|_{2}/\tilde{N}_{\cI}\leq 1$, and the equation holds only when there exists a vector $u$ such that $u\in \span(U_k^{\p})$ for all $k\in \{0\}\cup\cI$. In this case however, $u$ shall be included into the shared subspace as well. In essence, Assumption \ref{assum:2} justifies the procedure of taking the leading $r_\s$ eigenvectors of the average projection matrix in the Grassmannian barycenter step of Algorithm \ref{alg:ora}. Similar but often more restrictive identifiability assumptions could be made if different methods are used for retrieving the shared subspace. For instance, if it's the case that the leading $r_s$ eigenvectors of the pooled covariance matrix $\sum_{k\in\{0\}\cup\cI} n_k\Sigma_k^*/N_{\cI}$  form the shared subspace, the eigenvalues corresponding to the private subspaces shall not be overwhelmingly larger than those corresponding to the shared subspace as alluded previously. On the contrary, the Grassmannian barycenter method discards the eigenvalue information and only integrates the directional information, so the eigenvalues corresponding to the shared subspace naturally leads in the averaged projection matrix.

In the following, we establish the convergence rates of the oracle knowledge transfer estimator $\hat{P}_0= \hat{P}_0^{\s}+\hat{P}_0^{\p}$ from Algorithm \ref{alg:ora}. First denote
\begin{equation}\label{eq:s}
    s:=\tilde{N}_{\cI}^{-1/2}+\left(\sum_{k\in\{0\}\cup\cI}r_k^{1/2}\right)\tilde{N}_{\cI}^{-1}+h,
\end{equation} and let $\eta := \|(U^{\p}_0)^{\top}\Sigma_0^*U^{\s}_0\|_2$.
Recall from (\ref{eq:eigen}) that $\delta_0=d_{r_0}(\Sigma^*_0)=\lambda_{r_0}-\lambda_{r_0+1}$ and $\delta_{\p}=d_{r_0-r_\s}(\Sigma_0^{\p})=\lambda^{\p}_{r_0-r_\s}-\lambda_{r_0+1}$, where we set $\lambda^{\p}_{0}:=\infty$ by convention.We present the main results in the following Theorem \ref{theo:main}.

\begin{theorem}[Oracle knowledge transfer]\label{theo:main}
 Under Assumptions \ref{assum:1} and \ref{assum:2}, if we further assume that $\eta/\delta_{\p}+(\|\Sigma_0^*\|_2/\delta_{\p})^2s =O(1)$, while $\|\Sigma_0^*\|_2 s = o(\delta_\p)$ as $\min_{k\in \{0\}\cup\cI}(n_k), p\rightarrow \infty$ and $h\rightarrow 0$, the oracle knowledge transfer estimator then satisfies
 \begin{equation}\label{rate:main}
    \left\|\hat{P}_0-P_0^{*}\right\|_F=O_p\left(\frac{\delta_0}{\delta_{\p}}\tilde{n}_0^{-1/2}+s\right).
 \end{equation}

\end{theorem}

We make a few remarks on the results in Theorem \ref{theo:main}. First, the quantity $s=\tilde{N}_{\cI}^{-1/2}+\left(\sum_{k\in\{0\}\cup\cI}r_k^{1/2}\right)\tilde{N}_{\cI}^{-1}+h$ depicts the convergence rate of the shared subspace estimator. According to Lemma \ref{lemma:indPCA}, the PCA estimator using target dataset alone has the convergence rate of $\tilde{n}_0^{-1/2}$, where $\tilde{n}_0$ is the target effective sample size. For the idea of knowledge transfer to work, the lowest requirement is that $s\lesssim \tilde{n}_0^{-1/2}$. While the first term $ \tilde{N}_{\cI}^{-1/2}$ and the third term $h$ are standard in the transfer learning literature \citep{li2020transfer,tian2022transfer,transquantile}, the second term $(\sum_{k\in\{0\}\cup\cI}r_k^{1/2})\tilde{N}_{\cI}^{-1}$ arises from the high-order bias term using the Grassmannian barycenter method \citep{fan2019distributed}. As we assume $r_k\leq r_{\max}<\infty$, the number of the informative source studies is allowed to diverge as long as $|\cI|\lesssim \tilde{N}_{\cI}\tilde{n}_0^{-1/2}$, which offers more freedom when modeling knowledge transfer across multiple studies. In fact, if we instead perform PCA on the pooled dataset to estimate the shared subspace  in the first step, this second high-order bias term shall vanish in principle \citep{fan2019distributed,he2022distributed}. However, we in turn need to strengthen Assumption \ref{assum:2} for the identifiability of the shared subspace as remarked thereafter. In this work, we suggest using the Grassmannian barycenter method due to its statistical and computational advantages in various settings.
Second, while $s$  could be arbitrarily small given a reasonable number of informative source datasets with sufficiently large $\tilde{n}_k$ and sufficiently small $h$, the estimation performance of $\hat{P}_0$ also relies on the first term $\delta_0\tilde{n}_0^{-1/2}/\delta_{\p}$, which corresponds to the private subspace estimation as discussed right after the introduction of Algorithm \ref{alg:ora}. When $r_\s=r_0$ and the target study has no private subspace, this first term is $0$ and the gain by knowledge transfer is trivial. If $r_\s<r_0$, the knowledge transfer estimator $\hat{P}_0$ shows great advantage when $\delta_0\ll \delta_{\p}$ and $s\ll \tilde{n}_0^{-1/2}$. Ideally, one might imagine $\span(U_0^{\p})$ corresponds to some eigenvalues that diverge faster than those corresponding to $\span(U_0^{\s})$.
In the end, the quantity $\eta = \|(U^{\p}_0)^{\top}\Sigma_0^*U^{\s}_0\|_2\leq \|\Sigma_0^*\|_2$ measures the ``angle" of $U^{\p}_0$ and $U^{\s}_0$ from the invariant subspaces of $\Sigma_0^*$. For instance, assume that $\Sigma_0^*$ has no multiple eigenvalues, let $U^{\p}_0$ consist of the leading $(r_0-r_\s)$ eigenvectors of $\Sigma_0^*$, and let $U^{\s}_0$ consist of the subsequent $r_\s$ eigenvectors, then $\eta=0$ naturally. Problem occurs when $\eta\neq 0$, then the estimation error of $\span(U^{\s}_0)$ is likely to do more damage to the estimation  of $\span(U^{\p}_0)$ when projecting $\Sigma_0^*$ to $(\hat{P}_0^{\s})^{\perp}\Sigma_0^*(\hat{P}_0^{\s})^{\perp}$, rather than to $(P_0^{\text{s}})^{\perp} \Sigma^*_0 (P_0^{\text{s}})^{\perp}$ in the fine-tuning step of Algorithm \ref{alg:ora}. Yet this problem is still well-controlled as long as $\|\Sigma_0^*\|_2=O(\delta_{\p})$, which is rather mild.

\subsubsection{Asymptotic normality of bilinear forms}
This section is devoted to the asymptotic behaviors of the bilinear forms with respect to the oracle knowledge transfer subspace estimator, namely $\langle u, \hat{P}_0 v\rangle$ for some $u$, $v\in S^{p-1}$, where $S^{p-1}$ is the $(p-1)$-dimensional sphere. In \cite{koltchinskii2016asymptotics} the authors study the asymptotics of the bilinear forms of the empirical spectral projectors given Gaussian distributed data, and we aim to present the knowledge transfer analogue of their results under the settings of this work, please also refer to \cite{bao2022statistical} for a thorough discussion on the asymptotic behaviors of the extreme eigenvectors under fully general assumptions. While we do not restrict ourselves to the Gaussian distribution only, we still need to add some model structures to Assumption \ref{assum:1} due to technical reasons. In particular, as will be stated clearly in the following Corollary \ref{coro:AN}, we assume that the data vectors are only linearly dependent. Such linearly dependent model structure is common in the random matrix literature, see for instance \cite{bai2010spectral,wang2017asymptotics,bao2022statistical}.

Before we introduce the formal result, we define a few quantities that will appear in the variance of the asymptotically normal distribution. Given $u$, $v\in S^{p-1}$, we recall $\mathcal{U}^{\p}_0= (\{u_i^{\p}\}_{i=1}^{r_0-r_\s}\mid\{u_j\}_{j=r_0+1}^{p}\mid \{u_i^{\s}\}_{i=1}^{r_\s})$ from (\ref{eq:eigen}) and define
$$\omega_{ij}:=
\frac{\rho_{ij}(\lambda_i^{\p}\lambda_j)^{1/2}}{\lambda_i^{\p}-\lambda_j},\quad \rho_{ij}:=\left[(\mathcal{U}^{\p}_0)^{\top}u\right]_i\left[(\mathcal{U}^{\p}_0)^{\top}v\right]_j+\left[(\mathcal{U}^{\p}_0)^{\top}u\right]_j\left[(\mathcal{U}^{\p}_0)^{\top}v\right]_i,$$
for $i\leq r_0-r_\s$ and $r_0+1\leq j\leq p$. Here the subscript $i$ means the $i$-th element of the vector.

\begin{corollary}[Bilinear forms]\label{coro:AN}
     Under Assumptions \ref{assum:1} and \ref{assum:2}, if we further set $x_{0,i}:=(\Sigma_0^{*})^{1/2}z_{i}$, where $z_{i}$ consists of $p$ i.i.d. random variables $\{z_{i,m}\}_{m=1}^{p}$ such that $\E(z_{i,m})=0$, $\E(z^2_{i,m})=1$ and $\E(z^4_{i,m})=\nu_4$, while $(\delta_0/\delta_{\p})\tilde{n}_0^{-1/2}=o(1)$ as $\min_{k\in \{0\}\cup\cI}(n_k), p\rightarrow \infty$ and $h\rightarrow 0$, we have
     $$\left\langle u, \left(\hat{P}_0-P_0^*\right) v\right\rangle=L_0+R_0,\quad \text{where}\quad R_0=O_p\left[\left(\frac{\delta_0}{\delta_{\p}}\right)^2\tilde{n}^{-1}_0+\frac{\|\Sigma_0^*\|_2}{\delta_{\p}}s\right],$$
    $$n_0^{1/2}L_0/\sigma_{\p}\stackrel{\text{d}}{\longrightarrow}N(0,1),\quad\text{when}$$
     $$\sigma^2_{\p}:=\sum_{i=1}^{r_0-r_{\s}}\sum_{j=r_0+1}^{p}\omega^2_{ij}+(\nu_4-3)\sum_{m=1}^{p}\left(\sum_{i=1}^{r_0-r_{\s}}\sum_{j=r_0+1}^{p}\omega_{ij}(u_i^{\p})_m(u_j)_m\right)\neq 0.$$
\end{corollary}

First, assuming the linear structure such that $x_{0,i}=(\Sigma_0^{*})^{1/2}z_{i}$ as in \cite{bao2022statistical} is mainly for a simpler expression of $\sigma_{\p}^2$, and Corollary \ref{coro:AN} extends readily to $x_{0,i}=A z_{i}$ such that $\Sigma_0^*=AA^{\top}$. Second, we could see that the remainder term of the knowledge transfer bilinear forms is of order $(\delta_0/\delta_{\p})^2\tilde{n}_0^{-1}+\|\Sigma_0^*\|_2s/\delta_{\p}$. On the other hand, if we use the target sample covariance directly, the remainder term shall be of order $\tilde{n}_0^{-1}$. Hence, the asymptotic normality requires weaker eigenvalue gap conditions on $\Sigma_0^*$ by replacing $\delta_0$ by $\delta_{\p}$ via knowledge transfer as long as $s$ given in (\ref{eq:s}) is sufficiently small. Moreover, the variance of the asymptotically normal term is also smaller after knowledge transfer. Similar gain in statistical inference by transfer learning is observed in \cite{li2023estimation}, as their debiased estimator achieves asymptotic normality under weaker sparsity conditions on the target parameters when the source studies are sufficiently informative. In the end, although the asymptotic normality of the bilinear form could be better guaranteed via knowledge transfer, compared with the results derived under Gaussian distribution, the variance $\sigma_{\p}^2$ is related to the fourth moment of the underlying distribution under the already simplified linearly dependent model, making it harder for further statistical inference. Even if the data is Gaussian distributed ($\nu_4=3$) and the second term vanishes, $\sigma_{\p}^2$ still depends on the unknown parameters $\Sigma_0^{\p}$, which is challenging, if not impossible, to estimate in practice. To deal with this type of problem,
the idea of splitting the dataset into different parts and then doing estimation separately is proposed in \cite{koltchinskii2017new}, while \cite{naumov2019bootstrap,silin2020hypothesis} suggest using bootstrapping methods for building confidence intervals and performing hypothesis testing instead. Overall, discussion on statistical inference of the empirical spectral projectors is still relatively scarce in the literature, and we leave the remaining problems for future pursuit.

\subsection{Non-oracle knowledge transfer}\label{sec:tnokt}
We then turn to the non-oracle transfer problem (\ref{major solution-sample}) when the informative datasets are initially unknown. The main claim of this subsection is Theorem \ref{theo:nora} which shows that the oracle shared subspace estimator from (\ref{eq:oraGB2}) is also a local maximum of (\ref{major solution-sample}) with high probability. It justifies our attempt of searching for the local maximum of (\ref{major solution-sample}) in Algorithm \ref{alg:nora}. We need the additional Assumption \ref{assum:3} as follows.

\begin{assumption}[Separable non-informative sources]\label{assum:3}
Let $d_k= r_{\s}-\tr(P_k^*P_0^{\s})$, for those $k\in \cI^c$, assume that $d_k\geq d_\tau$ for some $d_\tau>0$.
\end{assumption}

Recall from (\ref{eq:dk}) that we have $0\leq d_k\lesssim h$ for $k\in\cI$ under Assumption \ref{assum:2}, so Assumption \ref{assum:3} ensures separability between the informative and non-informative source studies. It is worth mentioning that here $d_\tau$ is not necessarily treated as a constant, but is also allowed to tend to zero as $\min_{k\in \{0\}\cup[K]}(n_k), p\rightarrow \infty$ and $h\rightarrow 0$.

\begin{theorem}[Non-oracle knowledge transfer]\label{theo:nora}
    Under Assumptions \ref{assum:1}, \ref{assum:2} and \ref{assum:3} hold, if $s+\max_{k\in \cI^c}(\tilde{n}_k^{-1/2})=o(d_\tau)$ as $\min_{k\in \{0\}\cup[K]}(n_k), p\rightarrow \infty$ and $h\rightarrow 0$, set $\tau = r_\s-d_\tau/2$, then the oracle shared subspace estimator $\hat{P}_0^{\s}$ from (\ref{eq:oraGB2}) is also a local maximum of (\ref{major solution-sample}) with probability tending to 1.
\end{theorem}

Roughly speaking, $d_{\tau}$ characterizes the minimal deviation of $P_k^*$ away from $P_0^{\s}$ for $k\in \cI^c$ such that we are able to precisely separate $\cI^c$ from $\cI$ using Algorithm \ref{alg:nora}. Interestingly, as the criterion of dataset selection is basically choosing those larger $\tr[\hat{P}^{\s}_{0}\tilde{P}_k]$, we only require $s+\max_{k\in \cI^c}(\tilde{n}_k^{-1/2})=o(d_\tau)$ and in principle $\tilde{n}_0$ is allowed to be relatively small. However, such irrelevance of the target effective sample size comes along with the price of non-convexity in the optimization problem, and multiple initialization attempts are suggested when performing Algorithm \ref{alg:nora}. In the end, we would like to re-emphasize that Algorithm \ref{alg:nora} is designed to be computationally feasible even when the number of sources are large, and it might not be as statistically accurate as, e.g., the model selection aggregation method in \cite{li2020transfer} under certain cases. However, Algorithm \ref{alg:nora} is clearly more computationally friendly and manages to achieve satisfying statistical performance since the Grassmannian barycenter method is inherently robust against mild inclusion of non-informative sources, which is supported by extensive numerical experiments and real data analysis in the following sections.

\subsection{Extension to elliptical PCA}\label{sec:tepca}
As alluded in Section \ref{sec:intro}, the knowledge transfer methods in this article adapts readily to more general settings. To be more specific, $\Sigma^*_k$ and $\hat{\Sigma}_k$ in Algorithms \ref{alg:ora} and \ref{alg:nora} are not necessarily classical population and sample covariance matrices. Instead, they could be population and sample covariance-like matrices from other PCA settings. For illustration, here we briefly remark on the possible extension to elliptical PCA \citep{Fan2018LARGE,han2018eca,he2022large}.

Now we assume that the $k$-th dataset includes $n_k$ i.i.d. $p$-dimensional elliptical random vectors $x_{k,i}$ with zero mean such that $x_{k,i} = r_{k,i} A_k u_{k,i}$, where $u_{k,i}$ is uniformly distributed on the unit sphere $S^{p-1}$ and is independent of the positive random radius $r_{k,i}$. Here $A_k\in\mathbb{R}^{p\times q_k}$ is a deterministic matrix satisfying $\Sigma_k=A_kA_k^{\top}$ and $\Sigma_k$ is called the scatter matrix of $x_{k,i}$. Multivariate elliptical distribution is often used to model heavy-tailed data commonly observed in, e.g., financial markets \citep{han2018eca,Fan2018LARGE}. When $r_{k,i}$ is from heavy-tailed distribution and classical PCA is no longer promising, we turn to the population and sample versions of the spatial Kendall's $\tau$ matrix to perform elliptical PCA, which are defined separately as follows:
$$\Sigma^*_k=\E_{i\neq j}\left[\frac{(x_{k,i}-x_{k,j})(x_{k,i}-x_{k,j})^{\top}}{\|x_{k,i}-x_{k,j}\|_2^2}\right],\quad \hat{\Sigma}_k=\frac{2}{n_k(n_k-1)}\sum_{i<j}\frac{(x_{k,i}-x_{k,j})(x_{k,i}-x_{k,j})^{\top}}{\|x_{k,i}-x_{k,j}\|_2^2}.$$

According to Proposition 2.1 in \cite{han2018eca}, the leading $r_k$ eigenvectors of $\Sigma_k$ and $\Sigma_k^*$ are identical, up to some orthogonal transformation, with the same descending order of the corresponding eigenvalues under mild conditions. In fact, if $\text{rank}(\Sigma_k)=q_k$, we have
$$\lambda_j(\Sigma_k^*)=\E\left(\frac{\lambda_j(\Sigma_k)Y_j^2}{\sum_{i=1}^{q_k}\lambda_i(\Sigma_k)Y_i^2}\right),$$
where $(Y_1,\dots,Y_{q_k})^{\top}\sim N_{q_k}(0,I_{q_k})$ is a standard Gaussian random vector. Thus estimating the leading eigenspace of $\Sigma_k^*$ is equivalent to that of the scatter matrix $\Sigma_k$. We introduce the elliptical PCA setup in the following Assumption \ref{assum:4}.

\begin{assumption}[Elliptical PCA]\label{assum:4}
For $k\in\{0\}\cup[K]$, we assume that the datasets are generated independently such that: (a) the $k$-th dataset consists of $n_k$ i.i.d. elliptical random samples $\{x_{k,i}\}_{i=1}^{n_k}\in\mathbb{R}^p$ with zero mean and scatter matrix $\Sigma_k$; (b) the ratio of the maximal eigenvalue to minimal eigenvalue for the scatter matrix satisfies $ \lambda_1(\Sigma_k)/\lambda_p(\Sigma_k)\leq Cp^{\alpha}$ for some constant $C>0$ and fixed $0<\alpha<1/2$; (c) let $\kappa_k=\lambda_1\left(\Sigma_k\right)/d_{r_k}\left(\Sigma_k\right)$ be the conditional number and $e_k=\text{tr}\left(\Sigma_k\right)
/\lambda_1\left(\Sigma_k\right)$ be the effective rank of $\Sigma_k$, assume that $d_{r_k}(\Sigma_k)\geq c$ for some $c>0$, while $\log p = o(n_k)$ as $n_k$, $p\rightarrow \infty$.
\end{assumption}

\begin{corollary}[Extension to elliptical PCA]\label{coro:ext}
    If we replace Assumption \ref{assum:1} by Assumption \ref{assum:4} in Lemma \ref{lemma:indPCA} and Theorems \ref{theo:main}, \ref{theo:nora}, these results still hold by changing the effective sample size to $\tilde{n}_k= n_k/(\kappa^2_kr_k e_k^2 \log p)$.
\end{corollary}

As stated in Corollary \ref{coro:ext}, given results from \cite{he2022distributed}, it is not hard to verify that almost all theoretical arguments in this section, except Corollary \ref{coro:AN}, directly follows. Another exception is that while under the classical PCA setting we need $\delta_{\p}$ diverge faster than $\delta_0$ to ensure the statistical supremacy of the knowledge transfer estimator, under the elliptical PCA setting $\delta_{\p}$ is naturally bounded above by $1$. It is actually $\delta_0$ that might tend to $0$ in this case. Indeed, as shown in Lemma A.2 of \cite{he2022large}, the term $\delta_0^{-1}$ is allowed to diverge at a rate no faster than $p^{1+\alpha}$ for the fixed $0<\alpha<1/2$, making some room for $\delta_{0}\ll \delta_{\p}$.

\section{Numerical Simulation}\label{sec:num}
In this section, we conduct numerical experiments to validate our theoretical claims and provide empirical evidence on the usefulness of the proposed methods. In Section \ref{sec:mc}, we compare the empirical performance of the estimators from Algorithms \ref{alg:ora} and \ref{alg:nora} with some competitors. Then, the statistical rates in Theorem \ref{theo:main} and the asymptotic normality claimed by Corollary \ref{coro:AN} are verified in Sections \ref{sec:nsr} and \ref{sec:nan}, respectively.

\subsection{Comparisons of various methods}\label{sec:mc}
We generate the datasets as follows: (\romannumeral 1) for the target dataset, we first take the $r_\s$-dimensional $\span(U_0^{\s})$ uniformly from $\cG(p,r_\s)$, i.e., with respect to the Haar measure, where $U_0^{\s}$ is column orthogonal; (\romannumeral 2) for the useful dataset $k\in \cI$, we first acquire an orthogonal matrix $Q_k$ close to $I_p$ by taking QR decomposition of the perturbed identity matrix $I_p+N_k=Q_kR_k$ via the Gram–Schmidt procedure, where the elements of $N_k$ are taken independently from $N(0,h^2/p)$, then we set $U_k^{\s}=Q_k U_0^{\s}$; (\romannumeral 3) as for the useless dataset $k\in [K]\setminus\cI$, we have the following two settings: 1. we independently take $\span(U_k^{\s})$ uniformly from $\cG(p,r_\s)$; 2. we first uniformly generate the $p\times r_\s$ column orthogonal matrix $V_0^s$ as in (\romannumeral 1), which serves as an alternative subspace barycenter, then we have $U_k^{\s}=Q_k V_0^{\s}$ as in (\romannumeral 2).

After obtaining $U_k^{\s}$, we independently take the $p\times (r_k-r_\s)$ column orthogonal matrix $U_k^{\p}$ uniformly from the orthogonal complement of $\span(U_k^{\s})$ for $k\in\{0\}\cup [K]$. Then, we generate $\Sigma_k^*$ using $U_k^{\s}$ and $U_k^{\p}$. The $r_\s$ eigenvalues of $\Sigma_k^*$ corresponding to $U_k^{\s}$ are set identically as $\lambda_k^\s$, the $(r_k-r_\s)$ eigenvalues corresponding to $U_k^{\p}$ are set identically as $\lambda_k^\p$, while the rest $(p-r_k)$ eigenvalues are set identically as $1$. In the end, for the $k$-th dataset, we independently generate $n_k$ centered data $\{x_{k,i}\}_{i=1}^{n_k}$ with covariance matrix $\Sigma_k^*$ under either the multivariate Gaussian or the multivariate $t_3$ distribution.

Recall that the target aim is $\span(U_0^*)= \span(U_0^{\s})\oplus \span(U_0^{\p})$, we compare the following $4$ estimators of $\span(U_0^*)$: (a) blind GB estimator which directly applies Algorithm \ref{alg:ora} as if all datasets are informative; (b) blind PCA estimator which first blindly pools all datasets in $\{0\}\cup [K]$, and then perform PCA with dimension $r_0$; (c) selected GB estimator by applying Algorithm \ref{alg:nora} with the tuning parameter $\tau$; and (d) individual PCA estimator acquired by performing PCA on only the target dataset with dimension $r_0$. By large unreported simulation results, we find that all versions of Algorithm \ref{alg:nora} yield similar results, so we only report the results corresponding to the computationally more  friendly rectified Grassmannian K-means procedure for saving space. To deal with the heavier-tailed $t_3$ distribution, we perform both the classical PCA and the elliptical PCA \citep{Fan2018LARGE, he2022large}. To measure the distances between the estimated $\hat{U}_0$ and the target $U_0^*$, we adopt the popular scaled projection metric as in \cite{Yu2021Projected}, which is defined as
\begin{equation*}
	\mathcal{D}(\hat{U}_0,U_0^*)=\left(1- \frac{\tr(P_{\hat{U}_0}P_0^*)}{p_0}\right)^{1/2}= \frac{1}{(2p_0)^{1/2}}\|P_{\hat{U}_0}-P_0^*\|_F,
\end{equation*}
where $P_{\hat{U}_0}=\hat{U}_0\hat{U}_0^{\top}$. So $\mathcal{D}(\hat{U}_0,U_0^*)$ is always between $0$, corresponding to $\span(\hat{U}_0)= \span(U_0^*)$, and $1$, corresponding to $\span(\hat{U}_0)$ and $\span(U_0^*)$ are orthogonal. For $K\in\{6,12,18,24,30\}$, we consider the following 3 scenarios.
\begin{itemize}
    \item (S1) a target dataset with $K$ informative datasets from (\romannumeral 2).
    \item  (S2) a target dataset with $K/2$ informative datasets from (\romannumeral 2) and $K/2$ useless datasets from (\romannumeral 3).1.
    \item  (S3) a target dataset with $K/2$ informative datasets from (\romannumeral 2) and $K/2$ useless datasets from (\romannumeral 3).2.
\end{itemize}

\begin{figure}[H]
	\centering

    \begin{minipage}{0.28\linewidth}
		\centering
	\includegraphics[width=0.91\linewidth]{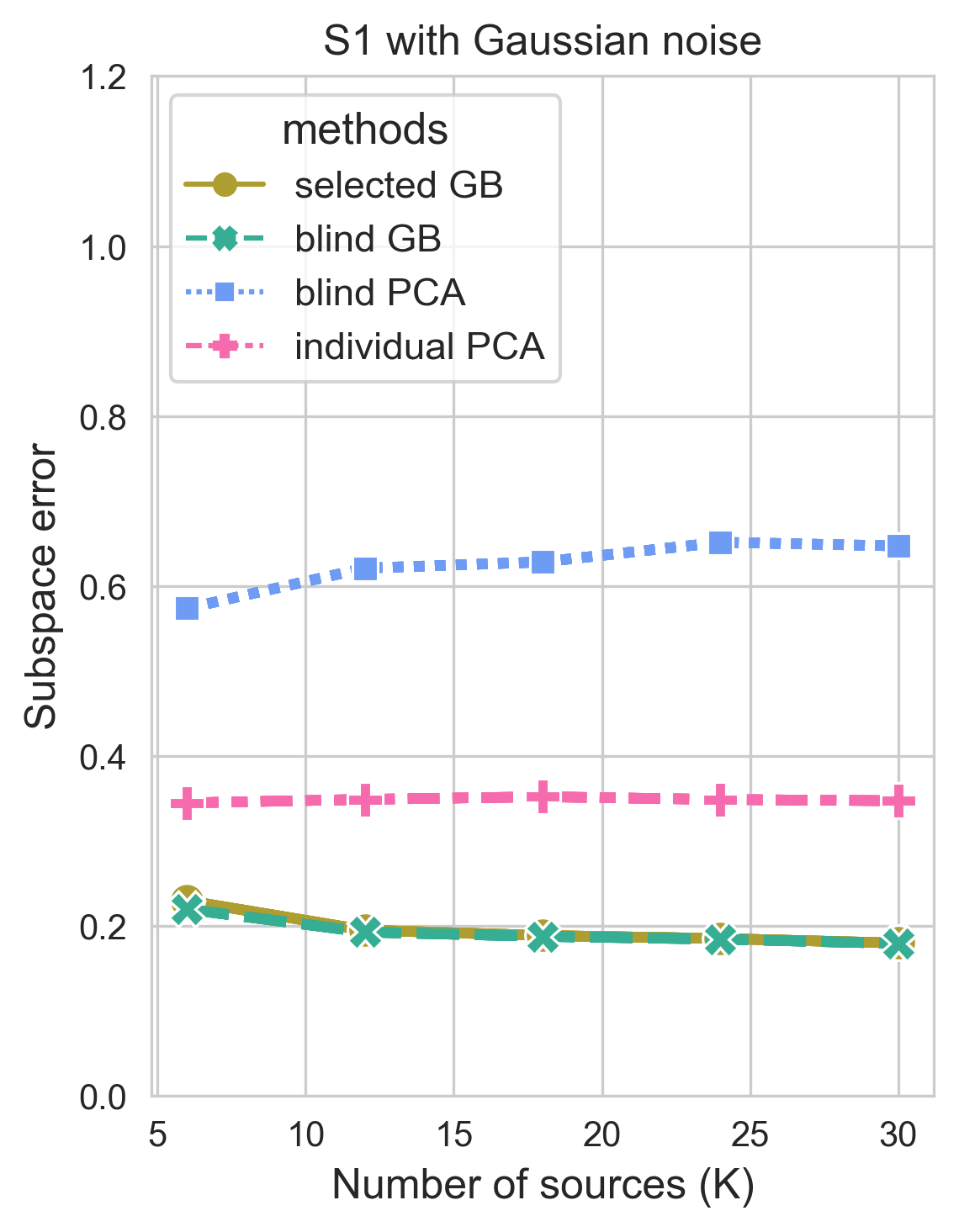}
	\end{minipage}
        \begin{minipage}{0.28\linewidth}
		\centering
	\includegraphics[width=0.91\linewidth]{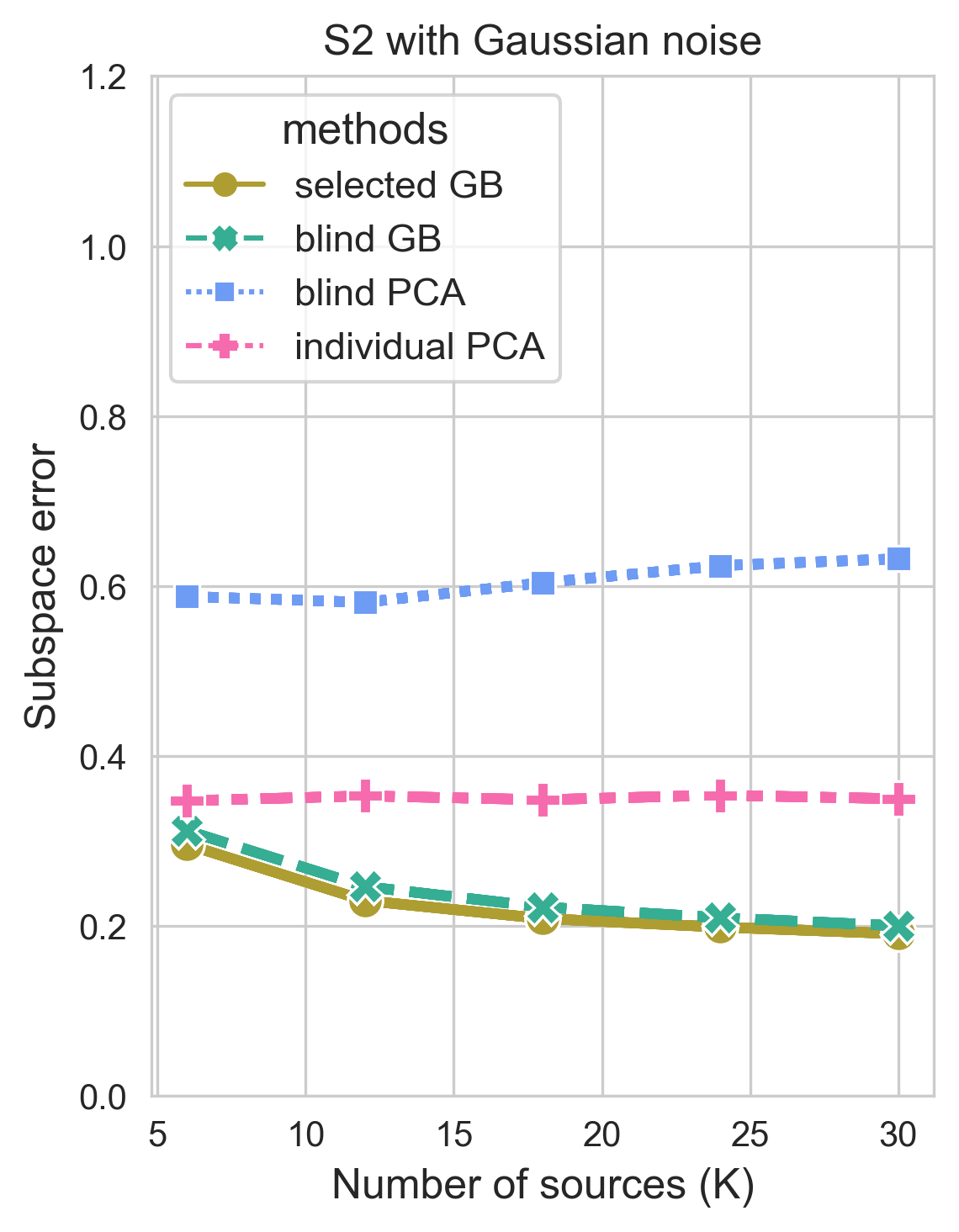}
	\end{minipage}
	\begin{minipage}{0.28\linewidth}
		\centering
	\includegraphics[width=0.91\linewidth]{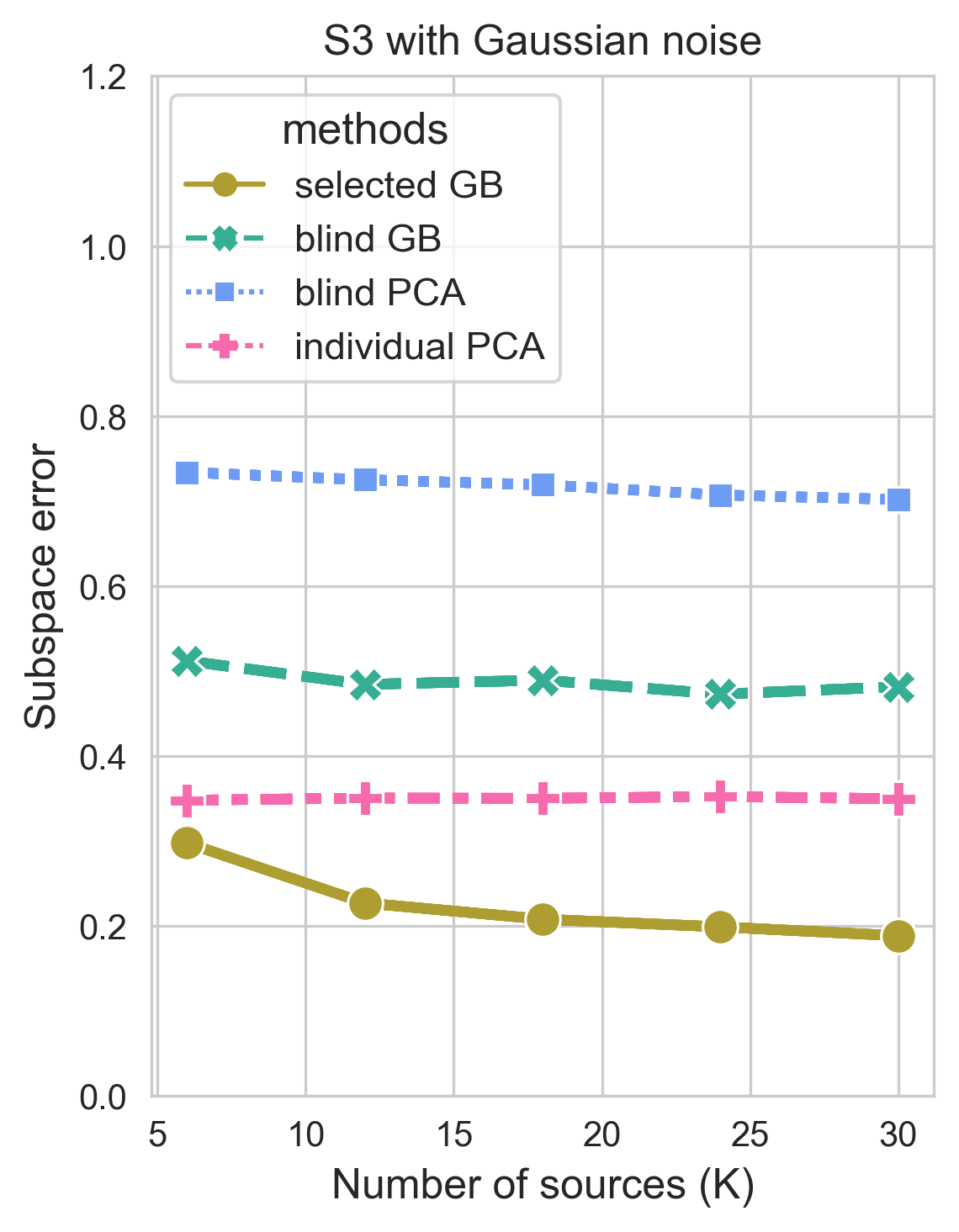}
	\end{minipage}
\caption{\label{Fig-sim-CN} Average subspace estimation error using various methods under Gaussian distribution with classical PCA, based on 100 replications. From left to right we report S1 (no inclusion of useless sources), S2 (mild inclusion of useless sources) and S3 (severe inclusion of useless sources), respectively.
}
\end{figure}

We set $p=50$, $n_k=100$, $r_k=4$, for $k\in \{0\}\cup [K]$, let $r_\s=2$, $h=0.05$, set $\lambda_0^{\p}=10$, while $\lambda_k^{\p}=8$ for $k\in[K]$. In the end, set $\lambda_k^{\s}=4$ for $k\in \{0\}\cup [K]$. The tuning parameter $\tau$ in Algorithm \ref{alg:nora} is set to be $0.5$, and we iterate the rectified Grassmannian K-means procedure $T=10$ times. Figure \ref{Fig-sim-CN} reports the average estimation error of $\span(U_0^*)$ by various subspace estimators using classical PCA under Gaussian distribution, based on $100$ replications. First, in S1 where all source datasets are informative, we witness the gain by knowledge transfer using both Algorithms \ref{alg:ora} and \ref{alg:nora}. However, the blind PCA method which directly pools all data and then perform PCA causes negative-transfer even under this oracle setting. The reason for this is that the target private subspace information easily drowns in the pooling procedure, indicating that the second fine-tuning step is necessary for the sake of successful knowledge transfer. Meanwhile, we impose different types of non-informative datasets in S2 and S3. In S3 the useless datasets have a mutual centered subspace different from the shared subspace that the target study desires, while in S2 there is no such center. We can see that the blind GB estimator which directly applies Algorithm \ref{alg:ora} on all datasets still yields satisfying performance in S2, demonstrating the robustness of the Grassmannian barycenter method against mild inclusion of some useless datasets. On the other hand, including those datasets with another subspace center in S3 causes negative impact to all other estimators except the selected GB estimator by Algorithm \ref{alg:nora}. The comparison between these scenarios further justifies Algorithm \ref{alg:nora}'s attention on computational feasibility rather than the statistical accuracy under non-oracle cases, since Algorithm \ref{alg:ora} is already robust against mild inclusion of useless datasets as in S2, while Algorithm \ref{alg:nora} is fully capable of tackling multi-center cases as in S3 with little extra computational burden.

\begin{figure}[H]
	\centering

    \begin{minipage}{0.28\linewidth}
		\centering
	\includegraphics[width=0.91\linewidth]{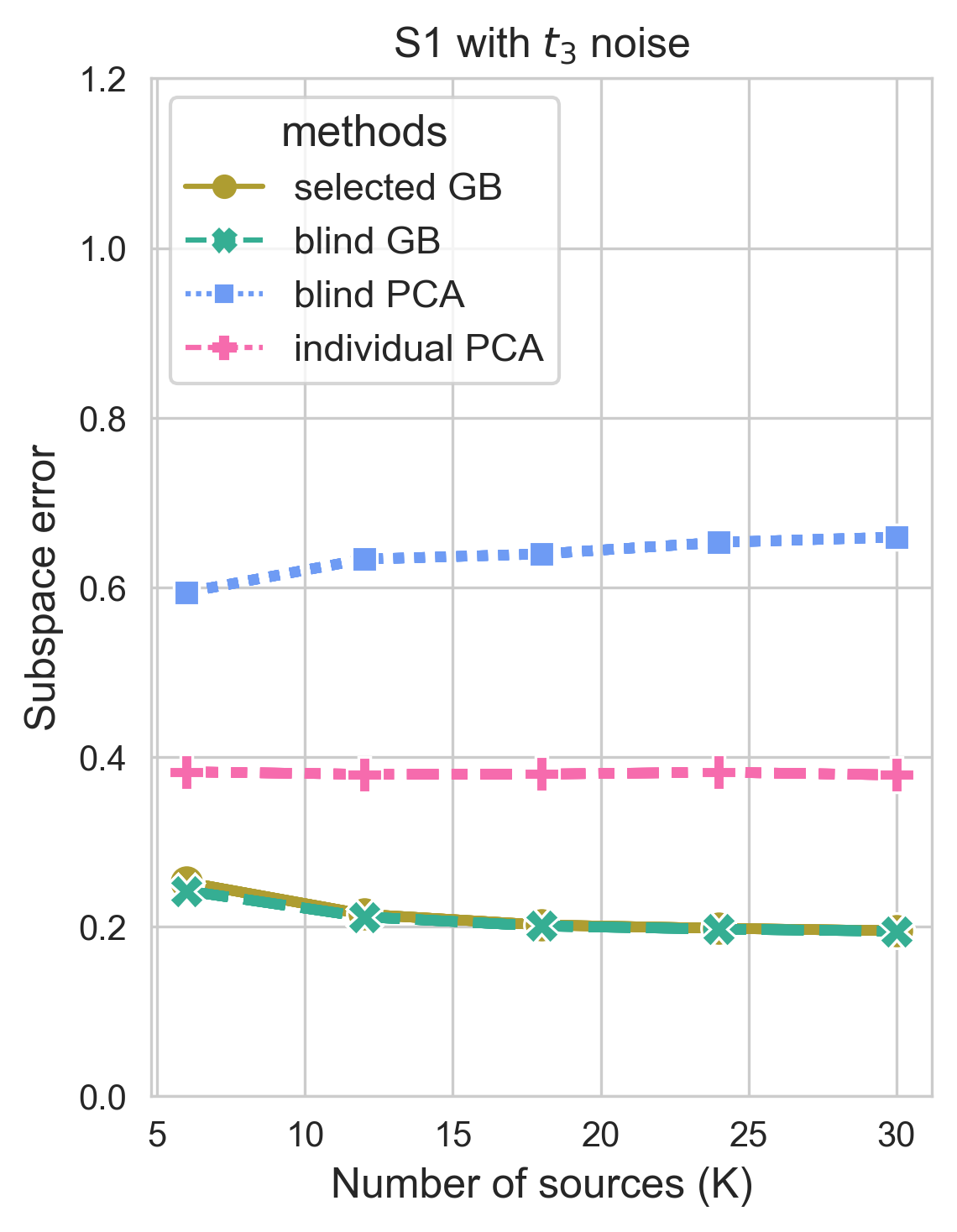}
	\end{minipage}
        \begin{minipage}{0.28\linewidth}
		\centering
	\includegraphics[width=0.91\linewidth]{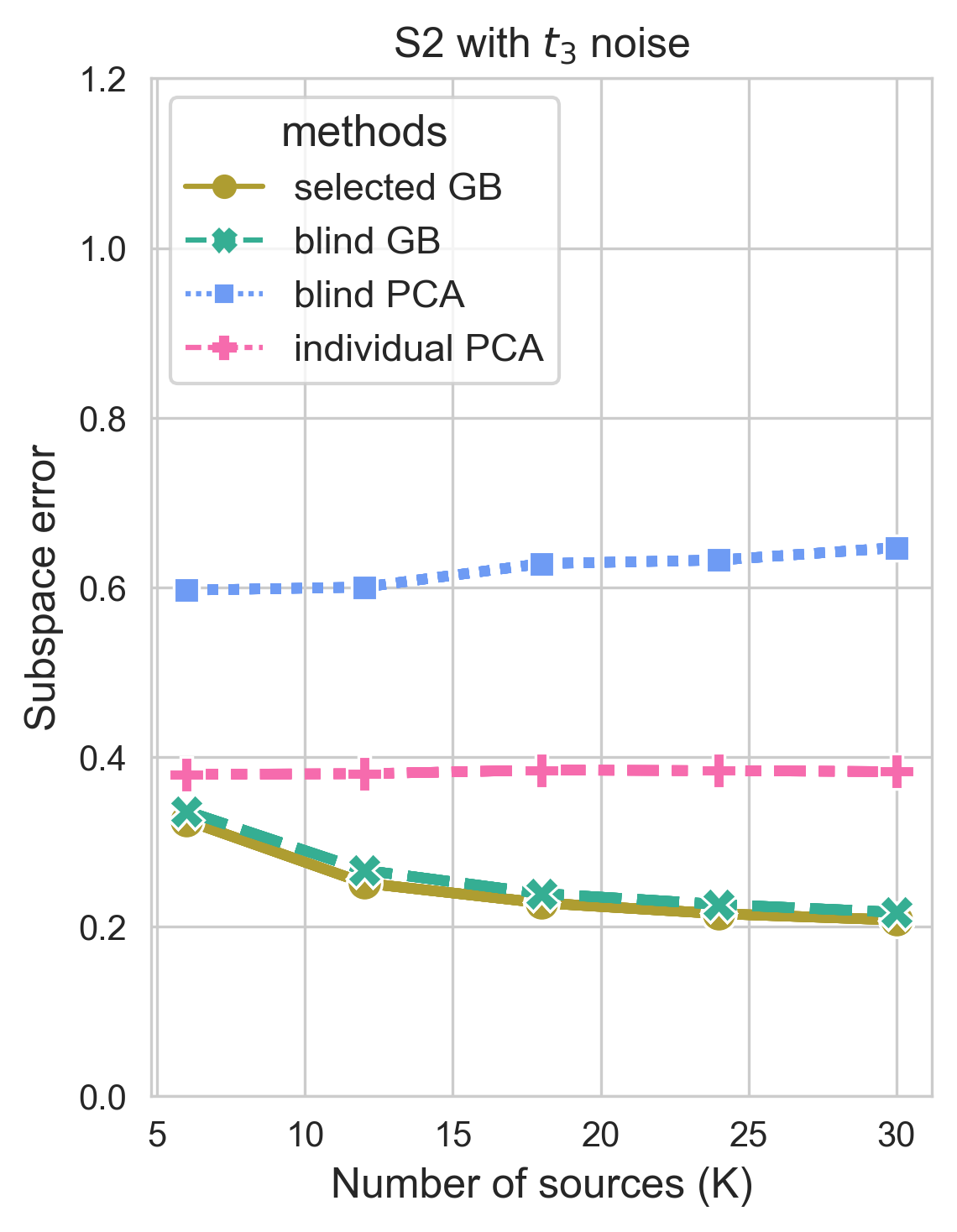}
	\end{minipage}
	\begin{minipage}{0.28\linewidth}
		\centering
	\includegraphics[width=0.91\linewidth]{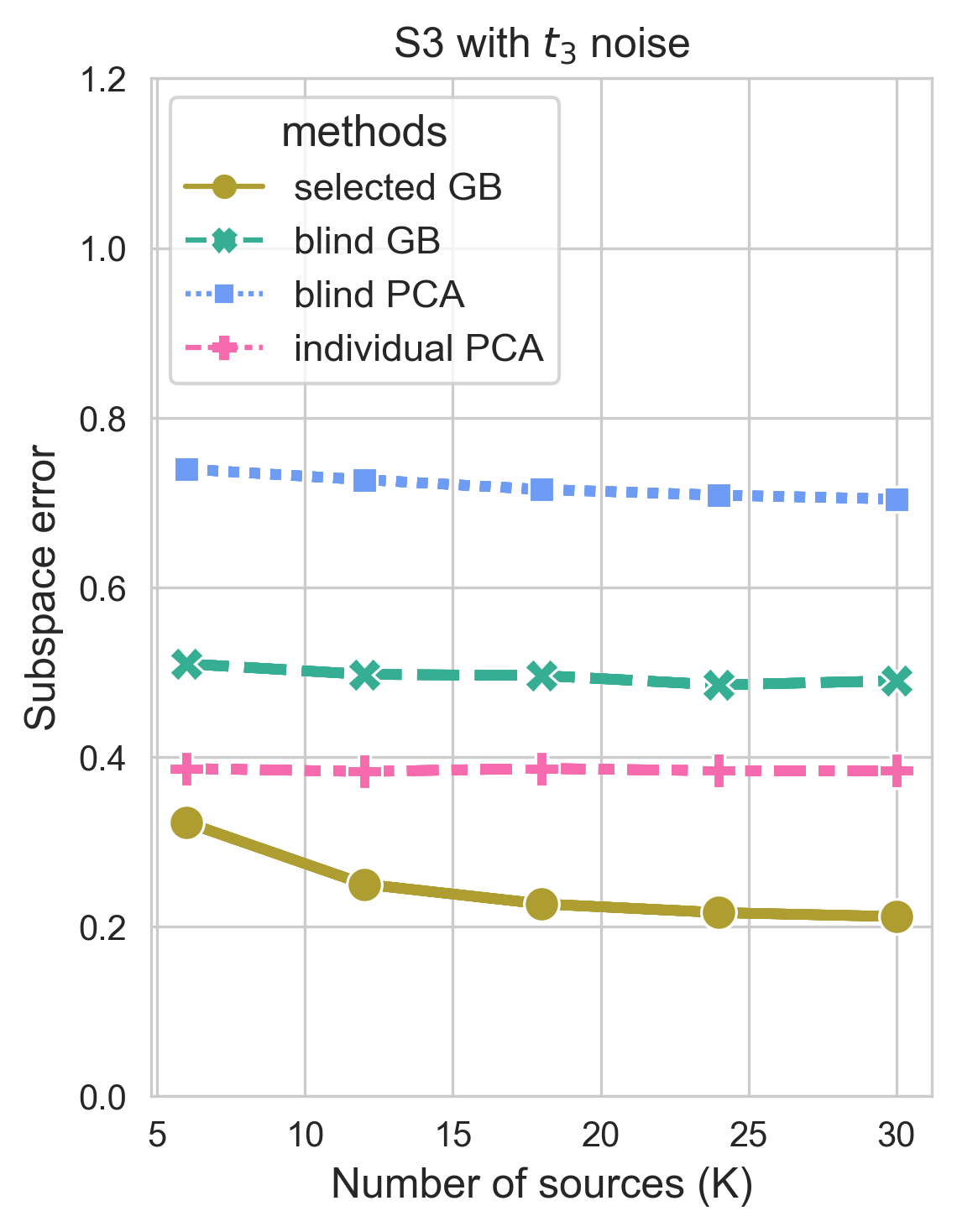}
	\end{minipage}

\caption{\label{Fig-sim-KT} Averaged subspace estimation error using various methods under $t_3$ distribution with elliptical PCA, based on 100 replications. From left to right we report S1 (no inclusion of useless sources), S2 (mild inclusion of useless sources) and S3 (severe inclusion of useless sources), respectively.
}
\end{figure}

The only difference in Figure \ref{Fig-sim-KT} is that we apply all methods with elliptical PCA under $t_3$ distribution instead. We see similar patterns in all three scenarios as in Figure \ref{Fig-sim-CN}. It is also worth mentioning that the computational burden of elliptical PCA grows quickly with the sample size, so our algorithms show great computational advantage over the pooled elliptical PCA method. Additional simulation results, including those of classical PCA under $t_3$ distribution and elliptical PCA under Gaussian distribution, are left to the supplementary material.

\subsection{Statistical rates}\label{sec:nsr}
In this section, we validate the statistical error rate of the oracle transfer estimator $\hat{P}_0$ given in Theorem \ref{theo:main}. While the data generation procedure is identical to that in Section \ref{sec:mc}, we still need to formulate different simulation settings as (\ref{rate:main}) eventually contains four terms. We aim to highlight each main term at a time, so the rest terms are set to be relatively smaller. To be more specific, we let $n_k=n$ for $k\in\{0\}\cup[K]$, and for each term, we generate data as follows:
\begin{itemize}
    \item The private term $\delta_0\tilde{n}_0^{-1/2}/\delta_{\p}$: for $k\in\{0\}\cup[K]$, we set
    $\lambda_k^\p\in\{2,4,6,8,10\}$ with $\lambda_k^\s\in\{5,6,7\}$, while we set $p=50$, $n=100$, $K=20$ and $h=0$.
    \item The variance term $\tilde{N}_{\cI}^{-1/2}$: under the setting of this section, $\tilde{N}_{\cI}\propto nK$, we give two settings testifying the variation tendency, which includes changing $n$ and $K$ respectively.
    \begin{itemize}
        \item Changing $n$: we set
    $n\in\{50,80,110,140,170\}$ with $p\in\{30,40,50\}$. Moreover we set $K=5$, $\lambda_k^\s=4$ while $\lambda_k^\p=100$ for $k\in\{0\}\cup[K]$ and $h=0$.
        \item  Changing $K$: we set
    $K\in\{5,7,9,11,13\}$ with $p\in\{35,40,45\}$. We set $n=150$, $\lambda_k^\s=2$ while $\lambda_k^\p=50$ for each $k\in\{0\}\cup[K]$ and $h=0$.
    \end{itemize}
    \item The high-order bias term $(\sum_{k\in\{0\}\cup\cI}r_k^{1/2})\tilde{N}_{\cI}^{-1}$: under the setting of this section, the high-order bias term is proportional to $n^{-1}$, so we choose $n\in\{35,40,45,50\}$ with $(p,K)\in\{(15,10),(25,20),(35,30)\}$. We set $\lambda_k^\s=2$ while $\lambda_k^\p=1000$ for each $k\in\{0\}\cup[K]$ and $h=0$.
    \item The subspace deviation term $h$:  we choose
    $h\in\{0.1,0.12,0.14,0.16,0.18,0.2\}$ with $(p,n)\in\{(15,150)$, $(20,200),(25,250)\}$. Further we set $K=10$ and $\lambda_k^\s=50$, $\lambda_k^\p=100$ for each $k\in\{0\}\cup[K]$.
\end{itemize}

\begin{figure}[H]
	\centering
\resizebox{0.97\linewidth}{!}{
		\includegraphics[width=0.2\linewidth]{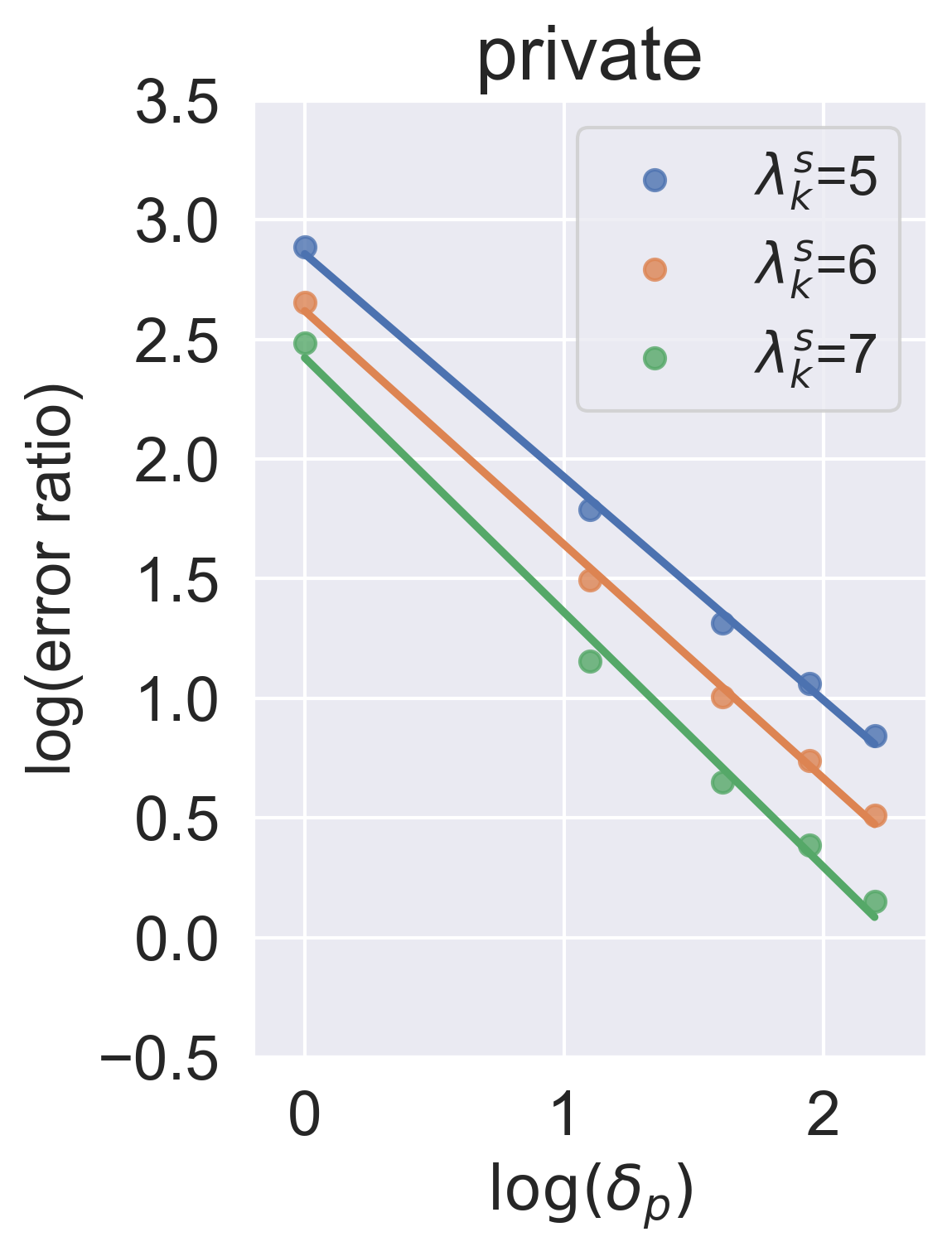}

		\includegraphics[width=0.2\linewidth]{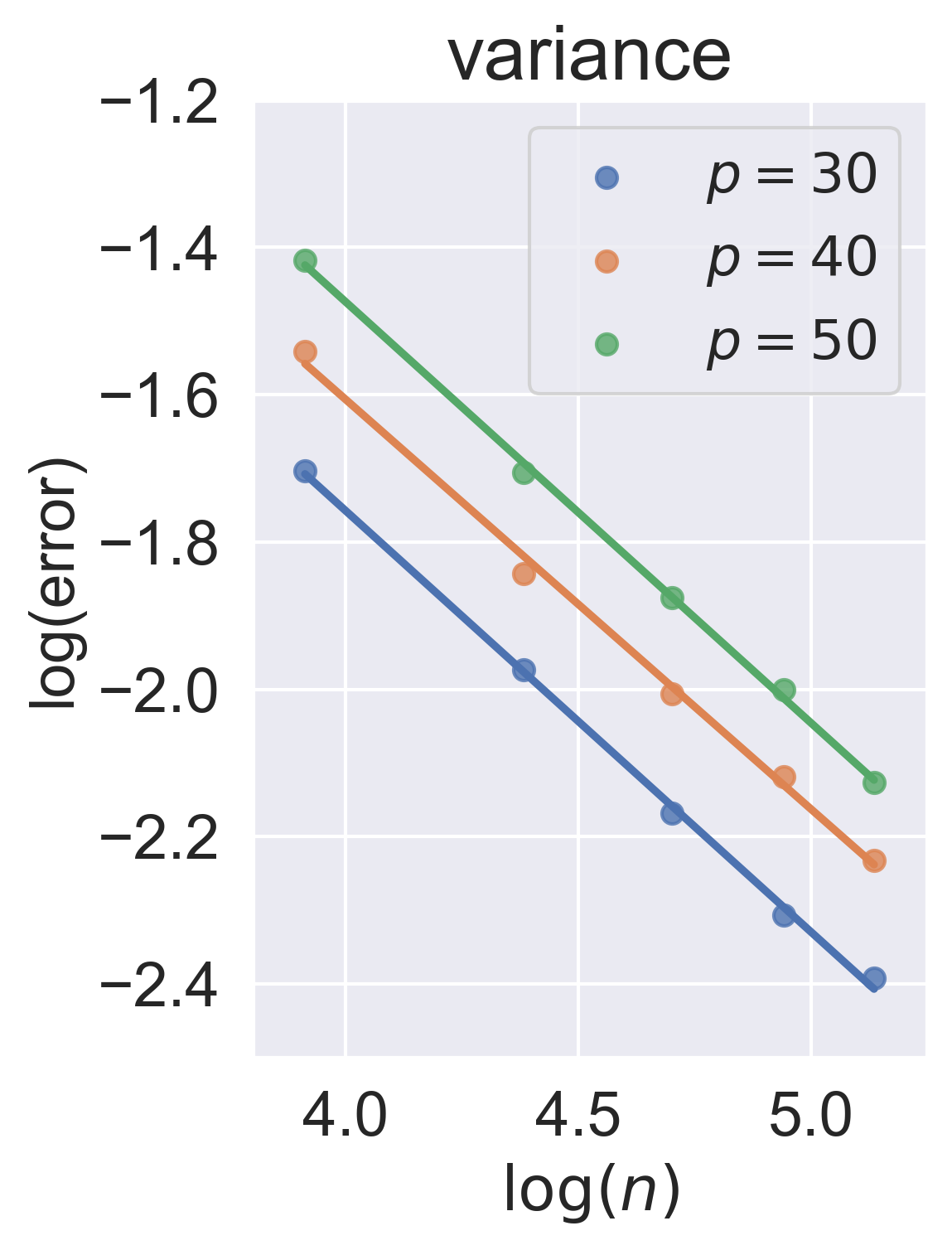}

		\includegraphics[width=0.2\linewidth]{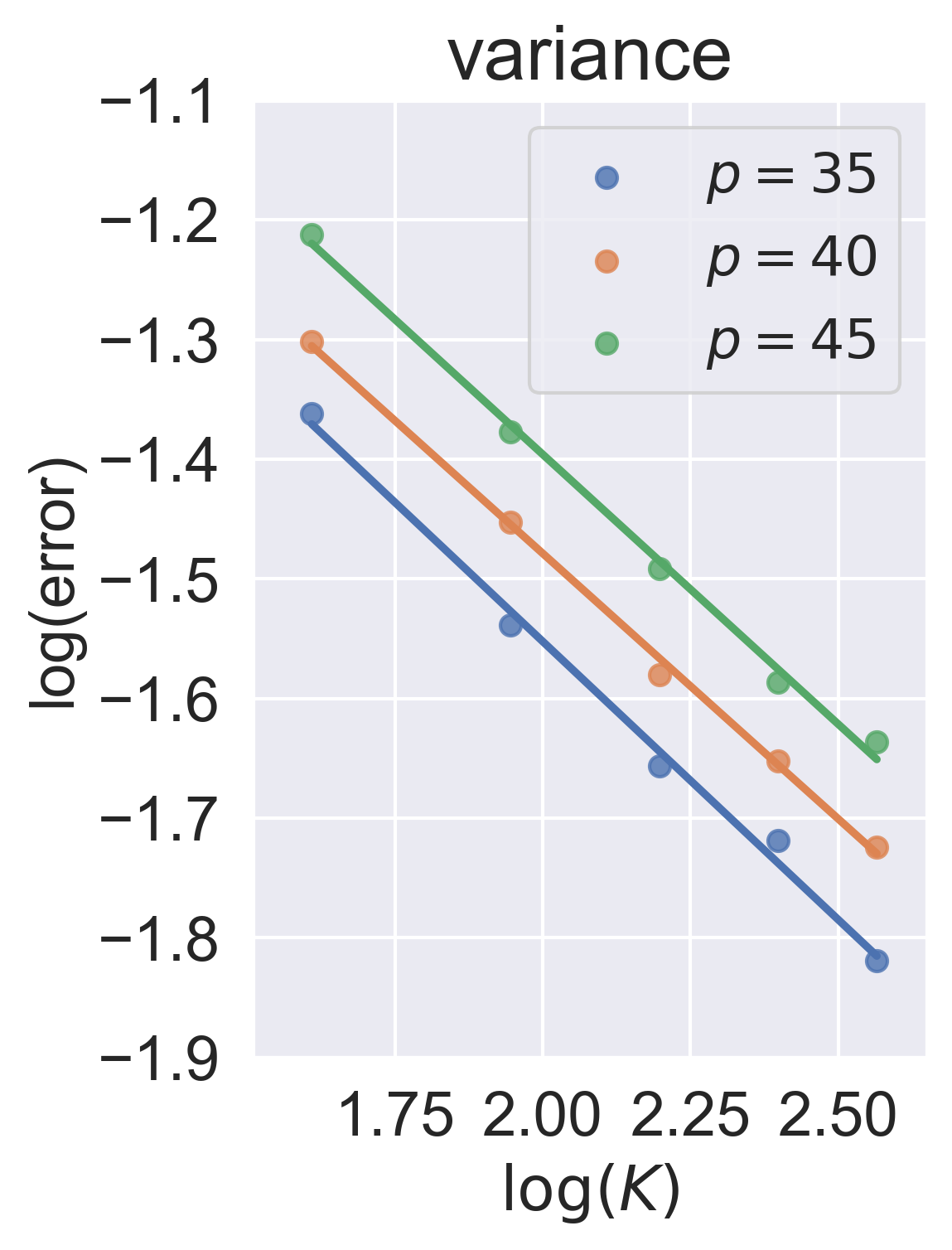}

            \includegraphics[width=0.2\linewidth]{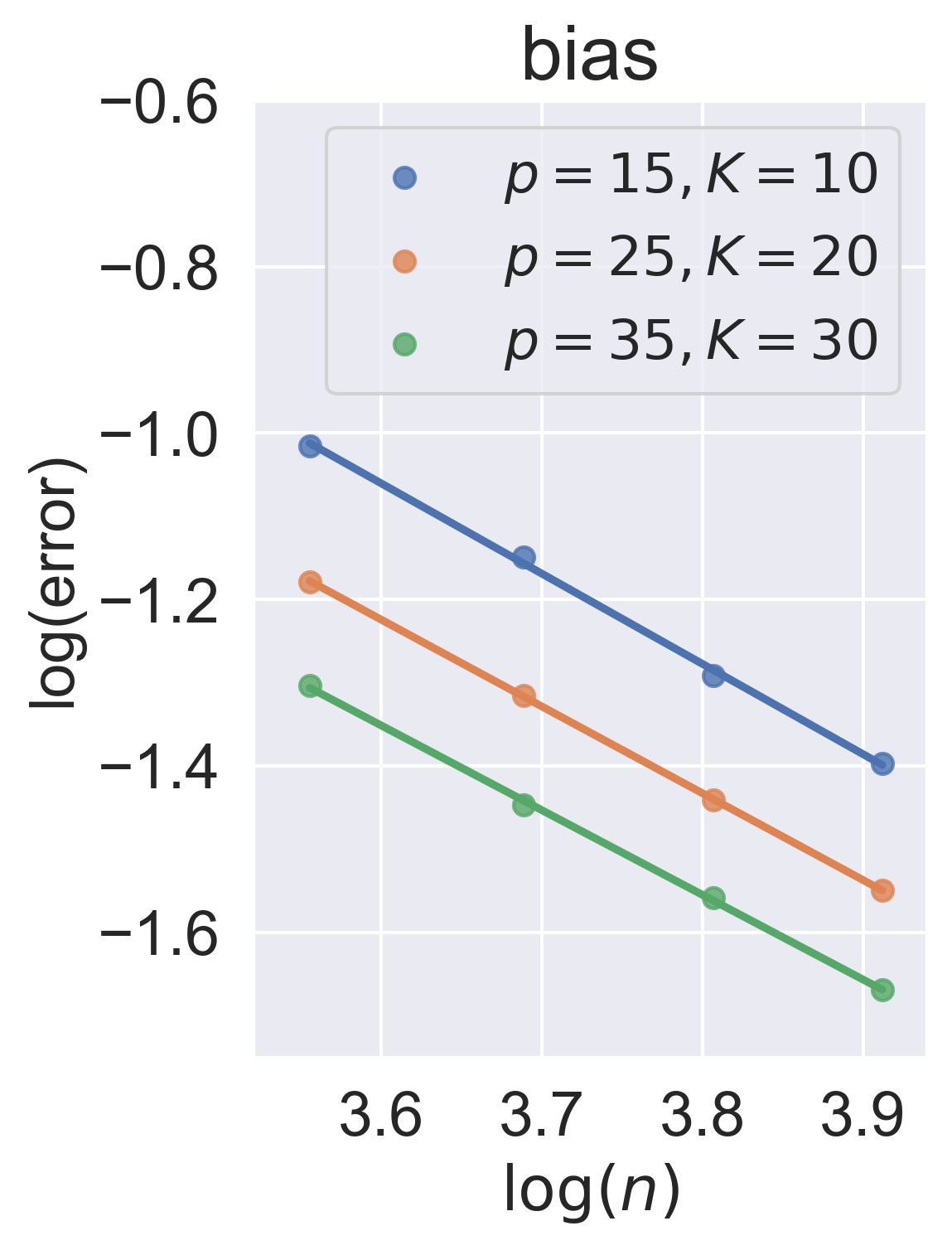}
		\includegraphics[width=0.2\linewidth]{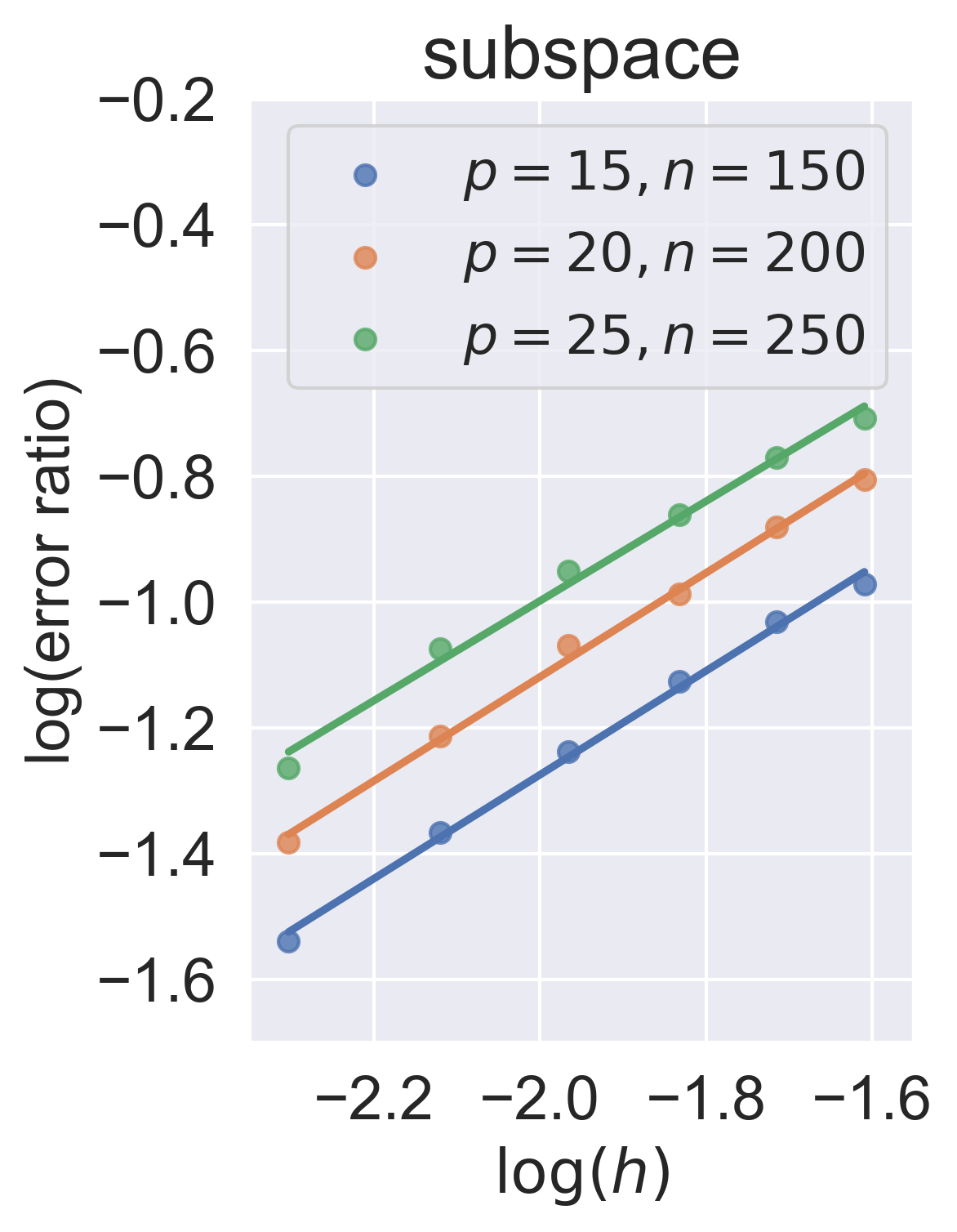}
}
 \caption{\label{ratetest}Validation of statistical rates in Theorem \ref{theo:main}. From left to right we show how the private term changes with $\delta_{\p}$; how the variance term changes with $n$ and $K$; how the bias term changes with $n$ and how the subspace deviation term changes with $h$.}
\end{figure}

Further remarks on the settings above are as follows. For the private term, we report the $\text{error ratio}:=\|\hat{P}_0-P_0^*\|_F/(\|\tilde{P}_0-P_0^*\|_F-\|\hat{P}_0-P_0^*\|_F)$ for $\tilde{P}_0$ the individual PCA estimator acquired directly using the target dataset alone, instead of the $\text{error}:=\|\hat{P}_0-P_0^*\|_F$. The reason we consider such error ratio is the fact that changing $\delta_{\p}$ also changes $\tilde{n}_0$ in the private term. We essentially want to show that by increasing $\delta_{\p}$, the private subspace estimation becomes a statistically easier task compared to the shared subspace estimation, whose difficulty is depicted by the fixed $\delta_0$. In fact, the denominator $(\|\tilde{P}_0-P_0^*\|_F-\|\hat{P}_0-P_0^*\|_F)$ is meant to approximate the unobservable shared subspace estimation error using the target dataset alone. Meanwhile, the bias term is higher-order compared to the second variance term in most practical cases, hence it is hard to observe under regular settings, and we have to work on settings which are somehow too harsh to achieve in practice, e.g., we set $\lambda_k^\p=1000$ for each $k\in\{0\}\cup[K]$.

Figure \ref{ratetest} demonstrates how the average error (ratio) changes with $\delta_p$, $n$, $K$, $h$ respectively based on $50$ replications. The results clearly show that the error (ratio) is proportional to $\delta_p^{-1}$ for the private term, $n^{-1/2}$, $K^{-1/2}$ for the variance term, and $n^{-1}$ for the bias term. However, the slope of $\log(h)$, corresponding to the subspace deviation term, is slightly smaller than 1 (about 0.82). The reason is that the subspace deviation magnitude is only related to $h$ indirectly in the simulation. In summary, the empirical experimental results
are consistent to the statistical rates provided in Theorem \ref{theo:main}.

\subsection{Asymptotic normality}\label{sec:nan}
In the end, we numerically validate the claims made right after Corollary \ref{coro:AN}, concerning the asymptotic normality of bilinear forms with respect to the spectral projectors. As a much larger number of replications would be required, we only compare the performance of the target subspaces given sufficiently reliable shared subspace information, instead of generating all informative source datasets and then estimate the shared subspace. To be more specific, after we generate the target population and sample covariance matrices $\Sigma_0^*$ and $\hat{\Sigma}_0$ in the same way as given in Section \ref{sec:mc}, let $P_0^\s=U_0^{\s}(U_0^{\s})^{\top}$ be the shared spectral projector, we slightly perturb $P_0^\s$ and acquire $\hat{P}_0^{\s}$ by taking the leading $r_{\s}$ eigenvectors of $P_0^\s + s G_0$. Here $G_0$ is taken from $\text{GOE}(p)$, which is defined as the $p\times p$ symmetric random matrix with off-diagonal entries taken independently from $N(0,(2p)^{-1})$ and the diagonal entries taken independently from $N(0,p^{-1})$, and $s$ measures the deviation of $\hat{P}_0^{\s}$ away from $P_0^\s$.

\begin{figure}[H]
	\centering
	\includegraphics[width=0.435\linewidth]{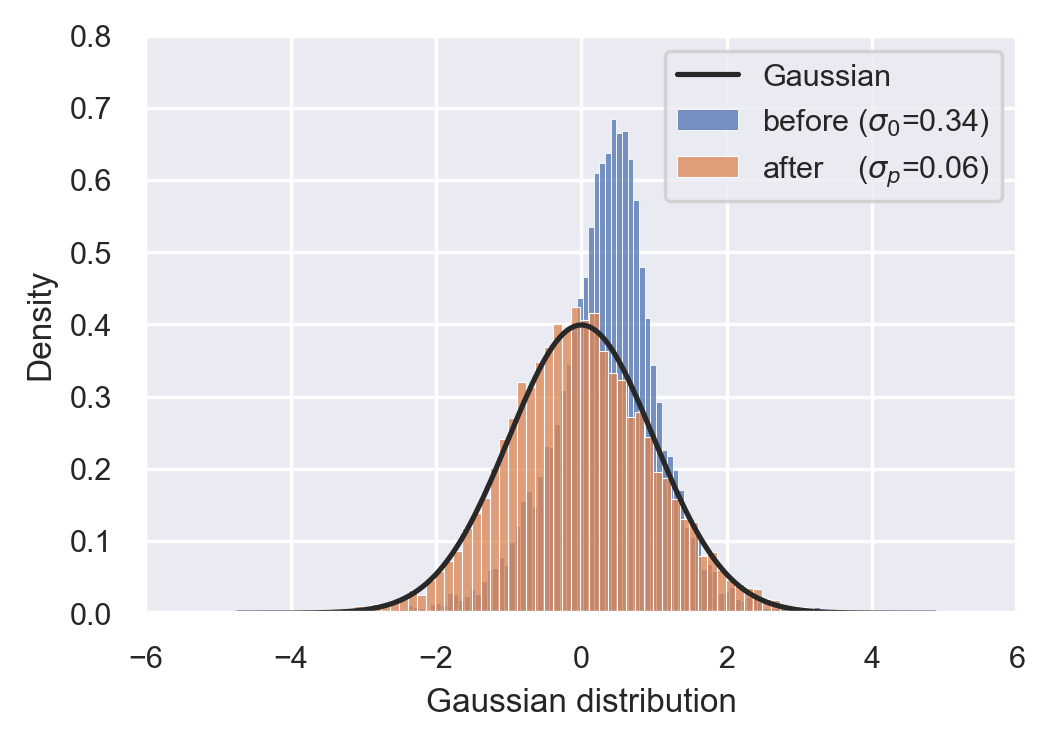}
	\includegraphics[width=0.435\linewidth]{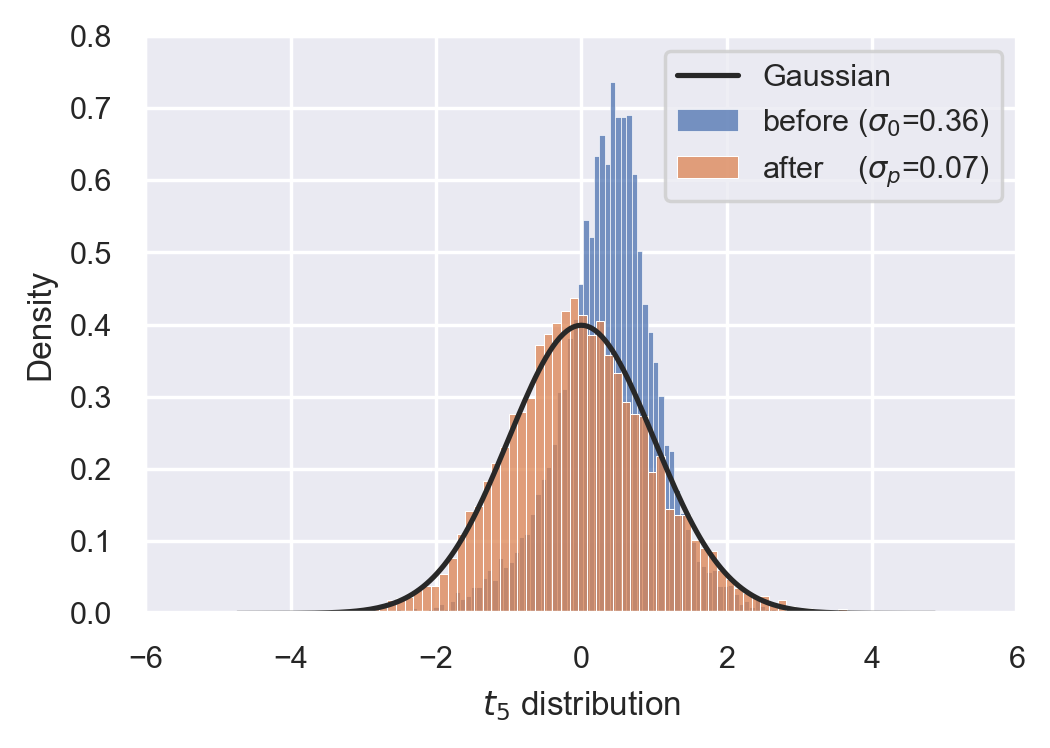}
\caption{\label{Fig-AN}
Histograms of the re-scaled bilinear forms with respect to the empirical spectral projectors acquired before and after knowledge transfer, compared to the standard Gaussian distribution, based on $10000$ replications.}
\end{figure}

We set $n_0=100$, $p=50$, $r_0=4$, $r_\s=2$, $\lambda_0^{\p}=10$, $\lambda_0^{\s}=2$, $s=0.05$, and generate data subject to Gaussian and $t_5$ distributions, respectively. We then report the performance of $\tilde{P}_0$, acquired by directly taking the leading $r_0=4$ eigenvectors of $\hat{\Sigma}_0$, and $\hat{P}_0$, acquired by the fine-tuning step of Algorithm \ref{alg:ora} given the initial shared subspace projector $\hat{P}_0^{\s}$. Figure \ref{Fig-AN} presents the histograms of the re-scaled bilinear forms $\tilde{Z}:=n^{1/2}\langle u, (\tilde{P}_0-P_0^*)v\rangle/\sigma_0$ (before knowledge transfer) and $\hat{Z}:=n^{1/2}\langle u, (\hat{P}_0-P_0^*)v\rangle/\sigma_{\p}$ (after knowledge transfer), based on $10000$ replications. As $\Sigma_0^*$ and $\Sigma_0^{\p} = (P_0^{\text{s}})^{\perp} \Sigma^*_0 (P_0^{\text{s}})^{\perp}$ are known in advance in numerical experiments, we are able to calculate the exact values of $\sigma_0$ and $\sigma_{\p}$ as given explicitly in Corollary \ref{coro:AN}, as $\Sigma^*_0$ could also be viewed as the private covariance matrix $\Sigma_0^{\p}$ with $(P_0^{\text{s}})^{\perp}=I_p$. Notice that $\sigma_0$ and $\sigma_{\p}$ change slightly given different distributions. We compare the re-scaled empirical distributions to the standard Gaussian distribution, and $\hat{Z}$ acquired via transfer learning clearly fits better. Moreover, the variance of the asymptotic normal term also becomes smaller after knowledge transfer. Unfortunately, the parameters $\sigma_0$ and $\sigma_{\p}$ are generally unknown in practice, and further investigation concerning statistical inference of spectral projectors is still very appealing.

\section{Real Data Analysis}\label{sec:real}
In this section, we apply the methods proposed in this work, with respect to the classical PCA, on an activity recognition dataset. In \cite{barshan2014recognizing}, the authors collected features of eight volunteer subjects (four female, four male, ages 20-30, ranging from $S_1$ to $S_8$) when performing each of the 19 activities (including sitting, standing, ascending and descending stairs, rowing, jumping, etc., ranging from $A_1$ to $A_{19}$) for 5 minutes, using inertial and magnetic sensor units carried on the chest, arms and legs. We treat each activity from each subject as an individual principal component analysis task, adding up to $K=152$ studies in total. The dataset is pre-processed in \cite {wang2018stratified} by extracting 27 features from both time and frequency domains for a single sensor. Since there are three sensors, i.e., accelerometer, gyroscope, and magnetometer, on one body part, in total we have $p=27\times 3\times 5=405$ features recorded for the subject $S_i$ when performing the activity $A_j$, while the sample size $n_{ij}=60$. The pre-processed dataset and other detailed information could be found on the website \url{https://www.kaggle.com/datasets/jindongwang92/crossposition-activity-recognition}.

In this real data case, we do not need to specify a target study of interest, and all $152$ studies are treated equally. In each experiment, after standardizing all datasets, we randomly split each study into training $(80\%)$ and testing $(20\%)$ datasets. We then apply the rectified Grassmannian K-means version of Algorithm \ref{alg:nora} on the training datasets for $k\in [K]$ to acquire a $r_\s$-dimensional shared subspace estimator $\hat{P}^{\s}_{0,\tau}$, with $\hat{\cI}=\{k\in[K]\mid \tr[\hat{P}^{\s}_{0,\tau}\tilde{P}_k]\geq \tau \}$ according to (\ref{eq:itercri}). Here $\hat{\cI}$ is the subset of $[K]$ consisting of the studies that provide information on the shared subspace estimation, and we regard the studies in $\hat{\cI}$ potential for knowledge transfer. For each $k \in \hat{\cI}$, we acquire the following four knowledge transfer estimators to estimate $P_k^*$: (a) blind GB estimator which directly applies Grassmannian barycenter on all training datasets in $[K]$ with dimension $r_\s$ and then fine-tunes using the $k$-th training dataset with dimension $r_{\p}$; (b) blind PCA estimator which first pools all training datasets in $[K]$ and then performs PCA with dimension $r=r_\s+r_{\p}$; (c) selected GB estimator by fine-tuning $\hat{P}^{\s}_{0,\tau}$ using the $k$-th training dataset with dimension $r_{\p}$; and (d) selected PCA estimator by performing PCA on the pooled training dataset by $\hat{\cI}$ with dimension $r=r_\s+r_{\p}$. We only consider the classical PCA setup for the real example in this section.

\begin{figure}[h]
  \centering
  \includegraphics[width=5.2in]{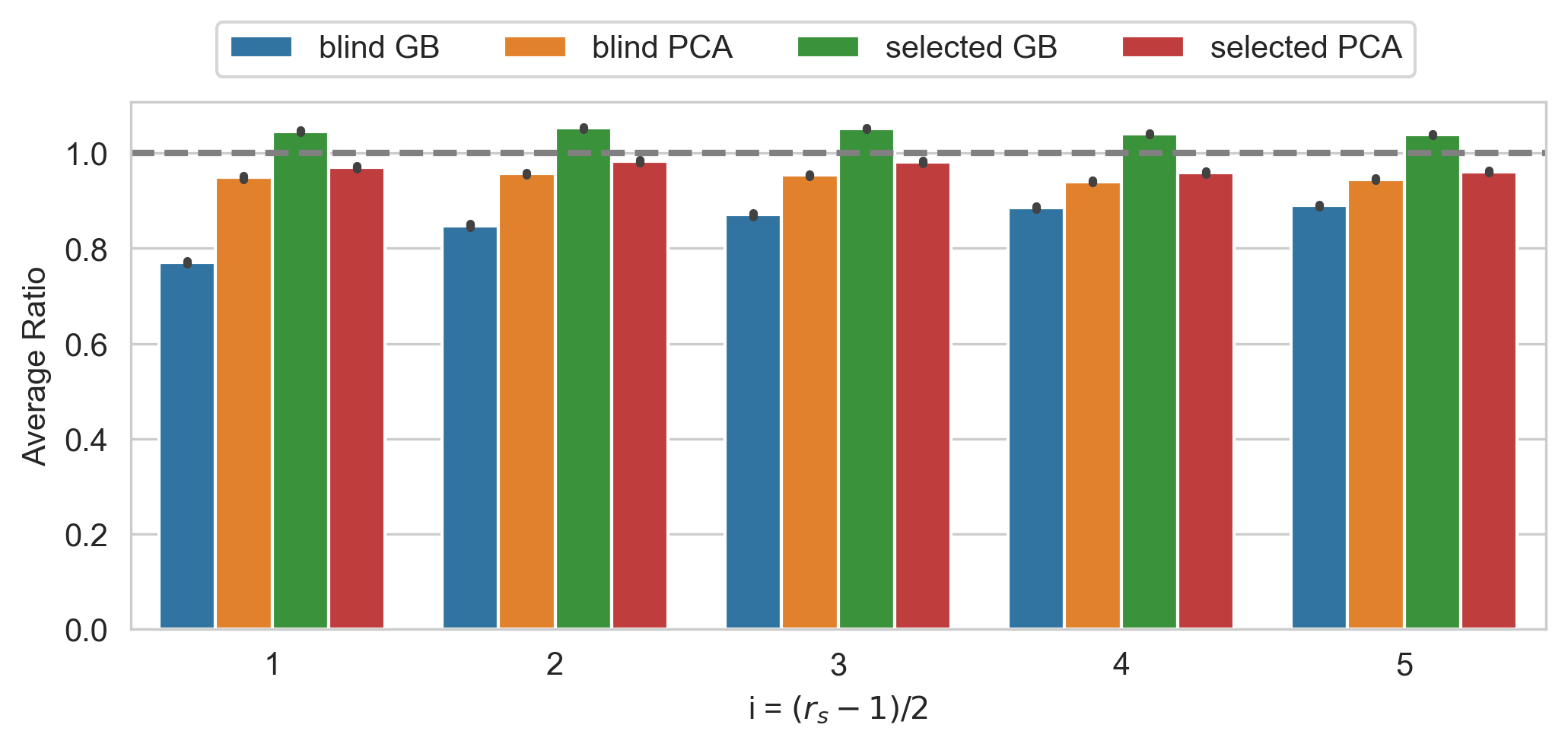}
  \caption{Mean of the Average relative information preservation Ratio (AR) on the testing datasets for different knowledge transfer subspace estimators against the individual PCA estimator $\tilde{P}_k$ with varying $i=(r_\s-1)/2$, based on 100 replications.}\label{fig:activity}
\end{figure}

Given any knowledge transfer subspace estimator $\hat{P}_k$, we use the $k$-th testing dataset to evaluate its relative performance against the individual PCA estimator $\tilde{P}_k$, acquired by performing PCA on the $k$-th training dataset with $r=r_\s+r_{\p}$. To be more specific, given a single observation $y_{k,i}$ in the testing dataset for $i\in [12]$, we compute the relative information preservation ratio as $\|\hat{P}_ky_{k,i}\|^2/\|\tilde{P}_ky_{k,i}\|^2$, where $\|\cdot\|$ is the vector $\ell_2$ norm. We then report the Average relative information preservation Ratio (AR) on the testing datasets across all $k\in \hat{\cI}$, which is defined as
$$\text{AR}=\frac{1}{12|\hat{\cI}|}\sum_{k\in \hat{\cI}}\sum_{i\in [12]}\frac{\|\hat{P}_ky_{k,i}\|^2}{\|\tilde{P}_ky_{k,i}\|^2}.$$

We fix $T=10$ and $r_{\p}=5$, set $r_{\s}=2i+1$ and $\tau = 3i^{0.65}/2$ for $i\in [5]$, and replicate the preceding procedure by 100 times. The average number of datasets included, namely the average of $|\hat{\cI}|$, is respectively $64$, $62$, $62$, $72$ and $79$ when $i$ varies from $1$ to $5$. The resulting mean of AR is reported in Figure \ref{fig:activity}. First, for the studies in the selected dataset $\hat{\cI}$, directly pooling all datasets in $[K]$ and then apply either Algorithm \ref{alg:ora} or PCA leads to unsatisfactory performance on the testing datasets. That is to say, those datasets in $\hat{\cI}^{c}:=[K]\setminus \hat{\cI}$ might not be helpful to the studies in $\hat{\cI}$, and blindly including $\hat{\cI}^{c}$ for knowledge transfer could cause the negative transfer problem. On the other hand, say we only consider transferring knowledge across the selected datasets in $\hat{\cI}$, the selected PCA estimator obtained by performing PCA on the pooled training dataset by $\hat{\cI}$ is still outperformed by the individual PCA estimator $\tilde{P}_k$. This is somehow foreseeable as the private subspace information of the $k$-th study could easily drown in the pooled sample covariance matrix. In the end, the selected GB estimator by Algorithm \ref{alg:nora} shows steady advantage over the individual PCA estimator and other knowledge transfer estimators due to its capability of simultaneously selecting the informative datasets, harnessing the shared subspace information and also taking into account the valuable private subspace information from each individual study.

\section{Discussion}
\label{sec:conc}
In this work, we propose methods to transfer knowledge between multiple principal component analysis studies with statistical guarantee, given the informative source datasets either known or unknown. We leave a few possible extensions for future work. First, the non-oracle knowledge transfer method, namely Algorithm \ref{alg:nora}, is initially designed to be computationally feasible even when there are a large number of source datasets distributed across different machines and might not be statistically optimal. In fact, we believe that there exists more statistically accurate useful dataset selection methods, e.g., by conducting hypothesis tests which harness the limiting eigenvector distributions. Second, we believe the idea of identifying shared and private components for multiple studies is essential in a wider range of applications including surveillance background subtraction, personalized medicine, etc. These applications might have to be considered under an alternative framework, for instance data could be stored in a natural tensor form and further efforts shall be made. Finally, it is also interesting to consider knowledge transfer for other unsupervised learning studies, as principal component analysis in this work is merely a beginning.

\bibliographystyle{apalike}
\bibliography{ref}

\newpage
\setcounter{footnote}{0}
\clearpage
\setcounter{page}{1}
\setcounter{section}{0}
\begin{center}
  \title{\bf \LARGE Supplementary Material for ``Knowledge Transfer across Multiple Principal Component Analysis Studies"}\\
  \author{Zeyu Li\footnotemark[2]\\
    School of Management, Fudan University\\
    Kangxiang Qin\footnotemark[2] \\
    School of Mathematics, Shandong University
    \\
    Yong He\footnotemark[1]\\
    Institute for Financial Studies, Shandong University\\
    Wang Zhou\\
    Department of Statistics and Data Science, National University of Singapore\\
    Xinsheng Zhang\\
    School of Management, Fudan University}
  \maketitle
\end{center}
\footnotetext[2]{The authors contributed equally to this work.}
\footnotetext[1]{Corresponding author, email: heyong@sdu.edu.cn.}
\renewcommand{\thesection}{S\arabic{section}}

The supplementary material contains algorithm details not included in the main article, additional numerical results, and proofs of the theoretical results. In Section \ref{ssec:dgo}, we give details of the manifold optimization procedures to solve for the rectified problem in Section 2.2, which is part of Algorithm 2 for non-oracle knowledge transfer. Then we present the additional numerical results in Section \ref{ssec:anr}. In Section \ref{ssec:ur} we introduce some useful results used in the following proofs, which compose the remainder of the supplementary material.

\section{Details of Grassmannian Optimization}\label{ssec:dgo}
First, we give the recursive formula for the Grassmannian gradient descent steps after obtaining $\cI_t$ from (\ref{eq:itercri}), we set
$$\bar{P}_t = \frac{1}{N_{[K]}}\left(n_0\tilde{P}_0 + \sum_{k\in\cI_t}n_k \tilde{P}_k \right).$$
For any proper defined matrices $A$, $B \in \reals^{n\times n}$, denote the Lie bracket by $[A,B] = AB-BA$. Given $P_{t-1}$ and $\bar{P}_t$, we have
\begin{equation*}\label{eq:grass_grad}
	P_{t}= P_{t-1}+\alpha_t [P_{t-1},[P_{t-1},\bar{P}_t]],
\end{equation*}
where $\alpha_t$ is the pre-determined step size.

Second, as for Newton's method on Grassmann manifolds, given $P_{t-1}$ and $\bar{P}_t$, we first perform the following eigenvalue decomposition
 $$P_{t-1}= (U^{1}_{t-1},U^{2}_{t-1})\left(
 \begin{array}{cc}
 I_{r_{\text{s}}} & 0\\
 0 & 0
 \end{array}\right) (U^{1}_{t-1},U^{2}_{t-1})^{\top},$$
where $U^{1}_{t-1}$ is of shape $p\times r_{\text{s}}$ and $U^{2}_{t-1}$ is of shape $p\times (p-r_{\text{s}})$. Then, let $\bar{P}_t^{ij} = (U^i_{t-1})^{\top} \bar{P}_t (U^j_{t-1})$, for $i, j \in \{1, 2\}$, we solve for $Z_t$ via the Sylvester equation $Z_t \bar{P}_t^{22}-\bar{P}_t^{11}Z_t = \bar{P}_t^{12}$. In the end, perform QR decomposition and acquire
\begin{equation*}\label{eq:grass_new}
    Q_tR_t = \left(\underbrace{U^1_t}_{p\times r_\s},U^2_t\right)R_t=\left(
 \begin{array}{cc}
 I_{r_{\text{s}}} & 0\\
 -Z_t^{\top} & I_{p-r_{\text{s}}}
 \end{array}\right),\quad \text{we have}\quad P_{t}=U^1_t (U^1_t)^{\top}.
\end{equation*}

\section{Additional Numerical Results}\label{ssec:anr}
\begin{figure}[H]
	\centering

    \begin{minipage}{0.28\linewidth}
		\centering
	\includegraphics[width=0.91\linewidth]{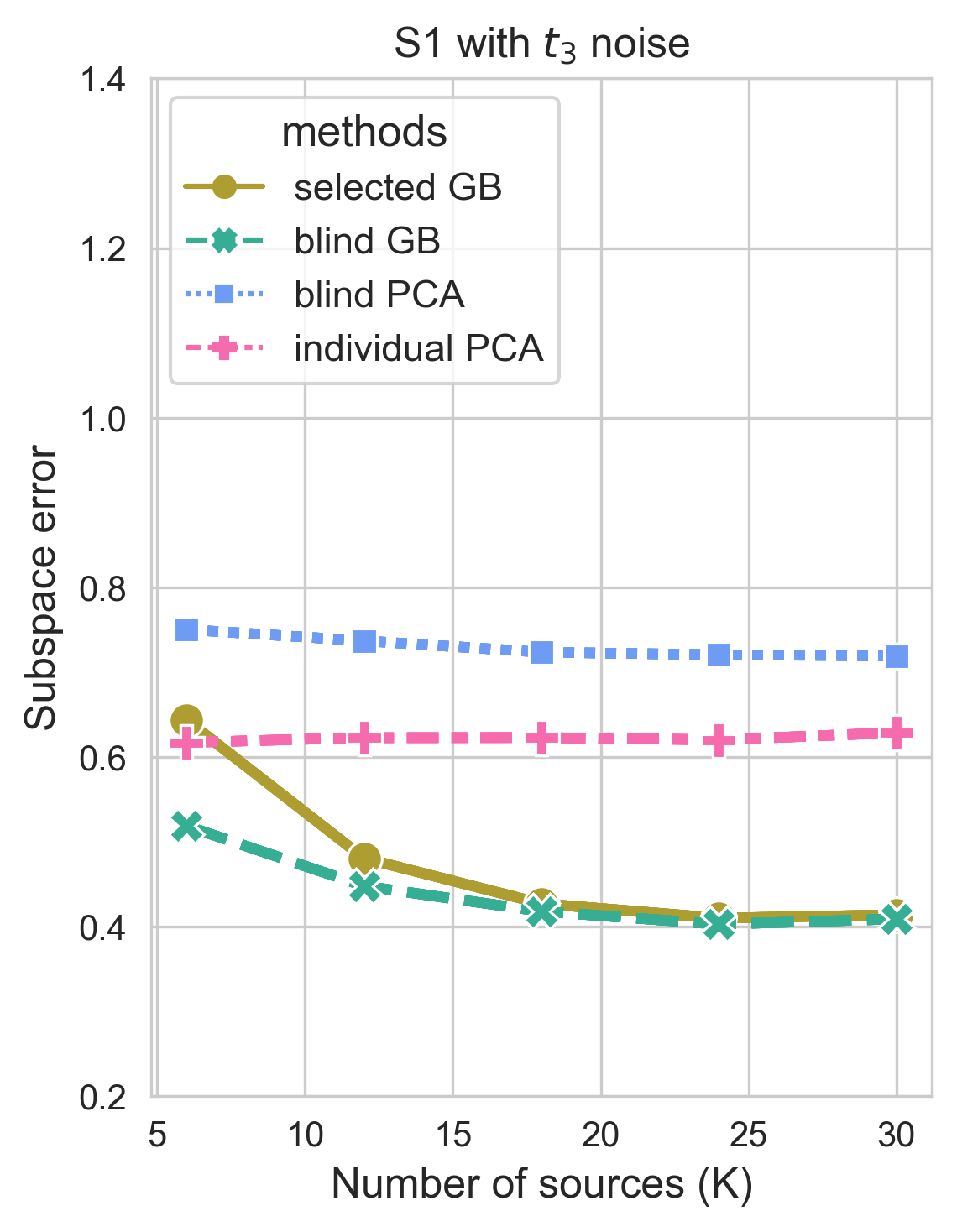}
	\end{minipage}
        \begin{minipage}{0.28\linewidth}
		\centering
	\includegraphics[width=0.91\linewidth]{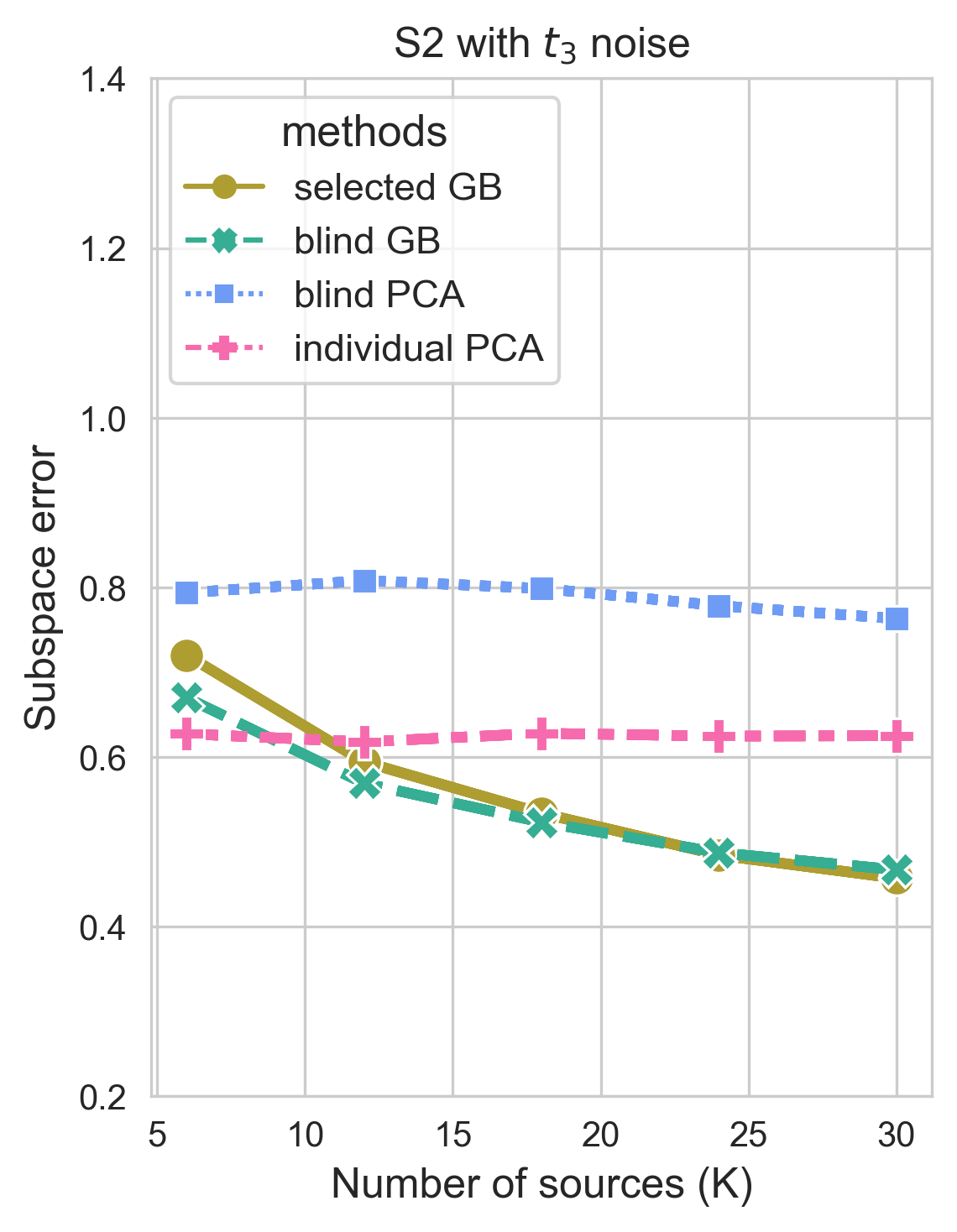}
	\end{minipage}
	\begin{minipage}{0.28\linewidth}
		\centering
	\includegraphics[width=0.91\linewidth]{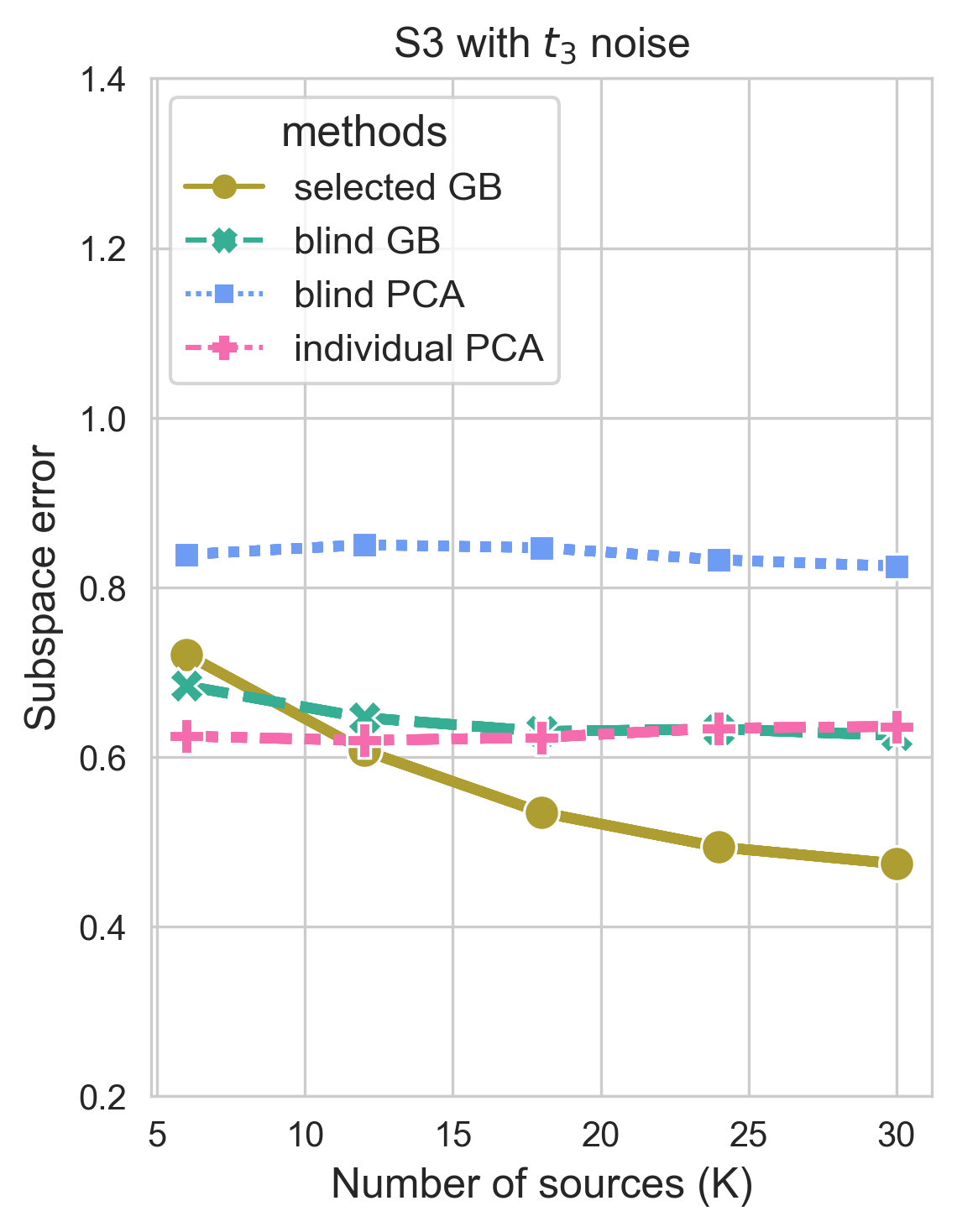}
	\end{minipage}

\caption{\label{Fig-sim-CT} Average subspace estimation error using various methods under $t_3$ distribution with classical PCA, based on 100 replications. From left to right we report S1 (no inclusion of useless sources), S2 (mild inclusion of useless sources) and S3 (severe inclusion of useless sources), respectively.
}
\end{figure}

\begin{figure}[H]
	\centering

    \begin{minipage}{0.28\linewidth}
		\centering
	\includegraphics[width=0.91\linewidth]{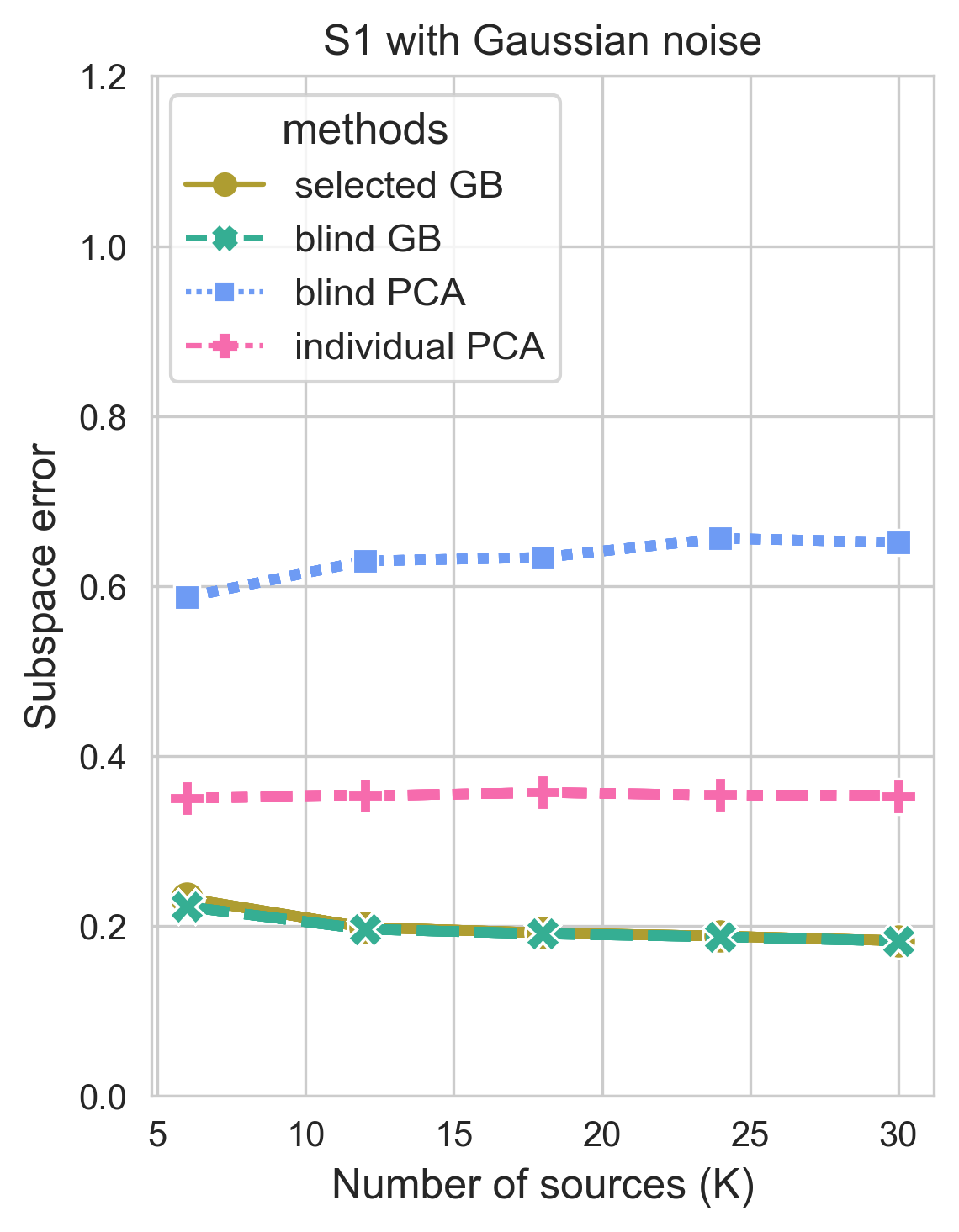}
	\end{minipage}
        \begin{minipage}{0.28\linewidth}
		\centering
	\includegraphics[width=0.91\linewidth]{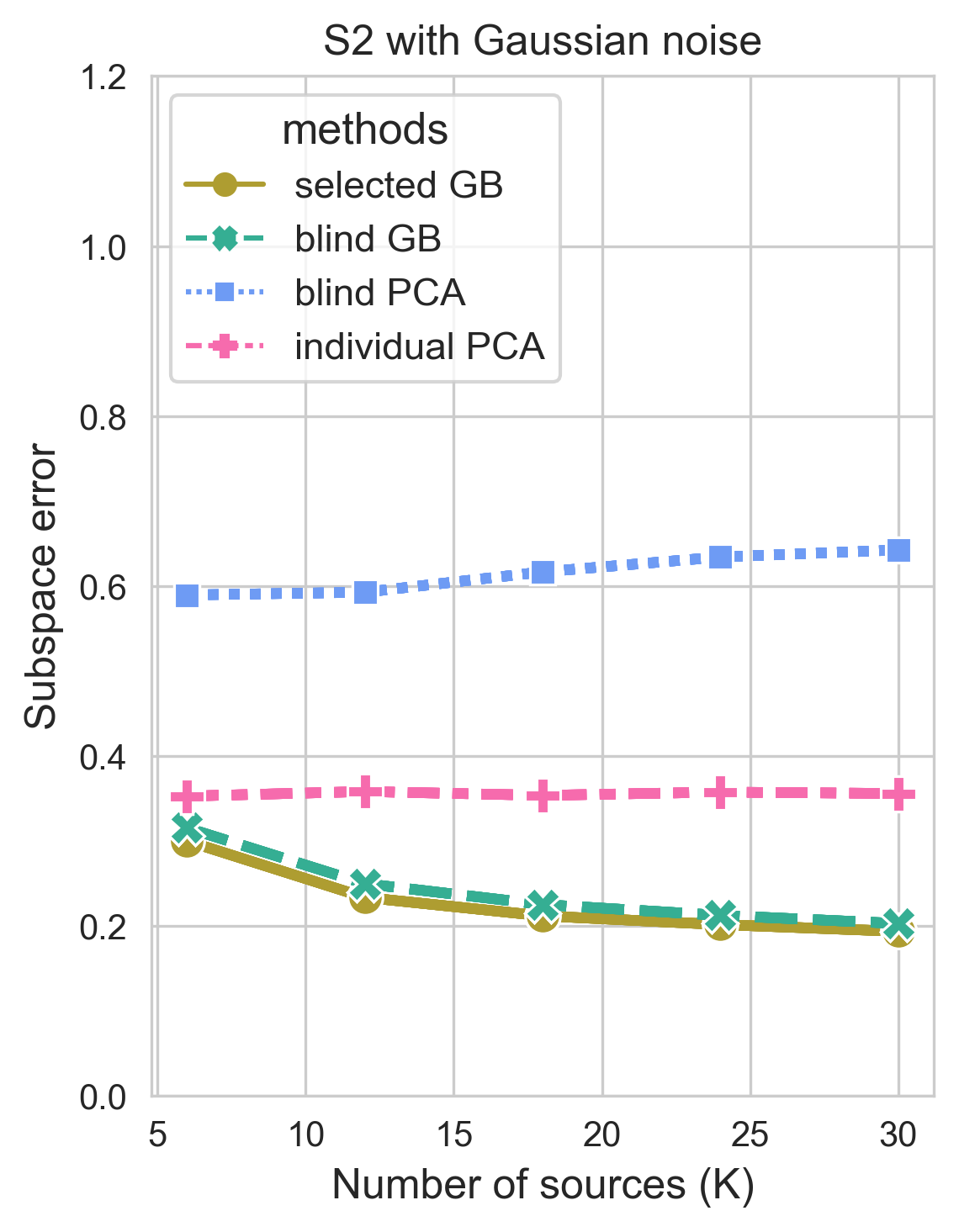}
	\end{minipage}
	\begin{minipage}{0.28\linewidth}
		\centering
	\includegraphics[width=0.91\linewidth]{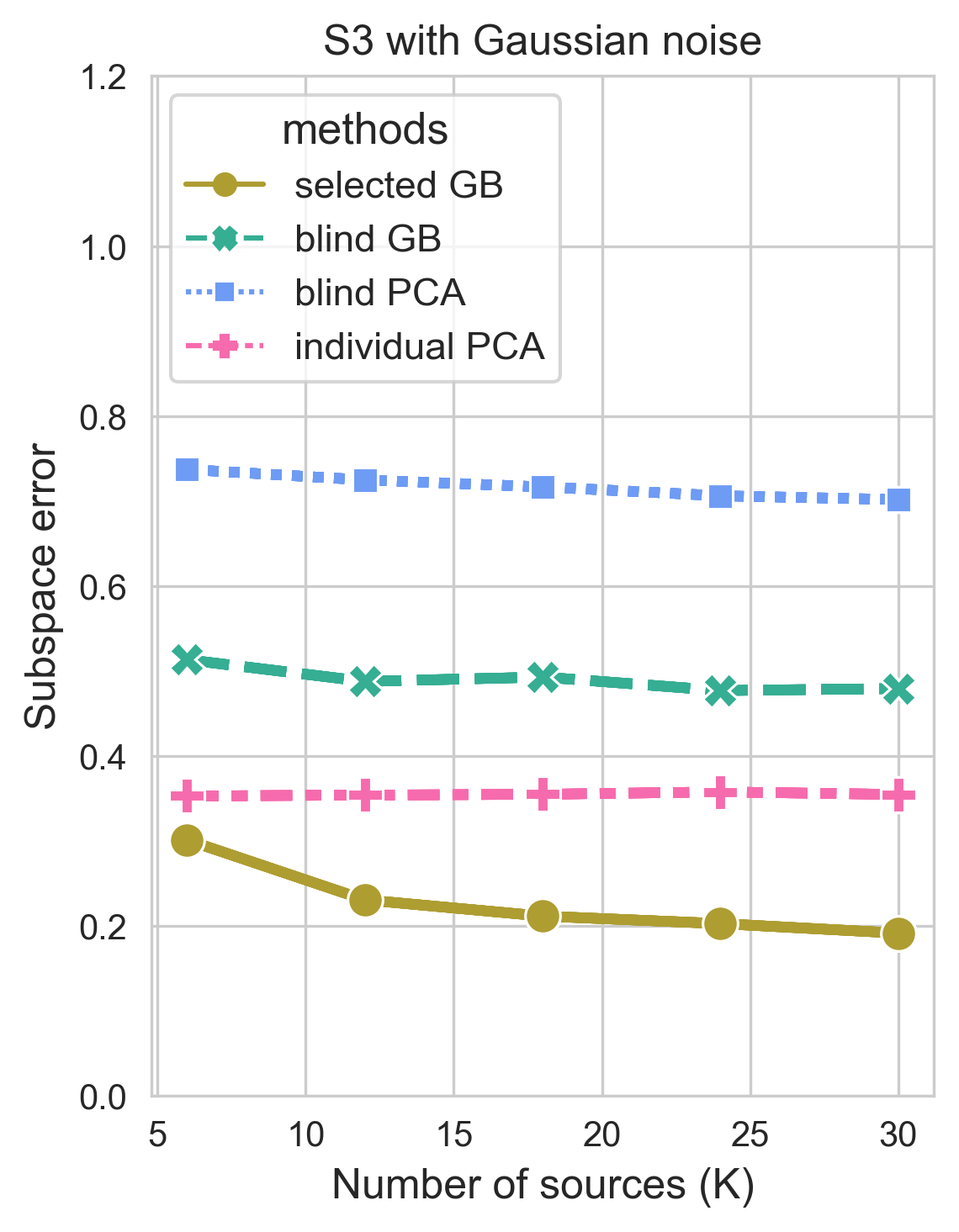}
	\end{minipage}

\caption{\label{Fig-sim-KN} Average subspace estimation error using various methods under Gaussian distribution with elliptical PCA, based on 100 replications. From left to right we report S1 (no inclusion of useless sources), S2 (mild inclusion of useless sources) and S3 (severe inclusion of useless sources), respectively.
}
\end{figure}

\section{Useful Results}\label{ssec:ur}
For convenience, we first introduce some useful results applied in this work.

\begin{proposition}[Lemma 2 from \cite{fan2019distributed}]\label{lem:taylor}
Let $\Sigma$ and $\hat{\Sigma}$ be $p\times p$ symmetric matrices with non-increasing eigenvalues $\lambda_1\geq \dots\geq \lambda_p$ and $\hat{\lambda}_1\geq \dots\geq \hat{\lambda}_p$, respectively. Let $\{u_i\}_{i=1}^{p}$, $\{\hat{u}_i\}_{i=1}^{p}$ be the corresponding eigenvectors such that $\Sigma u_i = \lambda_iu_i$ and $\hat{\Sigma} u_i = \hat{\lambda}_i\hat{u}_i$. Fix $s\in \{0,\dots,p-r\}$, let $\delta=\min(\lambda_s-\lambda_{s+1},\lambda_{s+r}-\lambda_{s+r+1})>0$, with $\lambda_0=+\infty$ and $\lambda_{p+1}=-\infty$. For $E = \hat{\Sigma}-\Sigma$, we denote $S=\{s+1,\dots,s+r\}$ and $S^c=[p]\setminus S$, and define $K$ of shape $p\times p$ where
$$K_{ij}=K_{ji}=
\left\{
    \begin{array}{lc}
        (u_i^{\top} E u_j)/(\lambda_i-\lambda_j) & i\in S,\, j\in S^c, \\
       0& \text{otherwise}.\\
    \end{array}
\right.$$
Then, let $U = (u_1,\dots, u_p)$, $\hat{U} = (\hat{u}_1,\dots, \hat{u}_p)$, while $U_S = (u_{s+1},\dots, u_{s+r})$, $\hat{U}_S = (\hat{u}_{s+1},\dots, \hat{u}_{s+r})$, we claim that
$$\hat{U}_S\hat{U}_S^{\top}-U_SU_S^{\top} = UKU^{\top}+\Delta,$$
where $\|K\|_F\leq (2r)^{1/2}\|E\|_2/\delta$. In addition, when $\|E\|_2/\delta\leq 1/10$, $\|\Delta\|_F\leq 24r^{1/2}(\|E\|_2/\delta)^2$.
\end{proposition}
\begin{proof}
    Proposition \ref{lem:taylor} follow directly from Lemma 2 in \cite{fan2019distributed}, we only need to check that $\|K\|_F\leq (2r)^{1/2}\|E\|_2/\delta$. To see this, note that $$\|K\|_F\leq  \frac{\sqrt{2}}{\delta}\left(\sum_{i=1}^{r}\left\|(U^{\top}EU)_{i,.}\right\|^2\right)^{1/2},$$
	where $A_{i,.}$ is the $i$-th row of matrix $A$. The result then follows from $\|(U^{\top}EU)_{i,.}\|\leq \|U^{\top}EU\|_2 = \|E\|_2$.
\end{proof}

\begin{proposition}[Theorem 2.5 from \cite{bosq2000stochastic}]\label{lem:subexp}
	For independent random vectors $\{X_i\}_{i=1}^{n}$ in a separable Hilbert space with norm $\|\cdot\|$, if $\E(X_i) = 0$ and $\|\|X_i\|\|_{\psi_1}\leq L_i<\infty$, we have
    $$\left\|\left\|\sum_{i=1}^{n}X_i\right\|\right\|_{\psi_1}\lesssim \left(\sum_{i=1}^{n}L_i^2\right)^{1/2}.$$
\end{proposition}

\section{Proof of Theorem \ref{theo:main}}\label{proof:main}
Recall that $\hat{P}_0= \hat{P}_0^{\s}+ \hat{P}_0^{\p}$ in Algorithm \ref{alg:ora}, while $P_0^* = P_0^\s+P_0^{\p}$. By triangular inequality we have $ \|\hat{P}_0-P_0^*\|_F\leq \|\hat{P}_0^{\s}-P_0^\s\|_F+\|\hat{P}_0^{\p}-P_0^\p\|_F$. The first term $\|\hat{P}_0^{\s}-P_0^\s\|_F$ is the estimation error of the Grassmannian barycenter estimator $\hat{P}_0^{\s}$ for the shared subspace $P_0^\s$. First, given Assumptions \ref{assum:1} and \ref{assum:2}, the following Lemma \ref{lemma:sse} controls the estimation error of the shared subspace. We defer its proof to Section \ref{proof:sse}.

\begin{lemma}[Shared subspace estimator]\label{lemma:sse}
    Under Assumptions \ref{assum:1} and \ref{assum:2}, as $\min_{k\in \{0\}\cup\cI}(n_k), p\rightarrow \infty$ and $h\rightarrow 0$, we have
    $$\left\|\left\|\hat{P}_0^{\text{s}}-P_0^{\text{s}}\right\|_F\right\|_{\psi_1}\lesssim s=\tilde{N}_{\cI}^{-1/2}+\left(\sum_{k\in\{0\}\cup\cI}r_k^{1/2}\right)\tilde{N}_{\cI}^{-1}+h. $$
\end{lemma}

With Lemma \ref{lemma:sse} in mind, Theorem \ref{theo:main} is trivial if $r_s=r_0$. When $r_s<r_0$ and there exists some private subspace to be estimated, to prove Theorem \ref{theo:main}, we only need to consider the fine-tuning step of Algorithm \ref{alg:ora} when the shared subspace is well-estimated, which is discussed in the following Lemma \ref{lemma:private}.

\begin{lemma}[Private subspace estimator]\label{lemma:private}
    Under Assumption \ref{assum:1}, assume that the shared subspace estimator satisfies $\|\hat{P}_0^{\text{s}}-P_0^{\text{s}}\|_F\leq s$, then for sufficiently small $\|\Sigma_0^*\|_2s/\delta_{\p}$, the private subspace estimator satisfies
    \begin{equation}
          \left\|\hat{P}_0^{\p}-P_0^{\p}\right\|_F=O_p\left( \frac{\delta_0}{\delta_{\p}}\tilde{n}_0^{-1/2} + s+\frac{\eta s}{\delta_{\p}}+\frac{\|\Sigma_0^*\|_2^2s^2}{\delta^2_{\p}}\right).
    \end{equation}
\end{lemma}

Theorem \ref{theo:main} is a direct consequence of both Lemmas \ref{lemma:sse} and \ref{lemma:private}, so it suffices to prove Lemma \ref{lemma:private}. To prove Lemma \ref{lemma:private}, note that $\hat{P}_0^{\p}$ is calculated by taking the leading $(r_0-r_\s)$ eigenvectors of the matrix
\begin{equation}\label{eq:ddotpert}
       \hat{\Sigma}_0^{\p} =(\hat{P}_0^{\s})^{\perp} \hat{\Sigma}_0 (\hat{P}_0^{\s})^{\perp} = \underbrace{(\hat{P}_0^{\s})^{\perp} \Sigma_0^* (\hat{P}_0^{\s})^{\perp}}_{\dot{\Sigma}_0^{\p}}+(\hat{P}_0^{\s})^{\perp} E_0 (\hat{P}_0^{\s})^{\perp}, \quad\text{where}
\end{equation}
\begin{equation}\label{eq:ddotP}
    \dot{\Sigma}_0^{\p}= \underbrace{(P_0^{\s})^{\perp} \Sigma_0^* (P_0^{\s})^{\perp}}_{\Sigma_0^{\p}} + \Delta_\s^{\perp} \Sigma_0^* (P_0^{\s})^{\perp} + (P_0^{\s})^{\perp} \Sigma_0^* \Delta_\s^{\perp} + \Delta_\s^{\perp}\Sigma_0^* \Delta_\s^{\perp}, \quad \text{for}
\end{equation}
$$\Delta_\s^{\perp} :=(\hat{P}_0^{\s})^{\perp}-(P_0^{\s})^{\perp} = -(\hat{P}_0^{\s}-P_0^{\s})=:-\Delta_\s.$$

First, take the leading $(r_0-r_\s)$ eigenvectors of $ \dot{\Sigma}_0^{\p}$ to form $\dot{P}_0^{\p}$, recall from (\ref{eq:eigen}) that $P_0^{\p} = U_0^{\p}(U_0^{\p})^{\top}$ could be acquired by take the leading $(r_0-r_\s)$ eigenvectors of $ \Sigma_0^{\p} = (P_0^{\s})^{\perp} \Sigma_0^* (P_0^{\s})^{\perp}$, we first bound $\|\dot{P}_0^{\p}-P_0^{\p}\|_F$ above by investigating the matrix perturbation in (\ref{eq:ddotP}).

When $\|\Delta_\s\|_F=\|\hat{P}_0^{\s}-P_0^{\s}\|_F\leq s\leq  1/10$, according to Proposition \ref{lem:taylor}, we are able to decompose $\Delta_\s =L_\s + R_\s$, where the remainder term satisfies $\|R_\s\|_F\lesssim r_\s^{1/2}s^2$, note the eigenvalue gap in this case for $P_0^\s$ is exactly $1$. By recalling $\mathcal{U}^{\p}_0$ from (\ref{eq:eigen}), the linear term $L_\s$ could be written in the following form
\begin{equation}\label{eq:Ls}
    L_\s  = \mathcal{U}^{\p}_0  \begin{pmatrix}
	0_{(p-r_\s)\times(p-r_\s)} & K_\s^{\top}\\
	K_\s & 0_{r_\s\times r_\s}
\end{pmatrix}(\mathcal{U}^{\p}_0)^{\top}, \quad \text{where} \quad K_\s =  (U_0^{\s})^{\top} \Delta_\s (U_0^{\s})^{\perp}.
\end{equation}

We then present the following Lemma \ref{lemma:ddotsep} that separates the linear terms in (\ref{eq:ddotP}) from other higher order terms.
\begin{lemma}\label{lemma:ddotsep}
    Assume that $\|\Delta_\s\|_F\leq s \leq 1/10$, we have
    \begin{equation}\label{eq:ddotsep}
        \dot{\Sigma}_0^{\p}= \Sigma_0^{\p}-\underbrace{\left(L_\s \Sigma_0^*(P_0^{\s})^{\perp}+(P_0^{\s})^{\perp}\Sigma_0^*L_\s\right)}_{\text{linear term}} + \underbrace{W_\s}_{\text{remainder}},
    \end{equation}
    where $\|W_s\|_2\lesssim r_s^{1/2}\|\Sigma_0^*\|_2 s^2$.
\end{lemma}
\begin{proof}
    We plug in $\Delta_\s^{\perp}=-\Delta_\s$ and $\Delta_\s =L_\s + R_\s$ into (\ref{eq:ddotP}), by simple calculation, (\ref{eq:ddotsep}) holds with
    $$W_\s = -\left(R_\s \Sigma_0^*(P_0^{\s})^{\perp}+(P_0^{\s})^{\perp}\Sigma_0^*R_\s\right) + \Delta_\s \Sigma_0^* \Delta_\s.$$
    Then, since $\|R_\s \Sigma_0^*(P_0^{\s})^{\perp}\|_2\leq r_s^{1/2}\|\Sigma_0^*\|_2 s^2$ and $\|\Delta_\s \Sigma_0^* \Delta_\s\|_2\leq \|\Sigma_0^*\|_2 \|\Delta_\s\|^2_2\leq \|\Sigma_0^*\|_2 s^2$, we acquire the proof.
\end{proof}

With the help of Lemma \ref{lemma:ddotsep}, we are able to bound $\|\dot{P}_0^{\p}-P_0^{\p}\|_F$ in the following Lemma \ref{lemma:actmain} by using again Proposition \ref{lem:taylor}. The proof is slightly technical and is deferred to Section \ref{proof:actmain}.
\begin{lemma}\label{lemma:actmain}
    For $\delta_{\p}=d_{r_0-r_\s}(\Sigma_0^{\p})=\lambda^{\p}_{r_0-r_\s}-\lambda_{r_0+1}$, set $\eta = \|(U^{\p}_0)^{\top}\Sigma_0^*U^{\s}_0\|_2$, let $\|\Delta_\s\|_F\leq s$, for sufficiently small $\|\Sigma_0^*\|_2s/\delta_{\p}$, we have
    $$\left\|\dot{P}_0^{\p}-P_0^{\p}\right\|_F\lesssim s+\frac{\eta s}{\delta_{\p}}+\frac{\|\Sigma_0^*\|_2^2s^2}{\delta^2_{\p}}.$$
\end{lemma}

Then, we proceed to control the difference between $\dot{P}_0^{\p}$ and our estimator $\hat{P}_0^{\p}$. We return to the matrix perturbation (\ref{eq:ddotpert}), which we restate where for convenience
$$\hat{\Sigma}_0^{\p} =(\hat{P}_0^{\s})^{\perp} \hat{\Sigma}_0 (\hat{P}_0^{\s})^{\perp} = \underbrace{(\hat{P}_0^{\s})^{\perp} \Sigma_0^* (\hat{P}_0^{\s})^{\perp}}_{\dot{\Sigma}_0^{\p}}+(\hat{P}_0^{\s})^{\perp} E_0 (\hat{P}_0^{\s})^{\perp}$$
From (\ref{eq:ddotsep}) we have $\|\dot{\Sigma}_0^{\p}-\Sigma_0^{\p}\|_2\lesssim \|\Sigma_0^*\|_2s$, since $\|L\|_F \lesssim  r_\s^{1/2} s$ according to Proposition \ref{lem:taylor}. Then, applying Weyl's inequality to (\ref{eq:ddotP}) and Davis-Kahan theorem to (\ref{eq:ddotpert}), for sufficiently small $\|\Sigma_0^*\|_2s/\delta_{\p}$, we have
$$\left\|\hat{P}_0^{\p}-\dot{P}_0^{\p}\right\|_F\lesssim \frac{\left\|(\hat{P}_0^{\s})^{\perp} E_0 (\hat{P}_0^{\s})^{\perp}\right\|_2}{\delta_{\p}-2\|\Sigma_0^*\|_2 s}\lesssim \frac{\|E_0\|_2}{\delta_{\p}}.$$
The prove is complete by recalling that $\tilde{n}_0^{-1/2}\asymp\|\|E_0\|_2\|_{\psi_1}/d_{r_0}(\Sigma_0^*)$  under Assumption \ref{assum:1}.

\subsection{Proof of Lemma \ref{lemma:sse}}\label{proof:sse}
We initialize by stating the following Lemmas \ref{basic lemma a.1} and \ref{basic lemma a.2}, which directly lead to Lemma \ref{lemma:sse}.

\begin{lemma}\label{basic lemma a.1}
Under Assumption \ref{assum:2}, let $\Delta_k=\tilde{P}_k-P_k^*$ be the matrix perturbation of the $k$-th individual PCA study, denote $\left(P_0^{\text{s}}\right)^{\perp}=I_p-P_0^{\text{s}}$, when $h<1$ is sufficiently small, take
 $$\tilde{\Sigma}=\frac{1}{\tilde{N}_{\cI}}\sum_{k\in\{0\}\cup \cI} \tilde{n}_k\tilde{P}_k ,\quad \tilde{\Sigma}^*=P_0^{\text{s}}+ \left(P_0^{\text{s}}\right)^{\perp} \left[\frac{1}{\tilde{N}_{\cI}} \left(\sum_{k\in\{0\}\cup \cI}\tilde{n}_kP_k^{\text{p}}\right)\right]\left(P_0^{\text{s}}\right)^{\perp},\quad \text{we have}$$
\begin{equation*}\label{basic decomposition for hatP_s^0}
\|\hat{P}_0^{\text{s}}-P_0^{\text{s}}\|_F\leq \frac{\|\tilde{\Sigma}-\tilde{\Sigma}^*\|_F}{d_{r_{\text{s}}}(\tilde{\Sigma}^*)}
\lesssim \frac{1}{\tilde{N}_{\cI}}\left\|\sum_{k\in\{0\}\cup \cI}\tilde{n}_k\Delta_k\right\|_F + h.
\end{equation*}
\end{lemma}

\begin{lemma}\label{basic lemma a.2}
  Under Assumption \ref{assum:1}, recall from Lemma \ref{lemma:indPCA} that $\Delta_k$ are independent with $\left\|\left\|\Delta_k\right\|_F\right\|_{\psi_1}\lesssim \tilde{n}_k^{-1/2}$ as $\min_{k\in \{0\}\cup\cI}(n_k), p\rightarrow \infty$, we have
    $$\left\|\frac{1}{\tilde{N}_{\cI}}\left\|\sum_{k\in\{0\}\cup \cI}\tilde{n}_k\Delta_k\right\|_F\right\|_{\psi_1}\lesssim \tilde{N}_{\cI}^{-1/2}+\left(\sum_{k\in\{0\}\cup\cI}r_k^{1/2}\right)\tilde{N}_{\cI}^{-1}.$$
\end{lemma}

\subsubsection{Proof of Lemma \ref{basic lemma a.1}}
First, by definition we have
\begin{equation*}
\tilde{N}_{\cI}\tilde{\Sigma}=\sum_{k\in\{0\}\cup \cI} \tilde{n}_kP_k^*+\sum_{k\in\{0\}\cup \cI} \tilde{n}_k\Delta_k=\sum_{k\in\{0\}\cup \cI} \tilde{n}_k(P_k^{\text{s}}+P_k^{\text{p}})+\sum_{k\in\{0\}\cup \cI} \tilde{n}_k\Delta_k.
\end{equation*}
Since $P_k^{\text{s}}P_k^{\text{p}}=0$ and $\left(P_k^{\text{s}}\right)^{\perp}+P_k^{\text{s}}=I_p$ holds for each $k\in\{0\}\cup \cI$, we can obtain the following equations
\begin{equation*}
\begin{aligned}
\tilde{N}_{\cI}\tilde{\Sigma}&=\sum_{k\in\{0\}\cup \cI} \tilde{n}_kP_0^{\text{s}}+\sum_{k\in\{0\}\cup \cI} \tilde{n}_k(P_k^{\text{s}}-P_0^{\text{s}})+\sum_{k\in\{0\}\cup \cI} \tilde{n}_kP_k^{\text{p}}+\sum_{k\in\{0\}\cup \cI} \tilde{n}_k\Delta_k\\
&=\sum_{k\in\{0\}\cup \cI} \tilde{n}_kP_0^{\text{s}}+\sum_{k\in\{0\}\cup \cI} \tilde{n}_k(P_k^{\text{s}}-P_0^{\text{s}})+\left[P_0^{\text{s}}+
(P_0^{\text{s}})^{\perp}\right]\left(\sum_{k\in\{0\}\cup \cI} \tilde{n}_kP_k^{\text{p}}\right)\left[P_0^{\text{s}}+(P_0^{\text{s}})^{\perp}\right]+\sum_{k\in\{0\}\cup \cI} \tilde{n}_k\Delta_k\\
&=\sum_{k\in\{0\}\cup \cI} \left[\tilde{n}_kP_0^{\text{s}}+(P_0^{\text{s}})^{\perp}\left(\sum_{k\in\{0\}\cup \cI} \tilde{n}_kP_k^{\text{p}}\right)(P_0^{\text{s}})^{\perp}\right]+\sum_{k\in\{0\}\cup \cI} \tilde{n}_k(P_k^{\text{s}}-P_0^{\text{s}})+
P_0^{\text{s}}\left(\sum_{k\in\{0\}\cup \cI} \tilde{n}_kP_k^{\text{p}}\right)(P_0^{\text{s}})^{\perp}\\
& \quad+(P_0^{\text{s}})^{\perp}\left(\sum_{k\in\{0\}\cup \cI} \tilde{n}_kP_k^{\text{p}}\right)P_0^{\text{s}}+P_0^{\text{s}}\left(\sum_{k\in\{0\}\cup \cI} \tilde{n}_kP_k^{\text{p}}\right)P_0^{\text{s}}+\sum_{k\in\{0\}\cup \cI} \tilde{n}_k\Delta_k\\
&=\tilde{N}_{\cI}\tilde{\Sigma}^*+\sum_{k\in\{0\}\cup \cI} \tilde{n}_k\Delta_k+\sum_{k\in\{0\}\cup \cI} \tilde{n}_k(P_k^{\text{s}}-P_0^{\text{s}})+P_0^{\text{s}}\left(\sum_{k\in\{0\}\cup \cI} \tilde{n}_kP_k^{\text{p}}\right)(P_0^{\text{s}})^{\perp}+(P_0^{\text{s}})^{\perp}\left(\sum_{k\in\{0\}\cup \cI} \tilde{n}_kP_k^{\text{p}}\right)P_0^{\text{s}}\\
& \quad+\sum_{k\in\{0\}\cup \cI} \tilde{n}_k\left[(P_0^{\text{s}}-P_k^{\text{s}})P_k^{\text{p}}(P_0^{\text{s}}-P_k^{\text{s}})+
\underbrace{P_k^{\text{s}}P_k^{\text{p}}P_0^{\text{s}}+
P_0^{\text{s}}P_k^{\text{p}}P_k^{\text{s}}-P_k^{\text{s}}P_k^{\text{p}}P_k^{\text{s}}}_{0}\right]
\\
&=\tilde{N}_{\cI}\tilde{\Sigma}^*+\sum_{k\in\{0\}\cup \cI} \tilde{n}_k\Delta_k+\sum_{k\in\{0\}\cup \cI} \tilde{n}_k(P_k^{\text{s}}-P_0^{\text{s}})+P_0^{\text{s}}\left(\sum_{k\in\{0\}\cup \cI} \tilde{n}_kP_k^{\text{p}}\right)(P_0^{\text{s}})^{\perp}+
(P_0^{\text{s}})^{\perp}\left(\sum_{k\in\{0\}\cup \cI} \tilde{n}_kP_k^{\text{p}}\right)P_0^{\text{s}}\\
& \quad+\sum_{k\in\{0\}\cup \cI} \tilde{n}_k\left[(P_0^{\text{s}}-P_k^{\text{s}})P_k^{\text{p}}(P_0^{\text{s}}-P_k^{\text{s}})\right],
\end{aligned}
\end{equation*}
which directly leads to
\begin{equation*}
\begin{aligned}
\tilde{N}_{\cI}(\tilde{\Sigma}-\tilde{\Sigma}^*)&=\sum_{k\in\{0\}\cup \cI} \tilde{n}_k\Delta_k+\sum_{k\in\{0\}\cup \cI} \tilde{n}_k(P_k^{\text{s}}-P_0^{\text{s}})+P_0^{\text{s}}\left(\sum_{k\in\{0\}\cup \cI} \tilde{n}_kP_k^{\text{p}}\right)(P_0^{\text{s}})^{\perp}+(P_0^{\text{s}})^{\perp}\left(\sum_{k\in\{0\}\cup \cI} \tilde{n}_kP_k^{\text{p}}\right)P_0^{\text{s}}\\
& \quad+\sum_{k\in\{0\}\cup \cI} \tilde{n}_k\left[(P_0^{\text{s}}-P_k^{\text{s}})P_k^{\text{p}}(P_0^{\text{s}}-P_k^{\text{s}})\right].
\end{aligned}
\end{equation*}
By utilizing Davis-Kahan theorem, we have $\|\hat{P}_0^{\text{s}}-P_0^{\text{s}}\|_F\leq \|\tilde{\Sigma}-\tilde{\Sigma}^*\|_F
/d_{r_{\text{s}}}(\tilde{\Sigma}^*)$. Note that from Assumption \ref{assum:2} we have $d_{r_{\text{s}}}(\tilde{\Sigma}^*)\geq g$ for some constant $g>0$, combining the above equation allow us to obtain
\begin{equation*}
\begin{aligned}
\tilde{N}_{\cI}\|\hat{P}_0^{\text{s}}-P_0^{\text{s}}\|_F
&\lesssim
\left\|\sum_{k\in\{0\}\cup \cI} \tilde{n}_k\Delta_k+\sum_{k\in\{0\}\cup \cI} \tilde{n}_k(P_k^{\text{s}}-P_0^{\text{s}})+\sum_{k\in\{0\}\cup \cI} \tilde{n}_k\left[(P_0^{\text{s}}-P_k^{\text{s}})P_k^{\text{p}}(P_0^{\text{s}}-P_k^{\text{s}})\right]\right\|_F\\
& \quad+\left\|P_0^{\text{s}}\left(\sum_{k\in\{0\}\cup \cI} \tilde{n}_kP_k^{\text{p}}\right)(P_0^{\text{s}})^{\perp}+(P_0^{\text{s}})^{\perp}\left(\sum_{k\in\{0\}\cup \cI} \tilde{n}_kP_k^{\text{p}}\right)P_0^{\text{s}}\right\|_F\\
&\leq\left\|\sum_{k\in\{0\}\cup \cI} \tilde{n}_k\Delta_k\right\|_F+
\left\|\sum_{k\in\{0\}\cup \cI} \tilde{n}_k(P_k^{\text{s}}-P_0^{\text{s}})\right\|_F+
\left\|\sum_{k\in\{0\}\cup \cI} \tilde{n}_k\left[(P_0^{\text{s}}-P_k^{\text{s}})P_k^{\text{p}}(P_0^{\text{s}}-P_k^{\text{s}})\right]\right\|_F\\
& \quad+2\left\|P_0^{\text{s}}\left(\sum_{k\in\{0\}\cup \cI} \tilde{n}_kP_k^{\text{p}}\right)(I_p-P_0^{\text{s}})\right\|_F\\
&\lesssim \left\|\sum_{k\in\{0\}\cup \cI} \tilde{n}_k\Delta_k\right\|_F+\sum_{k\in\{0\}\cup \cI} \tilde{n}_k\left\|P_k^{\text{s}}-P_0^{\text{s}}\right\|_F+
\sum_{k\in\{0\}\cup \cI} \tilde{n}_k\left\|P_0^{\text{s}}-P_k^{\text{s}}\right\|_F^2\left\|P_k^{\text{p}}\right\|_F\\
&\quad + 2(1+\|P_0^{\text{s}}\|_F)\sum_{k\in\{0\}\cup \cI} \tilde{n}_k\left\|(P_0^{\text{s}}-P_k^{\text{s}}+P_k^{\text{s}})P_k^{\text{p}}\right\|_F \\
&\lesssim\left\|\sum_{k\in\{0\}\cup \cI} \tilde{n}_k\Delta_k\right\|_F+\sum_{k\in\{0\}\cup \cI} \tilde{n}_k(h+h^2)
+ \sum_{k\in\{0\}\cup \cI} \tilde{n}_k \left\|(P_0^{\s}-P_k^{\s})P_k^{\p} \right\|_F\\
&\lesssim\left\|\sum_{k\in\{0\}\cup \cI} \tilde{n}_k\Delta_k\right\|_F+\sum_{k\in\{0\}\cup \cI} \tilde{n}_kh,
\end{aligned}
\end{equation*}
where the last inequality holds as long as $h$ is sufficiently small. The above inequality implies that
\begin{equation}
\|\hat{P}_0^{\text{s}}-P_0^{\text{s}}\|_F\lesssim
\frac{1}{\tilde{N}_{\cI}}\left\|\sum_{k\in\{0\}\cup \cI} \tilde{n}_k\Delta_k\right\|_F+h.
\end{equation}
This completes the proof of Lemma \ref{basic lemma a.1}.

\subsubsection{Proof of Lemma \ref{basic lemma a.2}}
With Lemma \ref{basic lemma a.1} in mind, we mainly focus on the discussion of the following term:
\begin{equation}\label{main term of a.2}
\frac{1}{\tilde{N}_{\cI}}\left\|\sum_{k\in\{0\}\cup \cI} \tilde{n}_k\Delta_k\right\|_F.
\end{equation}
First let $\tilde{\Delta}_k=\tilde{P}_k-\E[\tilde{P}_k]$, then we can obtain following result of $\tilde{\Delta}_k$ by simple decomposition
\begin{equation*}
\begin{aligned}
\left\|\|\tilde{\Delta}_k\|_F\right\|_{\psi_1}&\leq
\left\|\|\Delta_k\|_F\right\|_{\psi_1}+\|\E[\tilde{P}_k]-P_k^*\|_F
=\left\|\|\Delta_k\|_F\right\|_{\psi_1}+\|\E(\tilde{P}_k-P_k^*)\|_F\\
&\leq\left\|\|\Delta_k\|_F\right\|_{\psi_1}+\E\|(\tilde{P}_k-P_k^*)\|_F\\
&=\left\|\|\Delta_k\|_F\right\|_{\psi_1}+\E\|\Delta_k\|_F\\
&\leq 2\left\|\|\Delta_k\|_F\right\|_{\psi_1}
\lesssim\left\|\|\Delta_k\|_F\right\|_{\psi_1}
\lesssim \tilde{n}_k^{-1/2}.
\end{aligned}
\end{equation*}
We can further decompose term (\ref{main term of a.2}) as
\begin{equation}\label{main decomposition of a.2}
\frac{1}{\tilde{N}_{\cI}}\left\|\left\|\sum_{k\in\{0\}\cup \cI} \tilde{n}_k\Delta_k\right\|_F\right\|_{\psi_1}\leq
\frac{1}{\tilde{N}_{\cI}}\left\|\left\|\sum_{k\in\{0\}\cup \cI} \tilde{n}_k\tilde{\Delta}_k\right\|_F\right\|_{\psi_1}+
\frac{1}{\tilde{N}_{\cI}}\sum_{k\in\{0\}\cup \cI} \tilde{n}_k\left\|\E \tilde{P}_k-P_k^*\right\|_F.
\end{equation}

For the first term of inequality of (\ref{main decomposition of a.2}), As $\E [\tilde{n}_k
\tilde{\Delta}_k]=0$ and $\left\|\left\|\tilde{n}_k
\tilde{\Delta}_k\right\|_F\right\|_{\psi_1}\lesssim
\tilde{n}_k^{1/2}$ for all $k\in\{0\}\cup \cI$ while $\tilde{\Delta}_k$ is independent from each other, we can utilize Proposition \ref{lem:subexp} to acquire
\begin{equation*}
\frac{1}{\tilde{N}_{\cI}}\left\|\left\|\sum_{k\in\{0\}\cup \cI} \tilde{n}_k\tilde{\Delta}_k\right\|_F\right\|_{\psi_1}\lesssim
\frac{1}{\tilde{N}_{\cI}}\sqrt{\sum_{k\in\{0\}\cup \cI} \left\|\left\|\tilde{n}_k\tilde{\Delta}_k
\right\|_F\right\|_{\psi_1}^2}
=\tilde{N}_{\cI}^{-1/2}.
\end{equation*}

As for the second term of inequality of (\ref{main decomposition of a.2}), we can use Lemma 2 and Theorem 3 in \cite{fan2019distributed}. Define $W_k=\tilde{P}_k-P_k^*-Q_k$ where $Q_k=f(E_kU_k^*)(U_k^*)^{\top}+U_k^*f^{\top}(E_kU_k^*)$, $E_k=\hat{\Sigma}_k-\Sigma_k^*$ while $f$ is the linear mapping defined the same as Lemma 2 in \cite{fan2019distributed}. We also denote $\epsilon_k=\|E_k\|_2/d_{r_k}(\Sigma_k^*)$, recall that $\|\epsilon_k\|_{\psi_1}\lesssim \tilde{n}_k^{-1/2}$ under the settings of Lemma \ref{lemma:indPCA}. Further let $\mathbb{I}_A(x)$ be the indicator function of event $A$, namely
\begin{equation*}
\mathbb{I}_A(x)=\begin{cases} 1 & \text{if } x \in A \\ 0 & \text{otherwise} \end{cases}.
\end{equation*}
Then by the linearity of $f$ and $\E [E_k]=0$, we have $\E [Q_k]=f((\E [E_k])U_k^*)(U_k^*)^{\top}+U_k^*f^{\top}((\E [E_k])U_k^*)=0$. Let $\Omega=(0,\infty)$, $A=(0,1/10]$ and $A^c=\Omega\setminus A$, it gives us:
\begin{equation}\label{W_k}
\begin{aligned}
W_k&=W_k\mathbb{I}_A(\epsilon_k)+(W_k+Q_k)\mathbb{I}_{A^c}(\epsilon_k)
-Q_k\mathbb{I}_{A^c}(\epsilon_k)\\
&=W_k\mathbb{I}_A(\epsilon_k)+(\tilde{P}_k-P_k^*)\mathbb{I}_{A^c}(\epsilon_k)
-Q_k\mathbb{I}_{A^c}(\epsilon_k).
\end{aligned}
\end{equation}
Take expectation on both sides of (\ref{W_k}) and applying the fact of $\E[Q_k]=0$, we have
\begin{equation*}
\begin{aligned}
\E (\tilde{P}_k-P_k^*)&=\E (\tilde{P}_k-P_k^*-Q_k)=\E [W_k]\\
&=\E(W_k\mathbb{I}_A(\epsilon_k))
+\E((\tilde{P}_k-P_k^*)\mathbb{I}_{A^c}(\epsilon_k))-
\E(Q_k\mathbb{I}_{A^c}(\epsilon_k)).
\end{aligned}
\end{equation*}
Then with the triangular inequality and Jensen's inequality, the second term of inequality (\ref{main decomposition of a.2}) can be separated as follows:
\begin{equation}\label{decomposition for the second term}
\left\|\E [\tilde{P}_k]-P_k^*\right\|_F\leq
\E\left(\|W_k\|_F\mathbb{I}_A(\epsilon_k)\right)+
\E(\|\tilde{P}_k-P_k^*\|_F\mathbb{I}_{A^c}(\epsilon_k))+
\E(\|Q_k\|_F\mathbb{I}_{A^c}(\epsilon_k)).
\end{equation}
For the first term of (\ref{decomposition for the second term}), we have
\begin{equation*}
\E\left(\|W_k\|_F\mathbb{I}_A(\epsilon_k)\right)\leq
\E(24\sqrt{r_k}\epsilon_k^2\mathbb{I}_A(\epsilon_k))\lesssim
\sqrt{r_k}\E[\epsilon_k^2],
\end{equation*}
by using Lemma 2 in \cite{fan2019distributed}.
For the second and the third term of (\ref{decomposition for the second term}), by Davis-Kahan theorem and the definition of $Q_k$, we have $\|\tilde{P}_k-P_k^*\|_F\lesssim\sqrt{r_k}\epsilon_k$ and $\|Q_k\|_F\lesssim\|f(E_kU_k^*)\|_F\lesssim\sqrt{r_k}\epsilon_k$. While for $A^c=(1/10,\infty)$, $\epsilon_k\mathbb{I}_{A^c}(\epsilon_k)\leq 10 \epsilon_k^2$. Accordingly, we have
\begin{equation*}
\E(\|\tilde{P}_k-P_k^*\|_F\mathbb{I}_{A^c}(\epsilon_k))+
\E(\|Q_k\|_F\mathbb{I}_{A^c}(\epsilon_k))\lesssim
\sqrt{r_k}\E(\epsilon_k\mathbb{I}_{A^c}(\epsilon_k))\lesssim
\sqrt{r_k}\E[\epsilon_k^2].
\end{equation*}
In conclusion, we have $\|\E [\tilde{P}_k]-P_k^*\|_F\lesssim\sqrt{r_k}\E\epsilon_k^2$, further we can deduce that
\begin{equation*}
\frac{1}{\tilde{N}_{\cI}}\sum_{k\in\{0\}\cup \cI} \tilde{n}_k\left\|\E [\tilde{P}_k]-P_k^*\right\|_F\lesssim
\frac{1}{\tilde{N}_{\cI}}\sum_{k\in\{0\}\cup \cI} \tilde{n}_k\sqrt{r_k}\E[\epsilon_k^2]\lesssim
\frac{1}{\tilde{N}_{\cI}}\sum_{k\in\{0\}\cup \cI} \tilde{n}_k\sqrt{r_k}\|\epsilon_k\|^ 2_{\psi_1}\leq
\left(\sum_{k\in\{0\}\cup\cI}r_k^{1/2}\right)\tilde{N}_{\cI}^{-1},
\end{equation*}
which completes the proof of Lemma \ref{basic lemma a.2}.

\subsection{Proof of Lemma \ref{lemma:actmain}}\label{proof:actmain}
First, according to Proposition \ref{lem:taylor}, we have $\|L\|_F \lesssim  r_\s^{1/2} s$, so that $\|L_\s \Sigma_0^*(P_0^{\s})^{\perp}\|_2\leq \|L_\s\|_2\|\Sigma_0^*\|_2\lesssim r_\s^{1/2} \|\Sigma_0^*\|_2 s$. For simplicity, here we abbreviate $\mathcal{U}_0^{\p}$ in (\ref{eq:eigen}) as $U_0$, while denoting the $i$-th column of $U_0$ as $u_i$ with a slight abuse of notations. In the spirit of Proposition \ref{lem:taylor}, we aim to write the linear term from (\ref{eq:ddotsep}) explicitly in the form of the linear matrix function, for the $p\times p$ matrix $E$, define
$$\cL(E) =  U_0  \begin{pmatrix}
	0_{(r_0-r_\s)\times(r_0-r_\s)}& \cK_1(E) &\cK_2(E)\\
          \cK_1(E)^{\top} &0_{(p-r_0)\times(p-r_0)} &\\
	\cK_2(E)^{\top} & &0_{r_\s\times r_\s}
\end{pmatrix}U_0^{\top},\quad\text{where}$$
$$\cK_1(E)_{ij} = \frac{u_i^{\top} E u_{r_0-r_\s+j}}{\lambda_i^{\p}-\lambda_{r_0+j}},\quad \text{for} \quad 1\leq i\leq r_0-r_\s,\quad 1\leq j\leq p-r_0,$$
$$\cK_2(E)_{ij} = \frac{u_i^{\top} E u_{p-r_\s+j}}{\lambda_i^{\p}},\quad \text{for} \quad 1\leq i\leq r_0-r_\s,\quad 1\leq j\leq r_\s.$$

Applying Proposition \ref{lem:taylor} to (\ref{eq:ddotsep}), for $\delta_{\p}=d_{r_0-r_\s}(\Sigma_0^{\p})=\lambda^{\p}_{r_0-r_\s}-\lambda_{r_0+1}$, by the linearity of $\cL$ we have
\begin{equation*}
    \begin{aligned}
        \dot{P}_0^{\p}-P_0^{\p} &= -\cL\left( L_\s \Sigma_0^*(P_0^{\s})^{\perp}\right)-\cL\left((P_0^{\s})^{\perp}\Sigma_0^*L_\s\right)+\cW_\s\\
        &= -\cL\left( L_\s \Sigma_0^*(P_0^{\s})^{\perp}\right)-\cL\left((P_0^{\s})^{\perp}\Sigma_0^*(P_0^{\s})^{\perp}L_\s\right)\\
        &\quad-\underbrace{\cL\left((P_0^{\s})^{\perp}\Sigma_0^*P_0^{\s}L_\s\right)}_{(e)}+\cW_\s,
    \end{aligned}
\end{equation*}
where $\|\cW_\s\|_F\lesssim \|\Sigma_0^*\|_2^2s^2/\delta^2_{\p}$ given sufficiently small $\|\Sigma_0^*\|_2s/\delta_{\p}$, while $\|(e)\|_F\lesssim \eta s/\delta_{\p}$ for $\eta = \|(U^{\p}_0)^{\top}\Sigma_0^*U^{\s}_0\|_2$, due to the fact that
\begin{equation}
    \begin{aligned}
    \|\cL\left((P_0^{\s})^{\perp}\Sigma_0^*P_0^{\s}L_\s\right)\|_F&\lesssim \|(P_0^{\s})^{\perp}\Sigma_0^*P_0^{\s}L_\s\|_2/\delta_{\p}\\
    &\leq \underbrace{\|(P_0^{*})^{\perp}\Sigma_0^*P_0^{\s}L_\s\|_2}_{0}/\delta_{\p}+\|(P_0^{\p})^{\perp}\Sigma_0^*P_0^{\s}L_\s\|_2/\delta_{\p}\\
    &\lesssim \|(U^{\p}_0)^{\top}\Sigma_0^*U^{\s}_0\|_2 s/\delta_{\p},
    \end{aligned}
\end{equation}
according to Proposition \ref{lem:taylor}, and directly
\begin{equation*}
    \begin{aligned}
        \left\|\dot{P}_0^{\p}-P_0^{\p}\right\|_F& \lesssim \underbrace{\left\|  \cK_1(L_\s \Sigma_0^*(P_0^{\s})^{\perp})\right\|_F}_{(a)}+  \underbrace{\left\|\cK_2(L_\s \Sigma_0^*(P_0^{\s})^{\perp})\right\|_F}_{(b)}\\
        &+ \underbrace{\left\|  \cK_1((P_0^{\s})^{\perp}\Sigma_0^*(P_0^{\s})^{\perp}L_\s)\right\|_F}_{(c)}+  \underbrace{\left\|\cK_2((P_0^{\s})^{\perp}\Sigma_0^*(P_0^{\s})^{\perp}L_\s)\right\|_F}_{(d)} \\
        &+\eta s/\delta_{\p} + \|\Sigma_0^*\|_2^2s^2/\delta^2_{\p}.
    \end{aligned}
\end{equation*}

We then control the remaining four terms respectively. Recall from (\ref{eq:Ls}) that $L_s = U_0 \cK_\s U_0^{\top}$ for
$$\cK_\s = \begin{pmatrix}
	0_{(p-r_\s)\times(p-r_\s)} & K_\s^{\top}\\
	K_\s & 0_{r_\s\times r_\s}
\end{pmatrix},\quad\text{and}\quad K_\s =  (U_0^{\s})^{\top} \Delta_\s (U_0^{\s})^{\perp}.$$

First, $(a)=(b)=(c)=0$ by noticing that
\begin{equation*}
    \begin{aligned}
        \cK_1(L_\s \Sigma_0^*(P_0^{\s})^{\perp})_{ij} &= \frac{(u_i^{\top} U_0) \cK_\s U_0^{\top} \Sigma_0^*\left[(P_0^{\s})^{\perp}u_{r_0-r_\s+j} \right]}{\lambda_i^{\p}-\lambda_{r_0+j}} = \frac{e_i^{\top} \cK_\s U_0^{\top} (\Sigma_0^* u_{r_0-r_\s+j})}{\lambda_i^{\p}-\lambda_{r_0+j}},\\
        &= \frac{\lambda_{r_0+j} e_i^{\top} \cK_\s (U_0^{\top} u_{r_0-r_\s+j})}{\lambda_i^{\p}-\lambda_{r_0+j}} =\frac{\lambda_{r_0+j}}{\lambda_i^{\p}-\lambda_{r_0+j}}e_i^{\top} \cK_\s  e_{r_0-r_\s+j}=0,
    \end{aligned}
\end{equation*}
\begin{equation*}
        \cK_2(L_\s \Sigma_0^*(P_0^{\s})^{\perp})_{ij} = \frac{u_i^{\top} L_\s \Sigma_0^*\left[(P_0^{\s})^{\perp} u_{p-r_\s+j} \right]}{\lambda_i^{\p}}=0,
\end{equation*}
\begin{equation*}
    \begin{aligned}
        \cK_1((P_0^{\s})^{\perp}\Sigma_0^*(P_0^{\s})^{\perp}L_\s)_{ij} &= \frac{\left(u_i^{\top} (P_0^{\s})^{\perp}\Sigma_0^*(P_0^{\s})^{\perp}\right)U_0\cK_\s\left(U_0^{\top}u_{r_0-r_\s+j} \right)}{\lambda_i^{\p}-\lambda_{r_0+j}}\\
        &= \frac{\lambda_i^{\p}}{\lambda_i^{\p}-\lambda_{r_0+j}} e_i^{\top}\cK_\s e_{r_0-r_\s+j}=0.
    \end{aligned}
\end{equation*}

In the end, we claim $(d)\leq \|\cK_\s\|_F= \|L\|_F\leq  r_\s^{1/2} s$, due to the fact that
\begin{equation*}
    \begin{aligned}
        \cK_2((P_0^{\s})^{\perp}\Sigma_0^*(P_0^{\s})^{\perp}L_\s)_{ij} &= \frac{\left(u_i^{\top} (P_0^{\s})^{\perp}\Sigma_0^*(P_0^{\s})^{\perp}\right)U_0\cK_\s\left(U_0^{\top} u_{p-r_\s+j} \right)}{\lambda_i^{\p}}\\
        &=  \frac{\lambda_i^{\p}}{\lambda_i^{\p}} \left(u_i^{\top}U_0\right)\cK_\s e_{p-r_\s+j} = (\cK_\s)_{i, p-r_\s+j}.
    \end{aligned}
\end{equation*}

Finally, for sufficiently small $\|\Sigma_0^*\|_2s/\delta_{\p}$, we are able to claim that
    $$\left\|\dot{P}_0^{\p}-P_0^{\p}\right\|_F\lesssim s+\frac{\eta s}{\delta_{\p}}+\frac{\|\Sigma_0^*\|_2^2s^2}{\delta^2_{\p}}.$$

\section{Proof of Corollary \ref{coro:AN}}

First, note that $\langle u, (\hat{P}_0-P_0^*) v\rangle=\langle u, (\hat{P}_0^{\s}-P_0^{\s}) v\rangle+\langle u, (\hat{P}_0^{\p}-P_0^{\p}) v\rangle$. Since $u$, $v\in S^{p-1}$, while $\|\hat{P}_0^{\text{s}}-P_0^{\text{s}}\|_F=O_p(s)$ according to Lemma \ref{lemma:sse}, we have $\langle u, (\hat{P}_0^{\s}-P_0^{\s}) v\rangle=O_p(s)$ as well. Then, we focus on
\begin{equation}\label{eq:coro1}
    \langle u, (\hat{P}_0^{\p}-P_0^{\p}) v\rangle=\langle u, (\hat{P}_0^{\p}-\ddot{P}_0^{\p}) v\rangle+\langle u, (\ddot{P}_0^{\p}-P_0^{\p}) v\rangle,
\end{equation}
where $\ddot{P}_0^{\p}$ is formed by taking the leading $(r_0-r_\s)$ eigenvectors of $\ddot{\Sigma}_0^{\p}:=(P_0^{\s})^{\perp} \hat{\Sigma}_0 (P_0^{\s})^{\perp}$. For the first term on the right hand side of (\ref{eq:coro1}), we focus on the following matrix perturbation quite similar to that in (\ref{eq:ddotP}):
\begin{equation*}
    \ddot{\Sigma}_0^{\p}= (P_0^{\s})^{\perp} \hat{\Sigma}_0 (P_0^{\s})^{\perp} + \Delta_\s^{\perp} \hat{\Sigma}_0 (P_0^{\s})^{\perp} + (P_0^{\s})^{\perp} \hat{\Sigma}_0 \Delta_\s^{\perp} + \Delta_\s^{\perp}\hat{\Sigma}_0 \Delta_\s^{\perp}.
\end{equation*}
We then have $\|\hat{P}_0^{\p}-\ddot{P}_0^{\p}\|_2\lesssim \|\hat{\Sigma}_0\|_2s/\delta_{\p}=O_p(\|\Sigma^*_0\|_2s/\delta_{\p})$ by Davis-Kahan theorem following similar arguments as in Section \ref{proof:main}. Note that $\hat{\Sigma}=\Sigma_0^*+E_0$ and we have $\|E_0\|_2=O_p(\|\Sigma_0^*\|_2e_0^{1/2}n_0^{-1/2})$ under Assumption \ref{assum:1} according to Theorem 9 in \cite{koltchinskii2017concentration}. As for the second term on the right hand side of (\ref{eq:coro1}), namely $\langle u, (\ddot{P}_0^{\p}-P_0^{\p}) v\rangle$, we start by introducing the following Lemma \ref{lem:lt}.

\begin{lemma}[Linear term]\label{lem:lt}
    Under the settings of Proposition \ref{lem:taylor}, if we further set $\hat{\Sigma}=\sum_{k=1}^{n}Az_kz_kA^{\top}/n$ where $\Sigma=AA^{\top}$, while the mutually independent random vectors $\{z_k\}_{k=1}^{n}$ consist of $p$ i.i.d. random variables $\{z_{k,m}\}_{m=1}^{p}$, such that $\E(z_{k,m})=0$, $\E(z^2_{k,m})=1$ and $\E(z^4_{k,m})=\nu_4$, let $E=\hat{\Sigma}-\Sigma$, take $u$, $v\in S^{p-1}$ and we have,
   $$\left\langle u, (\hat{U}_S\hat{U}_S^{\top}-U_SU_S^{\top})v\right\rangle =\frac{1}{n}\sum_{k=1}^{n}L_{S,k} +R_S,$$
    where $\{L_{S,k}\}_{k=1}^{n}$ are i.i.d. random variables with $\E(L_{S,k})=0$. Set $\rho_{ij}:=(U^{\top}u)_i(U^{\top}v)_j+(U^{\top}u)_j(U^{\top}v)_i$, and the subscript $i$ means the $i$-th element of the vector. Let $\gamma_i := A^{\top}u_i$ and $\gamma_{im}:=(A^{\top}u_i)_m$, we have
    $$ \E(L^2_{S,k})=\sigma_S^2:=\sum_{i\in S}\sum_{j\in S^c}\frac{\rho^2_{ij}\lambda_i\lambda_j}{(\lambda_i-\lambda_j)^2}+(\nu_4-3)\sum_{m=1}^{p}\left( \sum_{i\in S}\sum_{j\in S^c} \frac{\rho_{ij}\gamma_{im} \gamma_{jm}}{\lambda_{i}-\lambda_{j}}\right)^2.$$
    Meanwhile, $R_S\lesssim r^{1/2}(\|E\|_2/\delta)^2$ within the event $\{\|E\|_2/\delta\leq 1/10\}$.
\end{lemma}

The proof is complete by plugging in $S=[r_0-r_{\s}]$ and $A=(P_0^{\s})^{\perp}(\Sigma_0^*)^{1/2}$ into Lemma \ref{lem:lt}, as the event ensuring the remainder term to be of order $(\delta_0/\delta_{\p})^2\tilde{n}_0^{-1}$ has probability tending to 1 as $(\delta_0/\delta_{\p})\tilde{n}_0^{-1/2} \rightarrow 0$. Also, this particular choice of $A$ guarantees that $\gamma_{im}=(\lambda^{\p}_i)^{1/2}(u_i^{\p})_m$ for $i\leq r_0-r_{\s}$ and $\gamma_{jm}=\lambda_j^{1/2}(u_j)_m$ for $r_0+1\leq j\leq p$. In the end, we present the proof of Lemma \ref{lem:lt}.

\subsection{Proof of Lemma \ref{lem:lt}}
Lemma \ref{lem:lt} follows directly from Proposition \ref{lem:taylor} by noticing that $K$ therein is linear with respect to $E=\hat{\Sigma}-\Sigma=\sum_{k=1}^{n}(Az_kz_kA^{\top}-AA^{\top})/n$, so the linear term of $\langle u, (\hat{U}_S\hat{U}_S^{\top}-U_SU_S^{\top})v\rangle$ could be written as

\begin{equation*}
       \left\langle U^{\top}u,KU^{\top}v \right\rangle
       =\sum_{i\in S}\sum_{j\in S^c}\rho_{ij}K_{ij},
\end{equation*}
where $\rho_{ij}:=(U^{\top}u)_i(U^{\top}v)_j+(U^{\top}u)_j(U^{\top}v)_i$, and $(U^{\top}u)_i$ means the $i$-th element of the vector $U^{\top}u$. Note that $K_{ij}=(u_i^{\top} (\hat{\Sigma}-\Sigma) u_j)/(\lambda_i-\lambda_j)$, while $u_i^{\top} \Sigma u_j=0$ for $i\in S$ and $j\in S^c$, we have $K_{ij}=\sum_{k=1}^{n}(u_i^{\top} Az_kz_kA^{\top}  u_j)/[n(\lambda_i-\lambda_j)]$. Hence directly $ \left\langle U^{\top}u,KU^{\top}v \right\rangle =\sum_{k=1}^{n}L_{S,k}/n$ where
\begin{equation*}
    \begin{aligned}
      L_{S,k} &= \sum_{i\in S}\sum_{j\in S^c}\frac{\rho_{ij}}{\lambda_i-\lambda_j}u_i^{\top} Az_kz_kA^{\top}  u_j\\
      &=\sum_{i\in S}\sum_{j\in S^c}\frac{\rho_{ij}}{\lambda_i-\lambda_j}\sum_{m=1}^{p}\sum_{n=1}^{p}(A^{\top}u_i)_m(A^{\top}u_j)_n z_{k,m}z_{k,n}, \quad \text{for}
    \end{aligned}
\end{equation*}

We now focus on the i.i.d. $\{L_{S,k}\}_{k=1}^{n}$. As clearly $\E(L_{S,k})=0$, we only need to calculate $\E(L^2_{S,k})$. We abbreviate $z_m:= z_{k,m}$ and $z_n:= z_{k,n}$ for brevity, let $\gamma_i := A^{\top}u_i$ and $\gamma_{im}:=(A^{\top}u_i)_m$, and we have
\begin{equation}\label{eq:VLSk}
        \E(L^2_{S,k})= \sum_{i_1\in S}\sum_{j_1\in S^c} \sum_{i_2\in S}\sum_{j_2\in S^c}\frac{\rho_{i_1j_1}\rho_{i_2j_2}}{(\lambda_{i_1}-\lambda_{j_1})(\lambda_{i_2}-\lambda_{j_2})}\left[\underbrace{\sum_{m_1=1}^{p}\sum_{n_1=1}^{p}\sum_{m_2=1}^{p}\sum_{n_2=1}^{p}\gamma_{i_1m_1} \gamma_{j_1n_1} \gamma_{i_2m_2} \gamma_{j_2n_2}\E(z_{m_1}z_{n_1}z_{m_2}z_{n_2})}_{T_{i_1j_1i_2j_2}}\right].\\
\end{equation}

For $T_{i_1j_1i_2j_2}$, since $\{z_{m}\}_{m=1}^{p}$ are i.i.d. random variables with zero mean and unit variance, we only need to consider the even order terms, to be more specific,
\begin{equation}\label{eq:Tiijj}
    \begin{aligned}
        T_{i_1j_1i_2j_2} &= \underbrace{\nu_4\sum_{m=1}^{p} \gamma_{i_1m} \gamma_{j_1m} \gamma_{i_2m} \gamma_{j_2m}}_{m_1=n_1=m_2=n_2}+   \underbrace{\sum_{m\neq n} (\gamma_{i_1m}\gamma_{i_2m})(\gamma_{j_1n}  \gamma_{j_2n})}_{m_1=m_2\neq n_1=n_2}\\
        &+\underbrace{\sum_{m\neq n} (\gamma_{i_1m} \gamma_{j_1m})(\gamma_{i_2n} \gamma_{j_2n})}_{m_1=n_1\neq m_2=n_2}
        +\underbrace{\sum_{m\neq n} (\gamma_{i_1m} \gamma_{j_2m})(\gamma_{i_2n} \gamma_{j_1n})}_{m_1=n_2\neq m_2=n_1}\\
        &= (\nu_4-3)\sum_{m=1}^{p} \gamma_{i_1m} \gamma_{j_1m} \gamma_{i_2m} \gamma_{j_2m} + \sum_{m=1}^{p}\sum_{n=1}^p (\gamma_{i_1m}\gamma_{i_2m})(\gamma_{j_1n}  \gamma_{j_2n})\\
        &+ \sum_{m=1}^{p}\sum_{n=1}^p (\gamma_{i_1m} \gamma_{j_1m})(\gamma_{i_2n} \gamma_{j_2n})+\sum_{m=1}^{p}\sum_{n=1}^p (\gamma_{i_1m} \gamma_{j_2m})(\gamma_{i_2n} \gamma_{j_1n})\\
        &= (\nu_4-3)\sum_{m=1}^{p} \gamma_{i_1m} \gamma_{j_1m} \gamma_{i_2m} \gamma_{j_2m} + \langle \gamma_{i_1},\gamma_{i_2}\rangle\langle\gamma_{j_1},\gamma_{j_2}\rangle\\
        &+ \underbrace{\langle \gamma_{i_1},\gamma_{j_1}\rangle\langle\gamma_{i_2},\gamma_{j_2}\rangle}_{0}+ \underbrace{\langle \gamma_{i_1},\gamma_{j_2}\rangle\langle\gamma_{i_2},\gamma_{j_1}\rangle}_{0}\\
        &=  (\nu_4-3)\sum_{m=1}^{p} \gamma_{i_1m} \gamma_{j_1m} \gamma_{i_2m} \gamma_{j_2m} +\lambda_{i_1}\lambda_{j_1} \mathbb{I}_{\{i_1=i_2,j_1=j_2\}}.
    \end{aligned}
\end{equation}
The last step is due to the fact that $\langle \gamma_{i},\gamma_{j}\rangle=\lambda_i\mathbb{I}_{\{i=j\}}$. Plug (\ref{eq:Tiijj}) into (\ref{eq:VLSk}) and notice that
\begin{equation}
    \begin{aligned}
        &\sum_{i_1\in S}\sum_{j_1\in S^c} \sum_{i_2\in S}\sum_{j_2\in S^c}\frac{\rho_{i_1j_1}\rho_{i_2j_2}}{(\lambda_{i_1}-\lambda_{j_1})(\lambda_{i_2}-\lambda_{j_2})}\sum_{m=1}^{p} \gamma_{i_1m} \gamma_{j_1m} \gamma_{i_2m} \gamma_{j_2m}\\
        &=\sum_{m=1}^{p}\left( \sum_{i_1\in S}\sum_{j_1\in S^c} \frac{\rho_{i_1j_1}\gamma_{i_1m} \gamma_{j_1m}}{\lambda_{i_1}-\lambda_{j_1}}\right)\left( \sum_{i_2\in S}\sum_{j_2\in S^c} \frac{\rho_{i_2j_2}\gamma_{i_2m} \gamma_{j_2m}}{\lambda_{i_2}-\lambda_{j_2}}\right)\\
        &=\sum_{m=1}^{p}\left( \sum_{i\in S}\sum_{j\in S^c} \frac{\rho_{ij}\gamma_{im} \gamma_{jm}}{\lambda_{i}-\lambda_{j}}\right)^2.
    \end{aligned}
\end{equation}
Finally we have
$$ \E(L^2_{S,k})=\sum_{i\in S}\sum_{j\in S^c}\frac{\rho^2_{ij}\lambda_i\lambda_j}{(\lambda_i-\lambda_j)^2}+(\nu_4-3)\sum_{m=1}^{p}\left( \sum_{i\in S}\sum_{j\in S^c} \frac{\rho_{ij}\gamma_{im} \gamma_{jm}}{\lambda_{i}-\lambda_{j}}\right)^2.$$

\section{Proof of Theorem \ref{theo:nora}}
First, according to Lemma \ref{lemma:sse}, we have $\|\hat{P}_0^{\text{s}}-P_0^{\text{s}}\|_F=O_p(s)$ under Assumptions \ref{assum:1} and \ref{assum:2}. In addition, we have $\|\tilde{P}_k-P_k^*\|_F=O_p(\tilde{n}_k^{-1/2})$ according to Lemma \ref{lemma:indPCA}. Then, define $\hat{d}_k:=r_\s - \tr(\tilde{P}_k\hat{P}_0^{\s})$, by triangular inequality $\hat{d}_k=d_k + O_p(s+\tilde{n}_k^{-1/2})$, so under Assumption \ref{assum:3} we have
$$\hat{d}_k = O_p(s+\tilde{n}_k^{-1/2}),\quad \text{for}\quad k\in \cI, $$
$$\hat{d}_k \geq d_\tau + O_p(s+\tilde{n}_k^{-1/2}),\quad \text{for}\quad k\in \cI^c.$$

Now, set $\tau = r_\s-d_\tau/2$, since $s+\max_{k\in \cI^c}(\tilde{n}_k^{-1/2})=o(d_\tau)$ as $\min_{k\in \{0\}\cup[K]}(n_k), p\rightarrow \infty$ and $h\rightarrow 0$, with probability tending to 1 we have $\tr[\hat{P}^{\s}_{0}\tilde{P}_k]\geq \tau$ for $k\in \cI$ and $\tr[\hat{P}^{\s}_{0}\tilde{P}_k]< \tau$ for $k\in \cI^c$. Correspondingly, the probability of the event $\{\hat{\cI}=\cI\}$ tends to 1 for
$$\hat{\cI}=\left\{k\in[K]\mid \tr[\hat{P}^{\s}_{0}\tilde{P}_k]\geq \tau \right\}.$$
Within this event, $\hat{P}^{\s}_0$ is a local maximum of (\ref{major solution-sample}), and the proof is complete.

\section{Proof of Corollary \ref{coro:ext}}
After carefully going through the proofs with respect to Lemma \ref{lemma:indPCA} and Theorems \ref{theo:main}, \ref{theo:nora}, it is not hard to notice that almost all theoretical derivations extend readily to the elliptical PCA setting by simply replacing Assumption \ref{assum:1} by Assumption \ref{assum:4}. As a quick reminder, Assumptions \ref{assum:2} and \ref{assum:3} are only related to the population subspaces, so that they are in some sense ``PCA-setting-free". In this section, we briefly mention the differences between the classical and elliptical PCA settings, one might also follow the same procedure and apply the knowledge transfer framework on other PCA settings of interest.

\subsection{Differences in Lemma \ref{lemma:indPCA}}
Before we give the proof of Lemma \ref{lemma:indPCA} with Assumption \ref{assum:1} replaced by Assumption \ref{assum:4}, some results from \cite{he2022distributed} are stated here as they are required in the following statements.
\begin{proposition}[Lemma A.2 and A.5 from \cite{he2022distributed}]\label{lem:ext}
    For $k\in\{0\}\cup[K]$ under the scenario given in Section \ref{sec:tepca}, when Assumption \ref{assum:4} holds, we have
    \begin{equation}\label{eq:ecaindierror}
    \left\|\left\|\hat{\Sigma}_k-\Sigma_k^*\right\|_2\right\|_{\psi_1}\lesssim
    \sqrt{\frac{\log p}{n_k}},
    \end{equation}
    where $\Sigma_k^*$ and $\hat{\Sigma}_k$ stand for the population and sample versions of the spatial Kendall's $\tau$ matrix. Moreover, recall that $\Sigma_k$ is the scatter matrix, for the $r_k$-th eigengap of $\Sigma_k^*$, we have the following bound when $p$ is sufficiently large:
    \begin{equation}\label{eq:ecagapbound}
        \frac{1}{d_{r_k}(\Sigma_k^*)}\lesssim\frac{\tr(\Sigma_k)}{d_{r_k}(\Sigma_k)}.
    \end{equation}
\end{proposition}

By Davis-Kahan theorem, we have
\begin{equation*}
\left\|\left\|\Delta_k\right\|_F\right\|_{\psi_1}=
        \left\|\left\|\tilde{P}_k-P^*_k\right\|_F\right\|_{\psi_1}
        \lesssim\frac{\sqrt{r_k}}{d_{r_k}(\Sigma^*_k)}\left\|\left\|\hat{\Sigma}_k-\Sigma^*_k\right\|_F\right\|_{\psi_1},
\end{equation*}
where $\Sigma_k^*$ and $\hat{\Sigma}_k$ represent the population and sample versions of the spatial Kendall's $\tau$ matrix respectively. Utilizing (\ref{eq:ecaindierror}) allow us to obtain the following inequality:
\begin{equation*}
    \frac{\sqrt{r_k}}{d_{r_k}(\Sigma^*_k)}\left\|\left\|\hat{\Sigma}_k-\Sigma^*_k\right\|_F\right\|_{\psi_1}\lesssim
    \frac{\sqrt{r_k}}{d_{r_k}(\Sigma^*_k)}\sqrt{\frac{\log p}{n_k}}.
\end{equation*}
Further using (\ref{eq:ecagapbound}),
\begin{equation*}
\left\|\left\|\Delta_k\right\|_F\right\|_{\psi_1}\lesssim
    \frac{\sqrt{r_k}\tr(\Sigma_k)}{d_{r_k}(\Sigma_k)}\sqrt{\frac{\log p}{n_k}}
    =\frac{\lambda_1(\Sigma_k)}{d_{r_k}(\Sigma_k)}
    \frac{\tr(\Sigma_k)}{\lambda_1(\Sigma_k)}\sqrt{\frac{r_k\log p}{n_k}}.
\end{equation*}
Replacing the first two fractions of the right hand side term with $\kappa_k$ and $e_k$, then we have
\begin{equation*}
\left\|\left\|\Delta_k\right\|_F\right\|_{\psi_1}\lesssim
    \kappa_ke_k\sqrt{\frac{r_k\log p }{n_k}}=\tilde{n}_k^{-1/2},
\end{equation*}
where $\tilde{n}_k$ is defined in Corollary \ref{coro:ext}, and the proof is complete.

\subsection{Differences in Theorems \ref{theo:main} and \ref{theo:nora}}
Indeed, Lemma \ref{lemma:indPCA}, or more specifically, the fact that $\Delta_k = \tilde{P}_k-P_k^*$ are mutually independent random matrices such that $\|\|\Delta_k\|_F\|_{\psi_1}\lesssim \tilde{n}_k^{-1/2}$ for $k\in\{0\}\cup[K]$, is used repeatedly throughout the proofs of Theorem \ref{theo:main} and \ref{theo:nora}, in order to provide control on the individual PCA error of the $k$-th PCA study. As previously mentioned, all arguments coming from Assumptions \ref{assum:2} and \ref{assum:3} are ``PCA-setting-free", and the proofs of Theorems \ref{theo:main} and \ref{theo:nora} directly follows by simply plugging-in the elliptical version of $\tilde{n}_k$ throughout the derivations.

\end{document}